\documentclass[runningheads]{llncs}

\usepackage{eccv}

\usepackage{eccvabbrv}

\usepackage{graphicx}
\usepackage{booktabs}

\usepackage[accsupp]{axessibility}  

\usepackage{amsmath}
\usepackage{amssymb}
\usepackage{mathtools}
\usepackage{amsfonts}
\usepackage{multirow,graphicx}
\usepackage{makecell}    
\usepackage{array}       
\usepackage{makecell}
\usepackage{subcaption}
\usepackage{adjustbox}
\usepackage{enumitem}

\usepackage{wrapfig}

\usepackage{amsfonts,bm}

\def\eqref#1{equation~\ref{#1}}

\def\1{\bm{1}}

\def\vtheta{{\bm{\theta}}}

\def\ve{{\bm{e}}}

\def\vs{{\bm{s}}}

\def\vx{{\bm{x}}}

\def\mE{{\bm{E}}}

\def\mH{{\bm{H}}}
\def\mI{{\bm{I}}}

\def\mM{{\bm{M}}}

\def\mP{{\bm{P}}}

\def\mW{{\bm{W}}}
\def\mX{{\bm{X}}}

\DeclareMathAlphabet{\mathsfit}{\encodingdefault}{\sfdefault}{m}{sl}
\SetMathAlphabet{\mathsfit}{bold}{\encodingdefault}{\sfdefault}{bx}{n}

\usepackage{hyperref}

\usepackage{orcidlink}

\makeatletter
\DeclareRobustCommand\onedot{\futurelet\@let@token\bmv@onedotaux}
\def\bmv@onedotaux{\ifx\@let@token.\else.\null\fi\xspace}
\def\eg{\emph{e.g}\onedot}

 \def\vs{\emph{vs}\onedot}

\makeatother

\begin{document}

\title{Structured Hyperedge Adaptation for Parameter-Efficient Fine-Tuning of Vision Transformers 
} 

\titlerunning{HyperAdapter: Structured Hyperedge Adaptation for ViTs
}

\author{Edwin Kwadwo Tenagyei\inst{1}\thanks{These authors contributed equally and are co-first authors.}\orcidlink{0009-0003-3045-4935} \and
Lei Wang\inst{1, 2}$^\star$\orcidlink{0000-0002-8600-7099} \and
Ugochukwu Ejike Akpudo\inst{1}\orcidlink{0000-0003-4221-5192} \and 
Jun Zhou\inst{1}\orcidlink{0000-0001-5822-8233} \and
Yongsheng Gao\inst{1}\thanks{Corresponding author: yongsheng.gao@griffith.edu.au (Yongsheng Gao)}\orcidlink{0000-0002-5382-5351}
}

\authorrunning{E.K.~Tenagyei et al.}

\institute{Griffith University, Nathan QLD 4111, Australia \and
Data61/CSIRO, Canberra ACT 2601, Australia
}

\maketitle

\begin{abstract}
  Parameter-efficient fine-tuning (PEFT) has become a practical solution for adapting large pretrained vision transformers (ViTs) to downstream tasks while updating only a small subset of parameters. However, existing adapter-based methods perform adaptation independently for each token, implicitly assuming that token refinements should be learned in isolation. This token-wise formulation overlooks the structured relationships among tokens that naturally arise in visual scenes, potentially leading to redundant updates and spatially inconsistent feature refinement.
  In this work, we revisit the design of parameter-efficient adapters and propose to perform adaptation in hyperedge space rather than token space. We introduce HyperAdapter, a hypergraph-based adapter architecture that enables structured, group-aware adaptation through soft token routing. HyperAdapter constructs a soft hypergraph over ViT tokens using prototype-based assignments, aggregates token features into latent hyperedge representations, applies lightweight bottleneck adaptation at the hyperedge level, and diffuses the resulting updates back to tokens via the hypergraph incidence structure. This design injects an explicit structural inductive bias into PEFT while preserving the modularity and efficiency of standard adapters.
  Extensive experiments across diverse visual benchmarks demonstrate that structured hyperedge adaptation consistently outperforms strong PEFT baselines under comparable parameter budgets, with particularly pronounced gains on tasks requiring structured reasoning. Our results suggest that the choice of adaptation space is a critical yet underexplored dimension in parameter-efficient transfer for ViTs.
\end{abstract}

\section{Introduction}

Large pretrained vision transformers (ViTs)~\cite{vit,swin,swinv2,mae,deit} have become the backbone of modern visual recognition systems, delivering remarkable performance across diverse tasks and domains. However, adapting these large-scale models to new datasets through full fine-tuning remains computationally expensive and memory intensive, particularly when multiple downstream tasks need to be supported. Parameter-efficient fine-tuning (PEFT) methods~\cite{vpt,adapter,lora,ssf} address this challenge by updating only a small subset of parameters while keeping the pretrained backbone frozen. Among them, adapter-based approaches~\cite{adapter,adaptformer,compacter,convpass} have emerged as a practical and modular solution, inserting lightweight trainable modules into transformer blocks to enable efficient transfer.

Despite their empirical success, existing adapter designs share a largely overlooked assumption: adaptation is performed independently for each token. In standard formulations, token representations are refined through low-rank or bottleneck transformations applied identically and independently across tokens. Although the frozen self-attention layers encode contextual interactions, the adaptation mechanism itself remains token-wise. This design implicitly assumes that feature refinement should occur in token space, without explicitly modeling the structured relationships that naturally arise among tokens in visual scenes. In practice, image tokens often correspond to coherent regions such as objects, parts, or semantic components. Ignoring such higher-order structure during adaptation may lead to redundant updates, spatially inconsistent feature refinement, and limited utilization of relational information.

In this work, we revisit the design principle of parameter-efficient adapters and argue that \emph{the choice of adaptation space} is a critical yet underexplored dimension in PEFT. Instead of refining tokens independently, we propose to perform adaptation in a structured interaction space that explicitly captures higher-order token relationships. To this end, we introduce HyperAdapter, a hypergraph-based adapter architecture that operates in hyperedge space rather than token space.
HyperAdapter constructs a soft hypergraph over ViT patch tokens using prototype-based routing, where each token is softly assigned to a small set of latent hyperedges according to representation similarity. These hyperedges serve as structured groups that aggregate information from multiple related tokens. Token features are first pooled into compact hyperedge representations, which are then refined through a lightweight bottleneck adapter at the hyperedge level. The adapted hyperedge features are subsequently diffused back to token representations via the hypergraph incidence structure. This hyperedge-level adaptation introduces an explicit structural inductive bias, encouraging coherent and group-aware feature updates while preserving the efficiency and modularity of standard adapters.

Importantly, HyperAdapter is a drop-in replacement for conventional adapters. It requires no modification to the pretrained backbone, introduces only a modest number of additional parameters, and incurs minimal computational overhead relative to the frozen self-attention layers. Moreover, we show that token-wise adapters arise as a special case of our formulation, demonstrating that HyperAdapter generalizes standard designs while retaining their favorable properties, including permutation equivariance and low-rank adaptation structure.
We evaluate HyperAdapter on 24 downstream tasks spanning diverse visual domains under parameter-efficient settings. Across multiple backbones and benchmarks, HyperAdapter consistently outperforms strong adapter-based baselines and other PEFT methods under comparable parameter budgets, with particularly notable gains on tasks requiring structured reasoning. These results indicate that modeling structured token interactions during adaptation provides a simple yet effective mechanism for enhancing parameter-efficient transfer in vision transformers.
Our main \textbf{contributions} are summarized as follows:
\renewcommand{\labelenumi}{\roman{enumi}.}
\begin{enumerate}[leftmargin=0.5cm]
    \item We revisit the design of parameter-efficient adapters and identify the adaptation space as a critical dimension, highlighting the limitations of token-wise refinement.
    \item We propose HyperAdapter, a hypergraph-based adapter architecture that performs structured adaptation in hyperedge space via prototype-based token grouping.
    \item We develop a group-level adaptation mechanism that aggregates token representations into hyperedges, applies lightweight bottleneck refinement at the hyperedge level, and diffuses structured updates back to tokens.
    \item We conduct extensive experiments across 24 downstream tasks and multiple transformer backbones, demonstrating consistent and parameter-efficient improvements over strong PEFT baselines.
\end{enumerate}

\section{Related Work}

\textbf{PEFT for ViTs.} The rapid scaling of ViTs~\cite{vit,swin,swinv2,deit,mae} has significantly increased the computational and memory cost of adaptation. Full fine-tuning updates all model parameters and typically achieves strong performance, but becomes impractical when deploying models across many downstream tasks. PEFT methods aim to reduce adaptation cost by updating only a small subset of parameters while keeping the pretrained backbone largely frozen.
Existing PEFT approaches for ViTs can be broadly grouped into several families. 
Prompt-based methods~\cite{vpt,e2vpt,self-supervised-vpt,lion,sa2vp,da-vpt} introduce learnable prompt tokens that interact with the self-attention mechanism to steer task adaptation without modifying backbone weights. 
Adapter-based methods~\cite{adapter,adaptformer,compacter,convpass,arc,householder,repadapter} insert lightweight bottleneck modules inside transformer blocks to refine intermediate representations. 
Selective tuning methods~\cite{bitfit,gps,gps-compact} update only specific existing parameters, such as biases or carefully chosen weight subsets. 
Reparameterization-based methods~\cite{lora,kharao,parameter,sine,fact,BOFT,goft,ssf,orthogonal_subspace} introduce low-rank or structured weight decompositions that can be merged into pretrained weights at inference time. 
Hybrid strategies~\cite{noah,vpeft} combine multiple PEFT paradigms to improve flexibility and expressiveness.
Despite their methodological differences, most existing PEFT approaches share a common design principle: \emph{adaptation is performed independently for each token}. In adapter-based and low-rank methods, the same transformation is applied token-wise to all patch embeddings. Prompt-based methods alter attention interactions but still refine token features individually after attention mixing. None of these approaches explicitly model structured group-level interactions during adaptation. 

Our work departs from this paradigm by revisiting the \emph{adaptation space} itself. Rather than refining token embeddings independently, we perform adaptation in a structured hyperedge space that aggregates and refines groups of tokens jointly. This design introduces an explicit relational inductive bias during parameter-efficient fine-tuning, while remaining fully compatible with standard adapter formulations. Importantly, conventional token-wise adapters arise as a special case of our formulation when each hyperedge contains a single token, showing that our method strictly generalizes existing adapter designs.

\textbf{Structured adaptation and token interaction modeling.} Several recent works have explored modeling structured relationships within ViTs. Self-attention inherently captures pairwise token interactions, and extensions such as dynamic routing or token clustering have been proposed to enhance relational reasoning. In PEFT settings, some approaches attempt to improve adaptation by modifying attention maps or introducing task-specific routing strategies. However, these methods typically alter attention behavior within the backbone rather than redefining where adaptation occurs.
More closely related are approaches that incorporate grouping or clustering mechanisms into transformer processing. For instance, graph-based vision transformers~\cite{vig,mobilevig,greedyvig} construct graphs over image patches to model spatial relationships, while token clustering methods aggregate patch tokens to reduce redundancy or improve efficiency. Nevertheless, these works focus on backbone architecture design or inference efficiency, not PEFT. They modify the transformer’s core computation rather than introducing a structured adaptation module on top of a frozen backbone.

In contrast, our method leaves the pretrained transformer unchanged and introduces structured modeling within the adapter module. The hypergraph construction in HyperAdapter serves as a task-adaptive grouping mechanism used solely for adaptation, without altering the underlying attention layers. This separation allows us to inject higher-order structure into PEFT while preserving the modularity and deployment advantages of adapter-based methods.

\textbf{Graph and hypergraph learning.} GNNs~\cite{inductive,collective,chemical} were originally developed for relational data and have since been extended to vision tasks~\cite{vig,mobilevig,greedyvig}. Vision Graphs (ViGs)~\cite{vig} model image patches as graph nodes connected through learned or predefined edges, enabling relational reasoning beyond grid-based convolutions. To reduce the cost of dynamic graph construction, later works such as MobileViG\cite{mobilevig} and GreedyViG\cite{greedyvig} adopt static or simplified graph structures.
Hypergraphs extend standard graphs by allowing hyperedges to connect more than two nodes, enabling higher-order relational modeling. Hypergraph-based methods have been explored for visual tasks including 3D understanding and video modeling~\cite{3d,video-hypergraph,video-multilevel}, and more recently integrated into vision GNNs~\cite{vignn,visionhgnn,hypergraph} to capture complex multi-way relationships.

However, existing graph and hypergraph approaches primarily target backbone design or relational feature learning from scratch. They typically replace or augment convolutional or transformer layers with graph-based computation. In contrast, our work operates in a fundamentally different regime. We do not redesign the backbone nor introduce heavy graph convolution layers. Instead, we use a soft hypergraph constructed over frozen token embeddings to perform structured adaptation within a bottleneck module. 
To our knowledge, this is the first work that formulates PEFT as hyperedge-level adaptation in a structured interaction space. By combining hypergraph modeling with PEFT, HyperAdapter bridges relational representation learning and efficient model adaptation in a unified and modular framework.

\begin{figure}[tbp]
  \centering
  \includegraphics[width=0.90\linewidth, keepaspectratio]{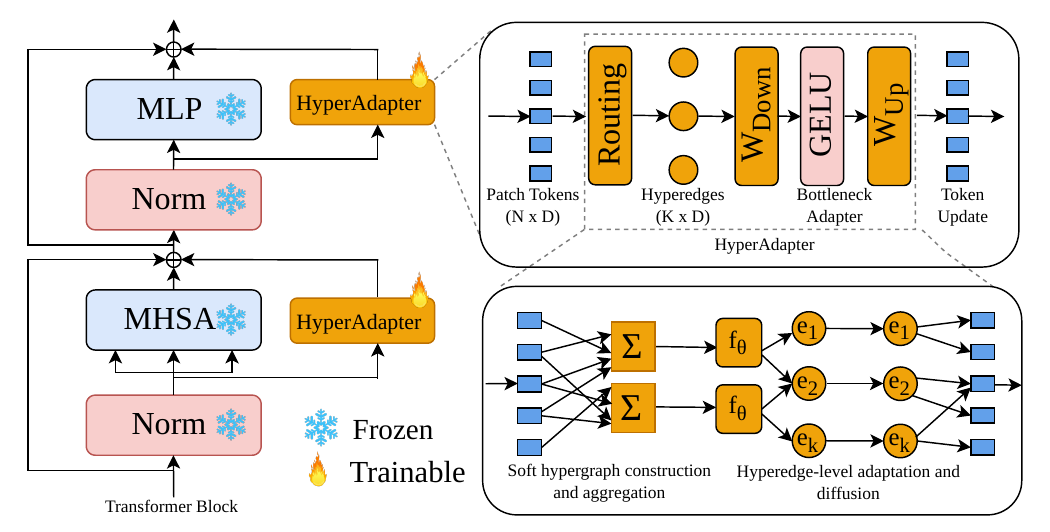} 
  \vspace{-0.3cm}
  \caption{HyperAdapter framework. We introduce HyperAdapter, a parameter-efficient adaptation module that operates in a structured interaction space. Given patch tokens from a frozen vision transformer, a routing mechanism softly groups tokens into hyperedges, capturing higher-order relationships among visually related regions. Each hyperedge aggregates information from multiple tokens and is refined using a lightweight bottleneck adapter, enabling group-level adaptation rather than independent token updates. The refined hyperedge features are then diffused back to tokens, producing structured and coherent feature updates. By shifting adaptation from individual tokens to token groups, HyperAdapter injects relational inductive bias while maintaining the efficiency and modularity of standard PEFT methods. 
  }
  \label{fig:pipeline}
  \vspace{-0.5cm}
\end{figure}

\section{Method}



We now introduce our method. We begin by revisiting the adaptation space. 

\subsection{Revisiting the Adaptation Space}

Let $f_{\vtheta}$ denote a pretrained ViT with frozen parameters $\vtheta$.
For an input image, the transformer produces token embeddings $\mX \!=\! [\vx_{\mathrm{cls}}, \vx_1, \ldots, \vx_N] \!\in\! \mathbb{R}^{(N+1)\times D}$,
where $\vx_i \!\in\! \mathbb{R}^{D}$ are patch tokens and $D$ is the hidden dimension.
In standard 
PEFT, each token is refined independently via a bottleneck transformation:
\begin{equation}
\Delta \vx_i = \mW_\text{up}\sigma(\mW_\text{down} \vx_i), \qquad r \ll D.
\end{equation}
This formulation implicitly assumes that adaptation occurs in \emph{token space}, where each token is updated in isolation.
We argue that this token-wise adaptation is inherently limiting. Visual tokens often correspond to coherent regions such as objects, object parts, or semantic components, naturally forming higher-order structures. Ignoring these relationships can lead to redundant updates, spatially inconsistent refinements, and underutilization of relational information encoded in the visual scene.
To overcome this limitation, we propose performing adaptation in a \emph{structured interaction space}, where groups of related tokens are refined jointly. By shifting the adaptation focus from individual tokens to token groups, this perspective enables parameter-efficient updates that respect the intrinsic structure of visual inputs, yielding more coherent and semantically consistent feature refinements. We next describe our hyperedge-space adaptation.

\subsection{Hyperedge-Space Adaptation}

We introduce HyperAdapter (Fig. \ref{fig:pipeline}), a parameter-efficient adaptation mechanism that operates in \emph{hyperedge space} rather than directly on individual tokens. Unlike conventional token-wise adapters, which refine each token independently, HyperAdapter first organizes tokens into structured groups and performs adaptation at the group level. This design introduces an explicit relational inductive bias, enabling updates that respect higher-order relationships among tokens while preserving the modularity and efficiency of standard adapters.

Formally, let $\mX_p = [\vx_1, \ldots, \vx_N] \in \mathbb{R}^{N \times D}$ denote the patch token embeddings produced by the frozen ViT. To capture structured interactions among tokens, we introduce $K$ learnable prototype vectors
\begin{equation}
\mE = [\ve_1, \ldots, \ve_K]^\top \in \mathbb{R}^{K \times D},
\end{equation}
which serve as latent hyperedges representing compact groups of tokens.\footnote{The learnable prototypes are Xavier-initialized and jointly optimized with adapter parameters.} Each token is softly assigned to these hyperedges based on representation similarity using temperature-scaled cosine routing:
\begin{equation}
\mM_{ik} = \frac{\exp\left(\langle \hat{\vx}_i, \hat{\ve}_k \rangle / \tau \right)}{\sum_{j=1}^{K} \exp\left(\langle \hat{\vx}_i, \hat{\ve}_j \rangle / \tau \right)},
\end{equation}
where $\mM \in \mathbb{R}^{N \times K}$ is the soft incidence matrix and $\tau>0$ controls the sharpness of assignments. This construction effectively models token-to-hyperedge relationships, allowing each token to contribute to multiple groups in a differentiable manner.
Once the hypergraph is defined, each hyperedge representation is computed as a normalized weighted average of the tokens assigned to it:
\begin{equation}
\mH = (\mM^\top \mX_p) \oslash (\mM^\top \mathbf{1}), \qquad \mH \in \mathbb{R}^{K \times D},
\end{equation}
where $\oslash$ denotes row-wise normalization. Each hyperedge thus aggregates information from multiple related tokens, forming a compact group-level representation that captures higher-order context.
Adaptation is performed in hyperedge space using a lightweight bottleneck module:
\begin{equation}
\Delta \mH = \mW_\text{up}\sigma(\mW_\text{down} \mH),
\end{equation}
where $\mW_\text{down} \!\in\! \mathbb{R}^{r \times D}$ and $\mW_\text{up} \!\in\! \mathbb{R}^{D \times r}$. Operating on hyperedges rather than individual tokens enables 
group-level feature refinement, allowing relational information to propagate across tokens within each group.
Finally, the hyperedge updates are diffused back to token space through the soft incidence matrix:
\begin{equation}
\Delta \mX = \mM \Delta \mH, \qquad \mX_p' = \mX_p + \alpha \Delta \mX,
\end{equation}
where $\alpha$ is a learnable scaling parameter. This diffusion ensures that each token receives structured updates informed by the hyperedges it belongs to. The classification token remains unchanged throughout this process.

HyperAdapter introduces structured, group-aware adaptation while maintaining the efficiency and modularity of standard PEFT methods. It generalizes token-wise adapters, which are recovered as a special case when each hyperedge contains a single token. By refining groups of tokens jointly, HyperAdapter injects an explicit relational inductive bias into parameter-efficient adaptation, leading to more coherent and semantically consistent feature updates.

\subsection{Unified View of HyperAdapter Generalization}

HyperAdapter provides a unified view of parameter-efficient adaptation by generalizing standard token-wise adapters. In the limiting case where each token forms an independent hyperedge with no cross-token aggregation, the method reduces exactly to a token-wise adapter.




\begin{proposition}[Token-wise adapter]
If the number of hyperedges equals the number of tokens ($K = N$) and the soft incidence matrix is the identity ($\mM = \mI_N$), HyperAdapter reduces to a standard token-wise adapter.
\end{proposition}

\begin{proof}
When $\mM = \mI_N$, the hyperedge aggregation step yields $\mH = \mM^\top \mX_p = \mX_p$,
and the diffusion step gives $\Delta \mX = \mM \Delta \mH = \Delta \mH$.

Thus, each token is independently refined:
\[
\vx_i' = \vx_i + \Delta \vx_i = \vx_i + \mW_{\text{up}}\,\sigma(\mW_{\text{down}} \vx_i),
\]
which exactly recovers the standard token-wise adapter.
\end{proof}


We now show that HyperAdapter treats patch tokens symmetrically: reordering the tokens results in a correspondingly reordered output, leaving the computations invariant.

\begin{proposition}[Permutation equivariance]
HyperAdapter is permutation equivariant with respect to patch tokens.
\end{proposition}

\begin{proof}
Let $\mP_\pi \in \mathbb{R}^{N \times N}$ denote a permutation matrix corresponding to a reordering $\pi$ of the $N$ tokens. Cosine-based routing preserves inner products under permutation, so
\[
\mM(\mP_\pi \mX_p) = \mP_\pi \mM(\mX_p).
\]

Hyperedge aggregation satisfies
\[
(\mP_\pi \mM)^\top (\mP_\pi \mX_p) = \mM^\top \mX_p,
\]
ensuring that hyperedge representations remain unchanged. Since the adapter operates independently on hyperedges, $\Delta \mH$ is unaffected. Diffusion back to token space then gives
\[
\Delta \mX(\mP_\pi \mX_p) = (\mP_\pi \mM) \Delta \mH = \mP_\pi (\mM \Delta \mH) = \mP_\pi \Delta \mX(\mX_p),
\]
establishing permutation equivariance.
\end{proof}
 
We next highlight that HyperAdapter updates lie in a low-dimensional subspace, reflecting the bottleneck structure of the adapter while preserving efficiency and structured group-level adaptation.

\begin{proposition}[Low-rank adaptation structure]
The token-level update induced by HyperAdapter lies in a low-dimensional subspace of rank at most $\min(K,r)$, where $r$ is the bottleneck dimension of the adapter.
\end{proposition}

\begin{proof}
By construction, the token update is
\[
\Delta \mX = \mM \Delta \mH, \quad \Delta \mH = \mW_{\text{up}}\, \sigma(\mW_{\text{down}} \mH) \in \mathbb{R}^{K \times D}.
\]
Since $\mW_{\text{down}} \in \mathbb{R}^{r \times D}$, the rank of $\Delta \mH$ is at most $r$. Using the submultiplicativity of matrix rank:
\[
\mathrm{rank}(\Delta \mX) = \mathrm{rank}(\mM \Delta \mH) \le \min(\mathrm{rank}(\mM), \mathrm{rank}(\Delta \mH)) \le \min(K,r).
\]
This shows that HyperAdapter preserves the low-rank nature of parameter-efficient adaptation while enabling structured group-level refinement.
\end{proof} 

\textbf{Structured smoothing interpretation.} The token-level update can be interpreted as a structured smoothing process over hyperedges, where information is first aggregated, transformed through a low-rank adapter, and then diffused back to tokens. The token-level update can be written as
\begin{equation}
\Delta \mX = \mM \mW_{\text{up}}\, \sigma(\mW_{\text{down}} (\mM^\top \mX_p)).
\end{equation}
Viewed this way, HyperAdapter performs a sequence of operations:
Token \\$\xrightarrow{\text{Hyperedge projection}}$ Hyperedge features $\xrightarrow{\text{Low-Rank adapter}}$ Hyperedge updates \\$\xrightarrow{\text{Structured diffusion}}$ Token updates.

This perspective reveals HyperAdapter as a learnable structured smoothing operator in feature space: tokens assigned to the same hyperedges share updates, producing coherent, group-level feature refinements. Standard token-wise adapters correspond to the degenerate case where each hyperedge contains only a single token, reducing the smoothing operator to the identity mapping.

\textbf{Distinction from existing PEFT methods.} HyperAdapter departs from conventional PEFT in two key ways. First, unlike standard adapters or LoRA that update tokens independently, HyperAdapter performs \emph{group-level adaptation} in a learned structured interaction space, capturing higher-order relationships among tokens. Second, unlike graph-based transformers that modify the backbone attention mechanism, HyperAdapter leaves the pretrained ViT untouched and introduces a modular hyperedge adapter. This design simultaneously achieves structured adaptation, low-rank updates, and permutation equivariance, all while maintaining minimal computational and parameter overhead.

\begin{table}[tbp]
\caption{VTAB-1K performance comparison of PEFT methods using a ViT-B/16 backbone. HyperAdapter achieves the highest average accuracy (77.6\%), demonstrating consistent improvements across \textit{Natural}, \textit{Specialized}, and \textit{Structured} tasks. \# Param(M) denotes the number of trainable parameters.}
\centering
\footnotesize
\setlength{\tabcolsep}{3pt}
\renewcommand{\arraystretch}{1.1}
\resizebox{\textwidth}{!}{%
\begin{tabular}{lc|rrrrrrr|rrrr|rrrrrrrr|r}
\hline
 &  & \multicolumn{7}{c|}{\textbf{Natural}} 
 & \multicolumn{4}{c|}{\textbf{Specialized}} 
 & \multicolumn{8}{c|}{\textbf{Structured}} 
 &  \\
\cline{3-21}

& \rotatebox{90}{\# Param (M)}
& \rotatebox{90}{Cifar100}
& \rotatebox{90}{Caltech101}
& \rotatebox{90}{DTD}
& \rotatebox{90}{Flower102}
& \rotatebox{90}{Pets}
& \rotatebox{90}{SVHN}
& \rotatebox{90}{Sun397}
& \rotatebox{90}{Camelyon}
& \rotatebox{90}{EuroSAT}
& \rotatebox{90}{Resisc45}
& \rotatebox{90}{Retinopathy}
& \rotatebox{90}{Clevr-Count}
& \rotatebox{90}{Clevr-Dist}
& \rotatebox{90}{DMLab}
& \rotatebox{90}{KITTI-Dist}
& \rotatebox{90}{dSpr-Loc}
& \rotatebox{90}{dSpr-Ori}
& \rotatebox{90}{sNORB-Azim}
& \rotatebox{90}{sNORB-Ele}
& \rotatebox{90}{Average} \\
\hline

\multicolumn{22}{l}{\textit{Traditional Finetuning}} \\
\hline

Full fine-tuning
& 85.8
& 68.9 & 87.7 & 64.3 & 97.2 & 86.9 & 87.4 & 38.8
& 79.7 & 95.7 & 84.2 & 73.9
& 56.3 & 58.6 & 41.7 & 65.5 & 57.5 & 46.7 & 25.7 & 29.1
& 68.9 \\

Linear probing
& 0
& 64.4 & 85.0 & 63.2 & 97.0 & 86.3 & 36.6 & 51.0
& 78.5 & 87.5 & 68.5 & 74.0
& 34.3 & 30.6 & 33.2 & 55.4 & 12.5 & 20.0 & 9.6 & 19.2
& 57.6 \\

\hline
\multicolumn{22}{l}{\textit{PEFT methods}} \\
\hline
BitFit\cite{bitfit}
& 0.10
& 72.8 & 87.0 & 59.2 & 97.5 & 85.3 & 59.9 & 51.4
& 78.7 & 91.6 & 72.9 & 69.8
& 61.5 & 55.6 & 32.4 & 55.9 & 66.6 & 40.0 & 15.7 & 25.1
& 65.2 \\

VPT-Shallow \cite{vpt}
& 0.06
& 77.7 & 86.9 & 62.6 & 97.5 & 87.3 & 74.5 & 51.2
& 78.2 & 92.0 & 75.6 & 72.9
& 50.5 & 58.6 & 40.5 & 67.1 & 68.7 & 36.1 & 20.2 & 34.1
& 67.8 \\

VPT-Deep \cite{vpt}
& 0.53
& 78.8 & 90.8 & 65.8 & 98.0 & 88.3 & 78.1 & 49.6
& 81.8 & 96.1 & 83.4 & 68.4
& 68.5 & 60.0 & 46.5 & 72.8 & 73.6 & 47.9 & 32.9 & 37.8
& 72.0 \\

E$^2$VPT \cite{e2vpt}
& 0.25
& 78.6 & 89.4 & 67.8 & 98.2 & 88.5 & 85.3 & 52.3
& 82.5 & 96.8 & 84.8 & 73.6
& 71.7 & 61.2 & 47.9 & 75.8 & 80.8 & 48.1 & 31.7 & 41.9
& 73.9 \\

Adapter \cite{adapter}
& 0.16
& 69.2 & 90.1 & 68.0 & 98.8 & 89.9 & 82.8 & 54.3
& 84.0 & 94.9 & 81.9 & 75.5
& 80.9 & 65.3 & 48.6 & 78.3 & 74.8 & 48.5 & 29.9 & 41.6
& 73.9 \\

AdaptFormer \cite{adaptformer}
& 0.16
& 70.8 & 91.2 & 70.5 & 99.1 & 90.9 & 86.6 & 54.8
& 83.0 & 95.8 & 84.4 & 76.3
& 81.9 & 64.3 & 49.3 & 80.3 & 76.3 & 45.7 & 31.7 & 41.1
& 74.7 \\

Convpass \cite{convpass}
& 0.33
& 72.3 & 91.2 & 72.2 & 99.2 & 90.9 & 91.3 & 54.9
& 84.2 & 96.1 & 85.3 & 75.6
& 82.3 & 67.9 & 51.3 & 80.0 & 85.9 & 53.1 & 36.4 & 44.4
& 74.5 \\

ARC \cite{arc}
& 0.13
& 72.2 & 90.1 & 72.7 & 99.0 & 91.0 & 91.9 & 54.4
& 84.9 & 95.7 & 86.7 & 75.8
& 80.7 & 67.1 & 48.7 & 81.6 & 79.2 & 51.0 & 31.4 & 39.9
& 75.8 \\

LoRA \cite{lora}
& 0.29
& 67.1 & 91.4 & 69.4 & 98.8 & 90.4 & 85.3 & 54.0
& 84.9 & 95.3 & 84.4 & 73.6
& 82.9 & 69.2 & 49.8 & 78.5 & 75.7 & 47.1 & 31.0 & 44.0
& 74.5 \\

NOAH \cite{noah}
& 0.36
& 69.6 & 92.7 & 70.2 & 99.1 & 90.4 & 86.1 & 53.7
& 84.4 & 95.4 & 83.9 & 75.8
& 82.8 & 68.9 & 49.9 & 81.7 & 81.8 & 48.3 & 32.8 & 44.2
& 75.5 \\

FacT \cite{fact}
& 0.07
& 70.6 & 90.6 & 70.8 & 99.1 & 90.7 & 88.6 & 54.1
& 84.8 & 96.2 & 84.5 & 75.7
& 82.6 & 68.2 & 49.8 & 80.7 & 80.8 & 47.4 & 33.2 & 43.0
& 75.6 \\

SSF \cite{ssf}
& 0.24
& 69.0 & 92.6 & 75.1 & 99.4 & 91.8 & 90.2 & 52.9
& 87.4 & 95.9 & 87.4 & 75.5
& 75.9 & 62.3 & 53.3 & 80.6 & 77.3 & 54.9 & 29.5 & 37.9
& 75.7 \\

RepAdapter \cite{repadapter}
& 0.22
& 72.4 & 91.6 & 71.0 & 99.2 & 91.4 & 90.7 & 55.1
& 85.3 & 95.9 & 84.6 & 75.9
& 82.3 & 68.0 & 50.4 & 79.9 & 80.4 & 49.2 & 38.6 & 41.0
& 76.1 \\

Res-Tuning  \cite{repadapter}
& 0.55
& 75.2 & 92.7 & 71.9 & 99.3 & 91.9 & 86.7 & 58.5
& 86.7 & 95.6 & 85.0 & 74.6
& 80.2 & 63.6 & 50.6 & 80.2 & 85.4 & 55.7 & 31.9 & 42.0
& 76.3 \\

\textbf{HyperAdapter (Ours)}
& 0.44
& \textbf{74.1} & 93.3 & 72.8 & 99.3 & 91.7 & 88.3 & 56.7
& 87.5 & 96.3 & 86.0 & 76.5
& 84.0 & 64.4 & 55.2 & 83.9 & 88.5 & 54.6 & 36.3 & 43.7
& \textbf{77.6} \\

\hline
\end{tabular}
}
\label{tab:vtab_results}
\vspace{-0.5cm}
\end{table}

\section{Experiment}


\subsection{Experimental Setup}

\textbf{Datasets.} We evaluate HyperAdapter on two visual adaptation benchmarks. \textit{VTAB-1K} \cite{vtab} consists of 19 classification tasks grouped into three categories: \emph{Natural}, \emph{Specialized}, and \emph{Structured}. Natural tasks contain real-world photographs captured with standard cameras. Specialized tasks include images from domain-specific sensors such as remote sensing and medical imaging. Structured tasks primarily involve synthetically generated images that test reasoning about scene structure and exhibit significant domain shift from natural images. Following the VTAB-1K protocol, each task provides 800 training samples and 200 validation samples, while the test sets follow the sizes of the original datasets. 
To evaluate performance in low-data regimes, we conduct experiments on five Few-shot fine-grained visual classification (\textit{FGVC}) datasets: FGVC-Aircraft \cite{aircraft}, Oxford Pets, Food-101 \cite{food101}, Stanford Cars \cite{cars}, and Oxford Flowers102 \cite{flower}. We report results under 1-, 2-, 4-, 8-, and 16-shot settings.

\textbf{Implementation details.}
We build upon a pretrained ViT-B/16 \cite{vit} backbone initialized with ImageNet-21k \cite{imagenet} weights. Unless otherwise specified, the backbone parameters are frozen, and only the classification head and HyperAdapter modules are trained.
Each HyperAdapter operates in hyperedge space with bottleneck rank $r\!=\!8$ and $K\!=\!8$ hyperedges. The routing temperature $\tau$ controls the softness of token-to-hyperedge assignments and is selected 
based on validation performance. 
Unless otherwise stated, $K$ is fixed across datasets, while $\tau$ is tuned. We use AdamW with weight decay $1\!\times\!10^{-4}$. The learning rate is $1\!\times\!10^{-3}$ for VTAB-1K and $5\!\times\!10^{-3}$ for FGVC tasks.. We use a cosine learning rate schedule with 10 warmup epochs and train for 100 epochs. The batch size is 64 and the input resolution is $224 \times 224$. All experiments are conducted on a single GPU.
Following prior work \cite{vpt,convpass,fact}, hyperparameters are tuned on the validation split, and results are reported as the mean over three runs.

Below, we present key evaluations, with further results in the \textbf{Appendix}.

\begin{figure}[tbp]
  \centering
  \includegraphics[width=\linewidth]{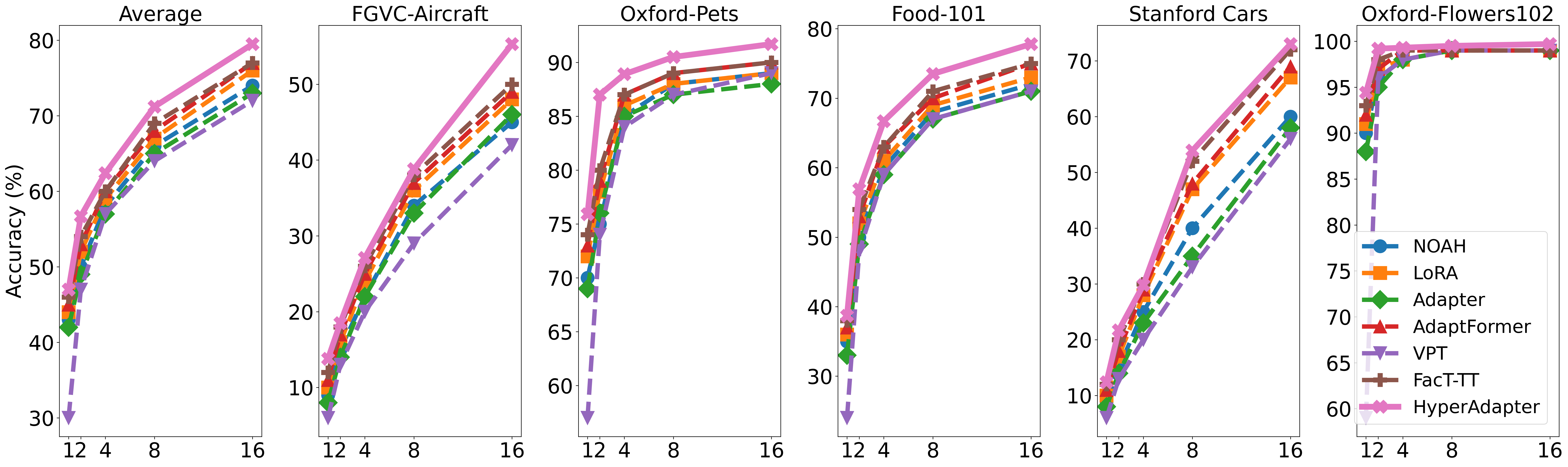} 
  \caption{Top-1 accuracy on few-shot FGVC benchmarks using ViT-B/16. HyperAdapter consistently outperforms prior PEFT methods, with the largest gains observed in the ultra-low-shot (1-4) regime.}
  \label{fig:fgvc}
  \vspace{-0.5cm}
\end{figure}

\subsection{Comparison to State-of-the-Art Methods}

\textbf{Results on VTAB-1K.} Table \ref{tab:vtab_results} summarizes performance across the \textit{Natural}, \textit{Specialized}, and \textit{Structured} task categories. Full fine-tuning achieves 68.9\% average accuracy, while linear probing lags at 57.6\%, highlighting the limitations of frozen-feature baselines in low-data regimes. Most PEFT methods substantially improve over full fine-tuning, with strong baselines such as Res-Tuning (76.3\%), RepAdapter (76.1\%), SSF (75.7\%), and LoRA (74.5\%) demonstrating the effectiveness of lightweight adaptation.
\textbf{HyperAdapter consistently outperforms all prior PEFT methods}, reaching 77.6\% average accuracy. The improvements are especially pronounced on \textit{Structured} tasks, \eg, KITTI-Dist, dSpr-Loc, and sNORB, reflecting HyperAdapter’s ability to capture higher-order spatial and relational dependencies among tokens. Performance gains are also maintained across \textit{Natural} and \textit{Specialized} categories, indicating robust generalization across diverse visual domains. 
These results illustrate that HyperAdapter’s structured, hyperedge-based adaptation provides a more expressive and efficient mechanism than traditional token-wise PEFT, enabling consistent gains across tasks with varying domain characteristics and data complexities.


\textbf{Few-shot fine-grained visual recognition.}  
Fig.~\ref{fig:fgvc} shows few-shot performance (1, 2, 4, 8, 16 shots) on five fine-grained benchmarks: \textit{FGVC-Aircraft}, \textit{Oxford-Pets}, \textit{Food-101}, \textit{Stanford Cars}, and \textit{Oxford-Flowers102}, along with the overall average. Performance increases consistently with more labeled samples, but the improvement is most critical in the ultra-low-shot regime (1-4 shots). HyperAdapter achieves the best or near-best results across nearly all datasets and shot settings, with particularly strong gains in the 1-4 shot regime, highlighting its ability to leverage limited supervision while preserving transferable representations from the pretrained backbone.

\begin{table}[tbp]
\centering
\begin{minipage}{0.49\textwidth}
\centering
\caption{VTAB-1k results with ViT-L/16 (ImageNet-21K). HyperAdapter achieves the highest average accuracy (77.7\%), with strong gains on \textit{Structured} tasks, showing effective hyperedge-based token grouping and parameter-efficient adaptation.}
\label{tab:vit_large}
\resizebox{\linewidth}{!}{
\begin{tabular}{lccccc}
\toprule
\textbf{Method} & \textbf{Natural} & \textbf{Specialized} & \textbf{Structured} & \textbf{Average} & \textbf{Params (M)} \\
\midrule
Full fine-tuning & 74.7 & 83.8 & 48.1 & 68.9 & 303.40 \\
Linear probing   & 70.9 & 69.1 & 25.8 & 55.3 & 0.05 \\
\midrule
Adapter          & 68.6 & 73.5 & 29.0 & 57.0 & 2.38 \\
VPT-Deep         & 82.5 & 83.9 & 54.1 & 73.5 & 0.49 \\
\midrule
ARC             & 82.3 & 85.6 & 57.3 & 75.1 & 0.74 \\
RepAdapter       & 84.0 & 86.3 & 60.1 & 76.8 & 0.79 \\
\midrule
\textbf{HyperAdapter (Ours)}   & \textbf{83.7} & \textbf{85.6} & \textbf{63.8} & \textbf{77.7} & 1.18 \\
\bottomrule
\end{tabular}
}
\end{minipage}
\hfill
\begin{minipage}{0.49\textwidth}
\centering
\caption{VTAB-1k results with Swin-Base (ImageNet-21K). HyperAdapter achieves top average accuracy (77.6\%), particularly improving \textit{Structured} tasks, showing hyperedge routing's effectiveness across transformer architectures.}
\label{tab:swin_base}
\resizebox{\linewidth}{!}{
\begin{tabular}{lccccc}
\toprule
\textbf{Method} & \textbf{Natural} & \textbf{Specialized} & \textbf{Structured} & \textbf{Average} & \textbf{Params (M)} \\
\midrule
Full fine-tuning & 79.1 & 86.2 & 59.7 & 75.0 & 86.90 \\
Linear probing   & 73.5 & 80.8 & 33.5 & 62.6 & 0.05\\
\midrule
VPT-Shallow         & 79.9 & 82.4 & 37.8 & 66.7 & 0.05 \\
VPT-Deep         & 76.8 & 84.5 & 53.4 & 71.6 & 0.22\\
\midrule
ARC             & 79.0 & 86.6 & 59.9 & 75.6 &0.16 \\
RepAdapter       & 82.8 & 87.2 & 61.2 & 77.0 & 0.39 \\
\midrule
\textbf{HyperAdapter (Ours)}   & \textbf{83.5} & \textbf{86.2} & \textbf{63.0} & \textbf{77.6} & 0.60 \\
\bottomrule
\end{tabular}
}
\end{minipage}
\vspace{-0.5cm}
\end{table}

\textbf{Results on larger ViTs.} Table~\ref{tab:vit_large} compares VTAB-1k performance using a ViT-L/16 pretrained on ImageNet-21K. HyperAdapter achieves the highest average accuracy of 77.7\%, outperforming prior PEFT methods such as ARC and RepAdapter. The gains are most pronounced on \textit{Structured} tasks (63.8\%), highlighting the effectiveness of hyperedge-based token grouping for capturing complex spatial and relational dependencies. With only 1.18M trainable parameters (under 0.5\% of the backbone), HyperAdapter improves performance across all task categories, demonstrating both expressiveness and parameter efficiency.

\textbf{Results on hierarchical Swin transformers.} Table~\ref{tab:swin_base} reports VTAB-1k performance using a Swin-Base backbone. HyperAdapter achieves the top average accuracy of 77.6\%, outperforming strong PEFT baselines including ARC and RepAdapter. Structured tasks again see the largest gains (63.0\%), confirming that hyperedge routing effectively models structural relationships in hierarchical architectures. Despite introducing only 0.60M trainable parameters, HyperAdapter consistently boosts performance across Natural, Specialized, and Structured tasks, illustrating its generality and architecture-agnostic adaptability.

\subsection{Ablation Study}

To analyze HyperAdapter’s key properties, we perform detailed ablation studies on VTAB-1k using a pretrained ViT-B/16.


\textbf{HyperAdapter \vs token-wise adapter.} Table~\ref{tab:baseline_ablation} compares HyperAdapter to a token-wise adapter baseline that retains the same adapter capacity but removes hyperedge routing, isolating the impact of structured adaptation. 
HyperAdapter consistently improves performance across all VTAB categories, with gains of +0.3 on Natural, +1.7 on Specialized, and +1.9 on Structured tasks, raising the overall average from 76.3\% to 77.6\% (+1.3\%). 

\begin{wraptable}{r}{0.6\columnwidth}
\vspace{-1.2cm}
\centering
\caption{Impact of hyperedge-level adaptation. HyperAdapter adds hypergraph routing to the adapter while keeping a similar parameter budget, yielding consistent gains across all VTAB-1k task categories. 
}
\label{tab:baseline_ablation}
\resizebox{\linewidth}{!}{
\begin{tabular}{lccccccc}
\toprule
Method & Params (M) &Natural& Specialized& Structure& Average (\%) & Gain \\
\midrule
Baseline & 0.29 &82.0 &84.9 &61.9 &76.3 & -- \\
HyperAdapter &0.44  & 82.3& 86.6&63.8 &77.6 & +1.3 \\
\bottomrule
\end{tabular}}
\vspace{-0.8cm}
\end{wraptable}

These improvements arise from hyperedge routing mechanism rather than increased parameters, as it enables tokens to be grouped and adapted jointly. By aggregating information across hyperedges, HyperAdapter captures higher-order relationships, producing more discriminative features. The largest gains on Specialized and Structured tasks highlight its effectiveness for fine-grained and spatially structured visual patterns.

\begin{wrapfigure}{r}{0.6\textwidth}
\vspace{-0.8cm}
\centering
\begin{subfigure}[t]{0.49\linewidth}
\centering
\includegraphics[width=\linewidth]{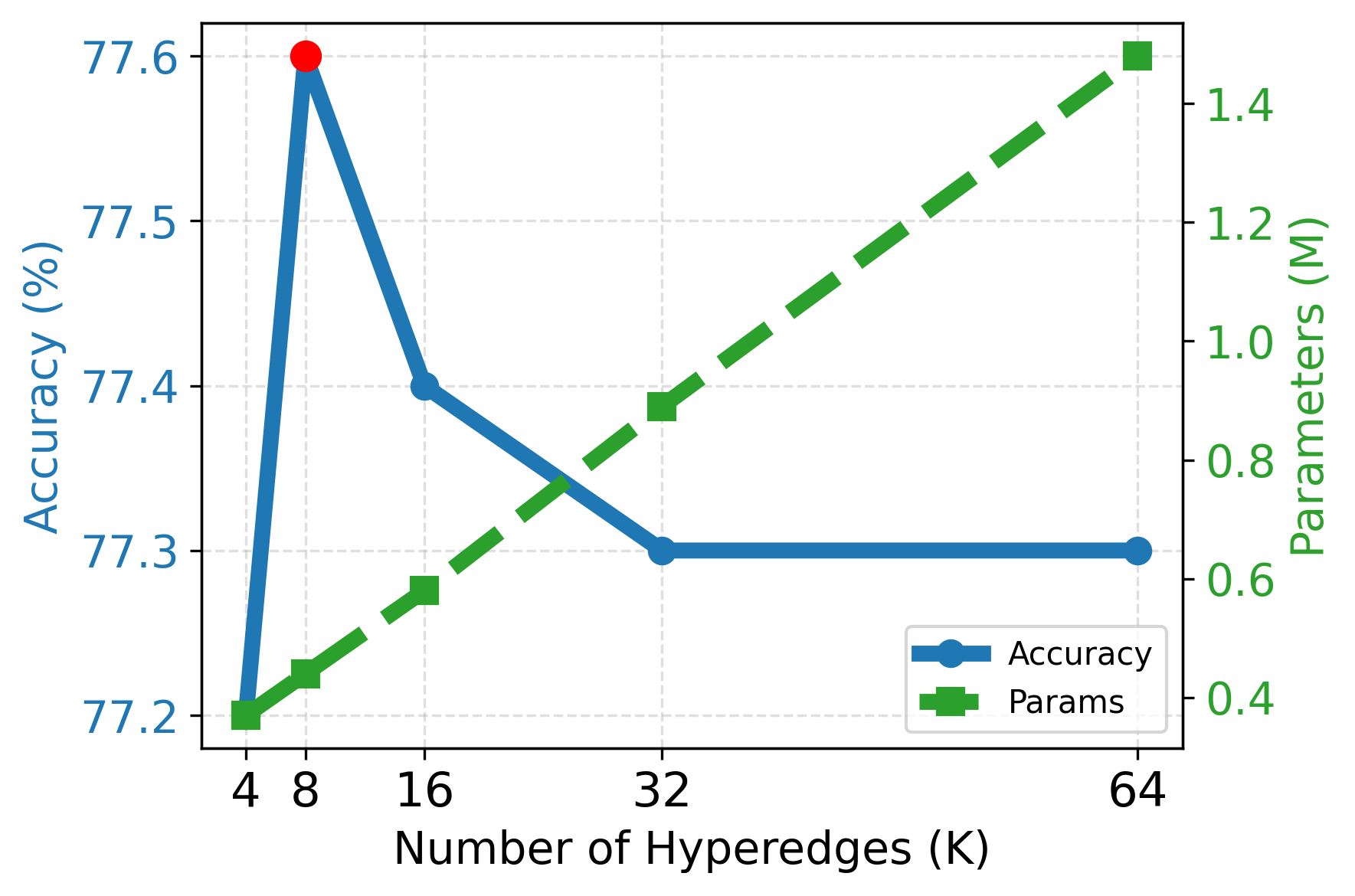}
\caption{\label{fig:K_sensitivity}}
\end{subfigure}
\hfill
\begin{subfigure}[t]{0.49\linewidth}
\centering
\includegraphics[width=\linewidth]{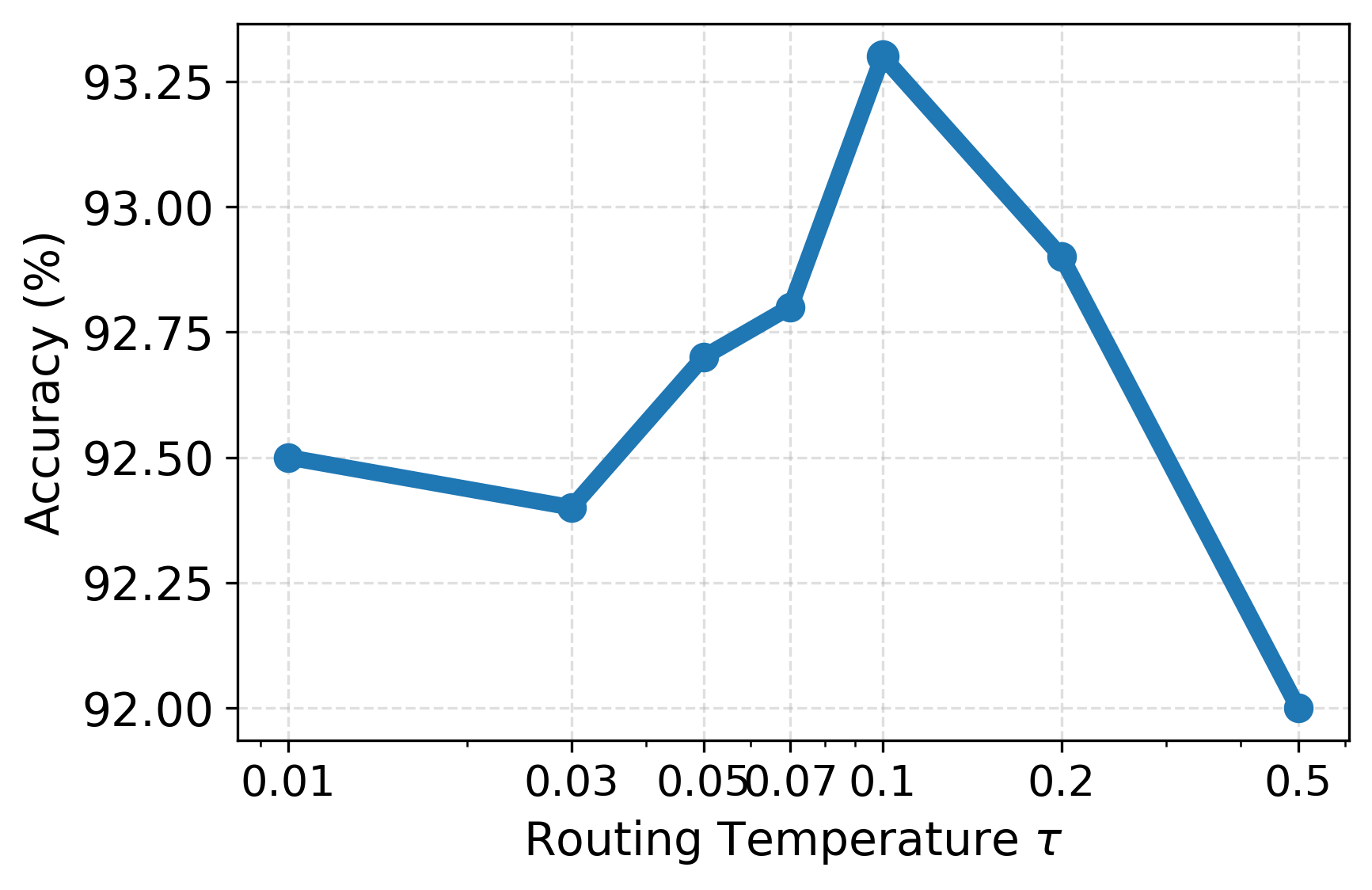}
\caption{\label{fig:tau_sensitivity}}
\end{subfigure}
\vspace{-0.2cm}
\caption{(a) Effect of the number of hyperedges $K$ on VTAB-1k, showing that a moderate $K=8$ balances model capacity and efficiency. (b) Sensitivity to routing temperature $\tau$ on Caltech101, where performance peaks at $\tau=0.10$, indicating optimal hyperedge assignment at moderate routing sharpness.}
\label{fig:routing_ablation}
\vspace{-0.8cm}
\end{wrapfigure}

\textbf{Impact of the number of hyperedges.}  
Fig.~\ref{fig:K_sensitivity} shows how varying the number of hyperedges $K$ affects performance and parameter count. Accuracy improves from 77.2\% at $K=4$ to a peak of 77.6\% at $K=8$, indicating that a moderate number of hyperedges effectively captures higher-order token relationships. Beyond $K=8$, performance slightly drops to 77.3\%, 
suggesting that too many hyperedges may fragment token groups and reduce structured aggregation benefits. Parameter count grows roughly linearly with $K$, from 0.39M at $K=4$ to 1.50M at $K=64$. Overall, $K=8$ offers the best trade-off, 
and is used as the default in all experiments.


\begin{wraptable}{r}{0.45\columnwidth}
\vspace{-1.2cm}
\centering
\caption{Ablation on adapter placement; parallel injection on Attention + MLP yields highest accuracy. 
}
\label{tab:placement_ablation}
\resizebox{\linewidth}{!}{\begin{tabular}{lccc}
\toprule
Placement & Average (\%) & Params (M) \\
\midrule
Pre (Attn+MLP) & 77.1 & 0.44 \\
Pre (Attn Only) & 76.3 & 0.22 \\
Pre (MLP Only) & 77.0 & 0.22 \\
\midrule
Post (Attn+MLP) & 77.2 & 0.44 \\
Post (Attn Only) & 76.7 & 0.22 \\
Post (MLP Only) & 77.0 & 0.22 \\
\midrule
Parallel (Attn+MLP) & \textbf{77.6} & 0.44 \\
Parallel (Attn Only) & 76.8 & 0.22 \\
Parallel (MLP Only) & 77.0 & 0.22 \\
\bottomrule
\end{tabular}}
\vspace{-0.8cm}
\end{wraptable}

\textbf{Sensitivity to routing temperature.} We analyze the effect of routing temperature $\tau$, which controls the sharpness of token-to-hyperedge assignments. Smaller $\tau$ yields confident, sharp assignments, while larger $\tau$ produces softer, more uniform routing. Fig.~\ref{fig:tau_sensitivity} shows results on Caltech101. Performance peaks at moderate temperatures ($\tau=0.10$). Too small $\tau$ limits information sharing across hyperedges, while too large $\tau$ reduces hyperedge specialization. These results indicate that HyperAdapter is stable across a broad range of $\tau$ values, with optimal performance at moderate routing sharpness.


\textbf{Adapter placement ablation.} Table~\ref{tab:placement_ablation} evaluates different adapter injection strategies (Pre, Post, Parallel) and module locations (Attention, MLP, or both). 
Parallel placement consistently achieves the best performance, reaching 77.6\% average accuracy, as it preserves the residual pathway and allows adapters to provide corrective updates without disrupting pretrained features. Attaching adapters to both Attention and MLP branches yields the highest gains, while single-branch configurations are slightly less effective, MLP-only performs competitively, whereas Attention-only shows lower accuracy. These results indicate that parallel placement with adapters on both branches is the most effective configuration, which we adopt in all experiments.


\begin{figure}[tbp]
\centering
\begin{subfigure}[t]{0.32\textwidth}
\centering
\includegraphics[width=\linewidth]{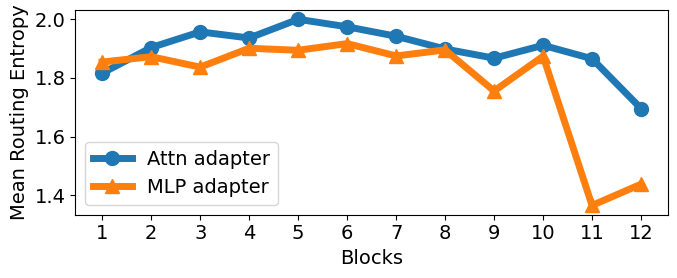}
\caption{CIFAR-100}
\end{subfigure}
\hfill
\begin{subfigure}[t]{0.32\textwidth}
\centering
\includegraphics[width=\linewidth]{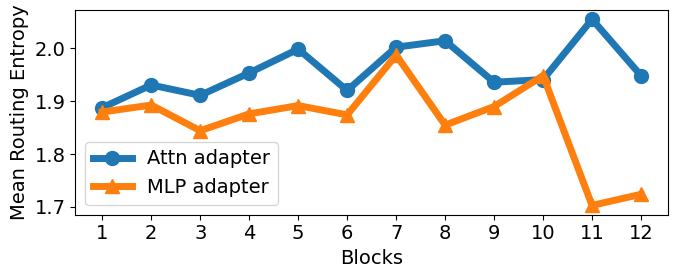}
\caption{EuroSAT}
\end{subfigure}
\hfill
\begin{subfigure}[t]{0.32\textwidth}
\centering
\includegraphics[width=\linewidth]{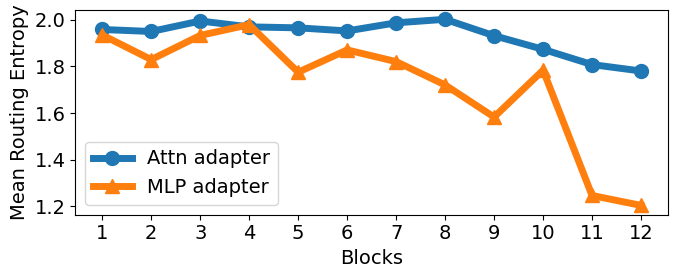}
\caption{KITTI}
\end{subfigure}
\vspace{-0.2cm}
\caption{Token-to-hyperedge routing entropy across transformer layers for CIFAR-100, EuroSAT, and KITTI. Higher entropy indicates more distributed assignments. Attention adapters maintain broader routing, while MLP adapters become increasingly specialized in deeper layers, reflecting progressive feature refinement.}
\label{fig:routing_entropy}
\vspace{-0.2cm}
\end{figure}

\subsection{Analysis and Discussion}

\begin{figure}[tbp]
\centering
\begin{subfigure}[t]{\textwidth}
\centering
\includegraphics[width=\linewidth]{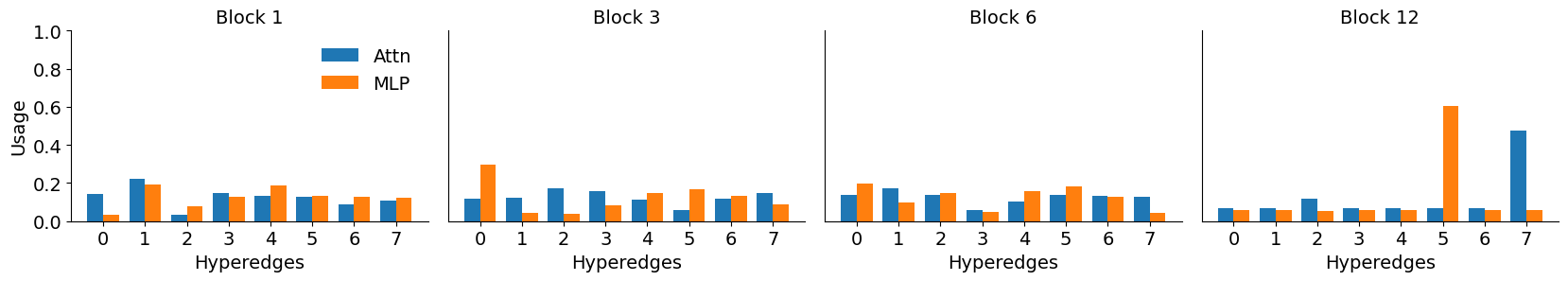}
\caption{CIFAR-100}
\end{subfigure}
\hfill
\begin{subfigure}[t]{\textwidth}
\centering
\includegraphics[width=\linewidth]{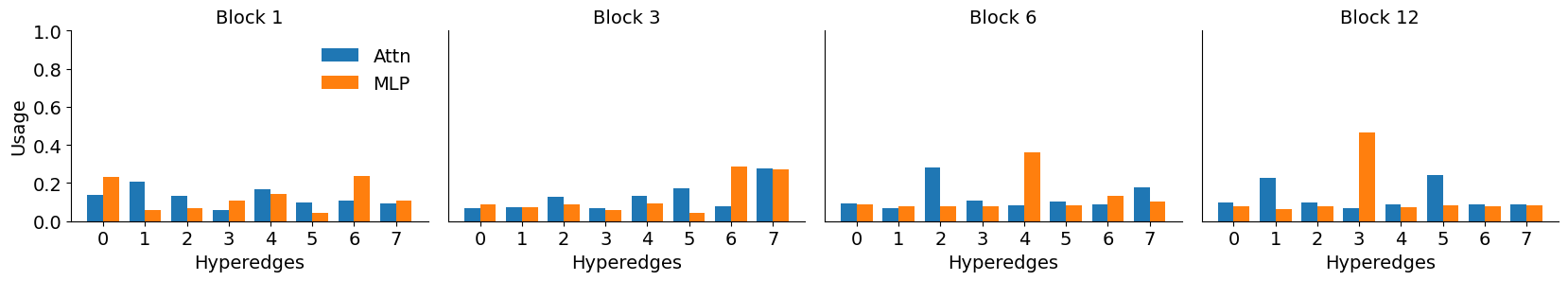}
\caption{EuroSAT}
\end{subfigure}
\hfill
\begin{subfigure}[t]{\textwidth}
\centering
\includegraphics[width=\linewidth]{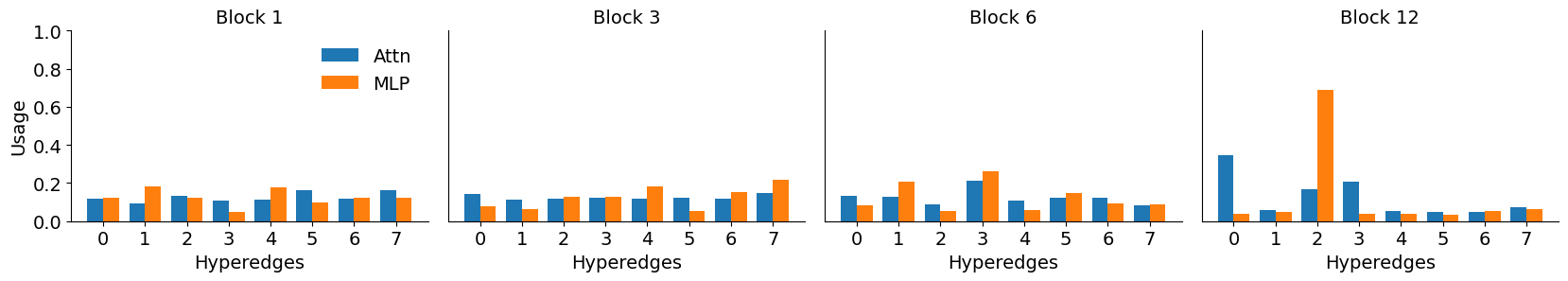}
\caption{KITTI}
\end{subfigure}
\vspace{-0.2cm}
\caption{Hyperedge usage distribution across transformer layers for CIFAR-100, EuroSAT, and KITTI. Early layers use hyperedges more evenly, while deeper layers increasingly concentrate on a subset of hyperedges, indicating progressive specialization.}
\label{fig:hyperedge_usage}
\vspace{-0.5cm}
\end{figure}

\textbf{Routing entropy across layers.} We analyze token-to-hyperedge routing entropy to understand HyperAdapter’s behavior (Fig.~\ref{fig:routing_entropy}). Higher entropy reflects more distributed assignments, while lower entropy indicates confident, specialized routing. Across CIFAR-100, EuroSAT, and KITTI, attention adapters maintain relatively high entropy, promoting flexible token interactions, whereas MLP adapters show decreasing entropy in deeper layers, indicating stronger hyperedge specialization. This trend reveals a progressive routing strategy: early layers explore diverse token-group interactions, and later layers focus on structured, discriminative features, enhancing feature adaptation and representation learning.


\begin{wrapfigure}{r}{0.5\textwidth}
\vspace{-0.8cm}
\centering
\begin{tabular}{cccc}
\tiny\textbf{Original} & \tiny\textbf{Baseline} & \tiny\textbf{AdaptFormer} & \tiny\textbf{Ours} \\

\includegraphics[width=0.24\linewidth, height=1.45cm]{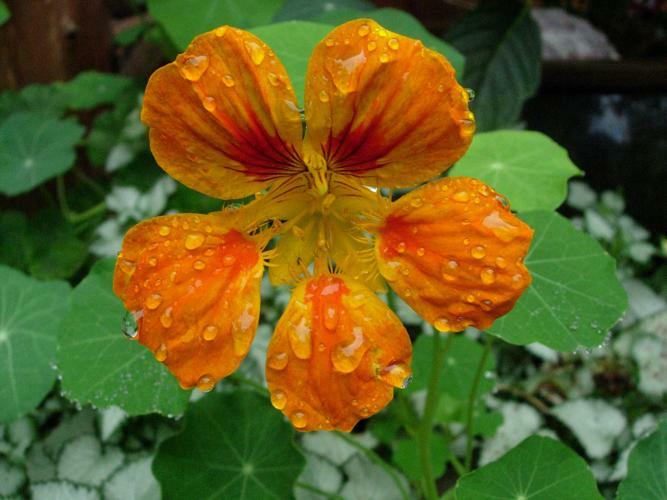} &
\includegraphics[width=0.24\linewidth]{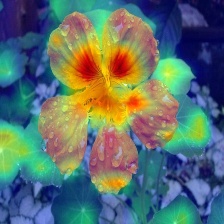} &
\includegraphics[width=0.24\linewidth]{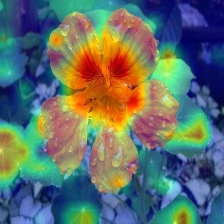} &
\includegraphics[width=0.24\linewidth]{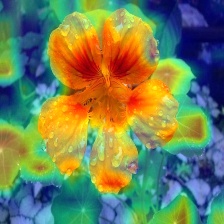} \\
\includegraphics[width=0.24\linewidth, height=1.45cm]{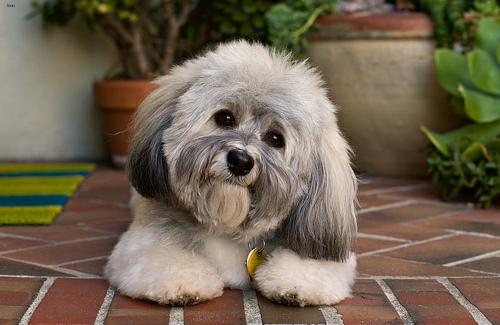} &
\includegraphics[width=0.24\linewidth]{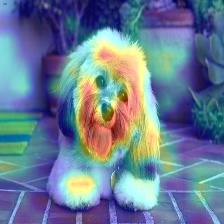} &
\includegraphics[width=0.24\linewidth]{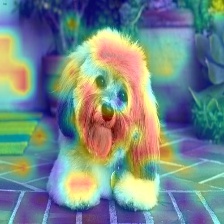} &
\includegraphics[width=0.24\linewidth]{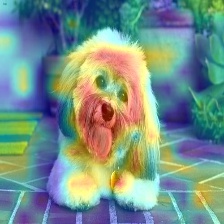} 
\end{tabular}
\vspace{-0.2cm}
\caption{DAAM\cite{daam} visualizations comparing spatial attribution across PEFT methods. Columns show the original image, token-wise baseline, AdaptFormer, and HyperAdapter. HyperAdapter produces more concentrated and semantically aligned activations, highlighting relevant object regions while reducing background noise, reflecting the benefits of hyperedge-based routing.}
\label{fig:daam_visualization}
\vspace{-0.8cm}
\end{wrapfigure}

\textbf{Hyperedge usage across layers.} We examine how HyperAdapter allocates capacity by tracking hyperedge usage frequency across layers (Fig.~\ref{fig:hyperedge_usage}). For selected blocks (1, 3, 6, 12), normalized usage shows that early layers distribute tokens relatively evenly across hyperedges, supporting exploratory feature interactions. In deeper layers, a smaller subset of hyperedges dominates, reflecting specialization to capture structured semantic patterns. Together with the routing entropy analysis, this indicates a progressive adaptation strategy: initial layers explore diverse token-hyperedge interactions, while later layers consolidate and specialize representations for discriminative feature learning.

\textbf{Qualitative visualization.} We analyze HyperAdapter’s effect on spatial attention using DAAM (Fig.~\ref{fig:daam_visualization}). Compared to the baseline and AdaptFormer, which produce diffuse activations extending into background areas, HyperAdapter generates more focused and coherent maps that align with object structures. For example, on flowers, attention concentrates on petals, while on pets, it emphasizes the animal body over surrounding regions. These results indicate that hyperedge routing enables structured token grouping, leading to more localized and semantically consistent feature aggregation.

Further analysis and discussion are provided in the \textbf{Appendix}.

\section{Conclusion}


We presented HyperAdapter, a hypergraph-based parameter-efficient adaptation method for Vision Transformers. By grouping tokens into learnable hyperedges, HyperAdapter captures structured inter-token relationships with minimal additional parameters. Extensive experiments demonstrate consistent improvements over existing PEFT methods, while ablations and visualizations validate the effectiveness of hyperedge routing. Future work will explore dynamic hyperedge structures and extensions to multimodal models.

\section*{Acknowledgments}

Edwin Kwadwo Tenagyei is supported by the Griffith University International Postgraduate Research Scholarship and the Griffith University Postgraduate Research Scholarship. 
Lei Wang conceived the research and led the development of the methodology, while Edwin Kwadwo Tenagyei implemented the method and conducted the experiments.

This work was supported in part by the Australian Research Council (ARC) under Industrial Transformation Research Hub Grant IH180100002. 
This work was also supported by the National Computational Merit Allocation Scheme (NCMAS 2026), with computational resources provided by NCI Australia, an NCRIS-enabled capability supported by the Australian Government.

\bibliographystyle{splncs04}
\bibliography{main}

\begin{thebibliography}{10}
\providecommand{\url}[1]{\texttt{#1}}
\providecommand{\urlprefix}{URL }
\providecommand{\doi}[1]{https://doi.org/#1}

\bibitem{kharao}
Albert, P., Zhang, F.Z., Saratchandran, H., van~den Hengel, A., Abbasnejad, E.:
  Towards higher effective rank in parameter-efficient fine-tuning using
  khatri-rao product. In: Proceedings of the IEEE/CVF International Conference
  on Computer Vision. pp. 1292--1302 (2025)

\bibitem{food101}
Bossard, L., Guillaumin, M., Gool, L.V.: Food-101 - mining discriminative
  components with random forests. In: European Conference on Computer Vision
  (2014)

\bibitem{adaptformer}
Chen, S., Ge, C., Tong, Z., Wang, J., Song, Y., Wang, J., Luo, P.: Adaptformer:
  Adapting vision transformers for scalable visual recognition. Advances in
  Neural Information Processing Systems  \textbf{35},  16664--16678 (2022)

\bibitem{gps-compact}
Chen, T., Chen, J., Zhang, B., Yu, Z., Chen, S., Ye, R., Li, X., Ye, Y.:
  Sensitivity-aware efficient fine-tuning via compact dynamic-rank adaptation.
  In: Proceedings of the Computer Vision and Pattern Recognition Conference.
  pp. 9655--9664 (2025)

\bibitem{householder}
Dong, W., Sun, Y., Yang, Y., Zhang, X., Lin, Z., Yan, Q., Zhang, H., Wang, P.,
  Yang, Y., Shen, H.: Efficient adaptation of pre-trained vision transformer
  via householder transformation. Advances in Neural Information Processing
  Systems  \textbf{37},  102056--102077 (2024)

\bibitem{arc}
Dong, W., Yan, D., Lin, Z., Wang, P.: Efficient adaptation of large vision
  transformer via adapter re-composing. Advances in Neural Information
  Processing Systems  \textbf{36},  52548--52567 (2023)

\bibitem{vit}
Dosovitskiy, A., Beyer, L., Kolesnikov, A., Weissenborn, D., Zhai, X.,
  Unterthiner, T., Dehghani, M., Minderer, M., Heigold, G., Gelly, S.,
  Uszkoreit, J., Houlsby, N.: An image is worth 16x16 words: Transformers for
  image recognition at scale. ArXiv  \textbf{abs/2010.11929} (2020)

\bibitem{feng2019hypergraph}
Feng, Y., You, H., Zhang, Z., Ji, R., Gao, Y.: Hypergraph neural networks. In:
  Proceedings of the AAAI conference on artificial intelligence. vol.~33, pp.
  3558--3565 (2019)

\bibitem{hypergraph}
Fixelle, J.: Hypergraph vision transformers: Images are more than nodes, more
  than edges. In: Proceedings of the IEEE/CVF Conference on Computer Vision and
  Pattern Recognition. pp. 9751--9761 (2025)

\bibitem{3d}
Gao, Y., Wang, M., Tao, D., Ji, R., Dai, Q.: 3-d object retrieval and
  recognition with hypergraph analysis. IEEE transactions on image processing
  \textbf{21}(9),  4290--4303 (2012)

\bibitem{inductive}
Hamilton, W., Ying, Z., Leskovec, J.: Inductive representation learning on
  large graphs. Advances in neural information processing systems  \textbf{30}
  (2017)

\bibitem{e2vpt}
Han, C., Wang, Q., Cui, Y., Cao, Z., Wang, W., Qi, S., Liu, D.: E\^{} 2vpt: An
  effective and efficient approach for visual prompt tuning. arXiv preprint
  arXiv:2307.13770  (2023)

\bibitem{vig}
Han, K., Wang, Y., Guo, J., Tang, Y., Wu, E.: Vision gnn: An image is worth
  graph of nodes. Advances in neural information processing systems
  \textbf{35},  8291--8303 (2022)

\bibitem{vignn}
Han, Y., Wang, P., Kundu, S., Ding, Y., Wang, Z.: Vision hgnn: An image is more
  than a graph of nodes. In: Proceedings of the IEEE/CVF International
  Conference on Computer Vision. pp. 19878--19888 (2023)

\bibitem{mae}
He, K., Chen, X., Xie, S., Li, Y., Doll{\'a}r, P., Girshick, R.: Masked
  autoencoders are scalable vision learners. In: Proceedings of the IEEE/CVF
  conference on computer vision and pattern recognition. pp. 16000--16009
  (2022)

\bibitem{parameter}
He, X., Li, C., Zhang, P., Yang, J., Wang, X.E.: Parameter-efficient model
  adaptation for vision transformers. In: Proceedings of the AAAI Conference on
  Artificial Intelligence. vol.~37, pp. 817--825 (2023)

\bibitem{adapter}
Houlsby, N., Giurgiu, A., Jastrzebski, S., Morrone, B., De~Laroussilhe, Q.,
  Gesmundo, A., Attariyan, M., Gelly, S.: Parameter-efficient transfer learning
  for nlp. In: International conference on machine learning. pp. 2790--2799.
  PMLR (2019)

\bibitem{lora}
Hu, E.J., Shen, Y., Wallis, P., Allen-Zhu, Z., Li, Y., Wang, S., Wang, L.,
  Chen, W., et~al.: Lora: Low-rank adaptation of large language models. Iclr
  \textbf{1}(2), ~3 (2022)

\bibitem{video-hypergraph}
Huang, Y., Liu, Q., Metaxas, D.: ] video object segmentation by hypergraph cut.
  In: 2009 IEEE conference on computer vision and pattern recognition. pp.
  1738--1745. IEEE (2009)

\bibitem{sine}
Ji, Y., Saratchandran, H., Gordon, C., Zhang, Z., Lucey, S.: Efficient learning
  with sine-activated low-rank matrices. arXiv preprint arXiv:2403.19243
  (2024)

\bibitem{vpt}
Jia, M., Tang, L., Chen, B.C., Cardie, C., Belongie, S., Hariharan, B., Lim,
  S.N.: Visual prompt tuning. In: European conference on computer vision. pp.
  709--727. Springer (2022)

\bibitem{convpass}
Jie, S., Deng, Z.H.: Convolutional bypasses are better vision transformer
  adapters. arXiv preprint arXiv:2207.07039  (2022)

\bibitem{fact}
Jie, S., Deng, Z.H.: Fact: Factor-tuning for lightweight adaptation on vision
  transformer. In: Proceedings of the AAAI conference on artificial
  intelligence. vol.~37, pp. 1060--1068 (2023)

\bibitem{compacter}
Karimi~Mahabadi, R., Henderson, J., Ruder, S.: Compacter: Efficient low-rank
  hypercomplex adapter layers. Advances in neural information processing
  systems  \textbf{34},  1022--1035 (2021)

\bibitem{cars}
Krause, J., Stark, M., Deng, J., Fei-Fei, L.: 3d object representations for
  fine-grained categorization. 2013 IEEE International Conference on Computer
  Vision Workshops pp. 554--561 (2013)

\bibitem{ssf}
Lian, D., Zhou, D., Feng, J., Wang, X.: Scaling \& shifting your features: A
  new baseline for efficient model tuning. Advances in Neural Information
  Processing Systems  \textbf{35},  109--123 (2022)

\bibitem{daam}
Liao, Y., Gao, Y., Zhang, W.: Dynamic accumulated attention map for
  interpreting evolution of decision-making in vision transformer. Pattern
  Recognit.  \textbf{165},  111607 (2025)

\bibitem{BOFT}
Liu, W., Qiu, Z., Feng, Y., Xiu, Y., Xue, Y., Yu, L., Feng, H., Liu, Z., Heo,
  J., Peng, S., et~al.: Parameter-efficient orthogonal finetuning via butterfly
  factorization. arXiv preprint arXiv:2311.06243  (2023)

\bibitem{swinv2}
Liu, Z., Hu, H., Lin, Y., Yao, Z., Xie, Z., Wei, Y., Ning, J., Cao, Y., Zhang,
  Z., Dong, L., et~al.: Swin transformer v2: Scaling up capacity and
  resolution. In: Proceedings of the IEEE/CVF conference on computer vision and
  pattern recognition. pp. 12009--12019 (2022)

\bibitem{swin}
Liu, Z., Lin, Y., Cao, Y., Hu, H., Wei, Y., Zhang, Z., Lin, S., Guo, B.: Swin
  transformer: Hierarchical vision transformer using shifted windows. 2021
  IEEE/CVF International Conference on Computer Vision (ICCV) pp. 9992--10002
  (2021)

\bibitem{repadapter}
Luo, G., Huang, M., Zhou, Y., Sun, X., Jiang, G., Wang, Z., Ji, R.: Towards
  efficient visual adaption via structural re-parameterization. arXiv preprint
  arXiv:2302.08106  (2023)

\bibitem{video-multilevel}
Lv, X., Wang, L., Zhang, Q., Zheng, N., Hua, G.: Video object co-segmentation
  from noisy videos by a multi-level hypergraph model. In: 2018 25th IEEE
  International Conference on Image Processing (ICIP). pp. 2207--2211. IEEE
  (2018)

\bibitem{goft}
Ma, X., Chu, X., Yang, Z., Lin, Y., Gao, X., Zhao, J.: Parameter efficient
  quasi-orthogonal fine-tuning via givens rotation. arXiv preprint
  arXiv:2404.04316  (2024)

\bibitem{aircraft}
Maji, S., Rahtu, E., Kannala, J., Blaschko, M.B., Vedaldi, A.: Fine-grained
  visual classification of aircraft. ArXiv  \textbf{abs/1306.5151} (2013)

\bibitem{mobilevig}
Munir, M., Avery, W., Marculescu, R.: Mobilevig: Graph-based sparse attention
  for mobile vision applications. In: Proceedings of the IEEE/CVF Conference on
  Computer Vision and Pattern Recognition. pp. 2211--2219 (2023)

\bibitem{greedyvig}
Munir, M., Avery, W., Rahman, M.M., Marculescu, R.: Greedyvig: Dynamic axial
  graph construction for efficient vision gnns. In: Proceedings of the IEEE/CVF
  Conference on Computer Vision and Pattern Recognition. pp. 6118--6127 (2024)

\bibitem{flower}
Nilsback, M.E., Zisserman, A.: A visual vocabulary for flower classification.
  2006 IEEE Computer Society Conference on Computer Vision and Pattern
  Recognition (CVPR'06)  \textbf{2},  1447--1454 (2006)

\bibitem{cats}
Parkhi, O.M., Vedaldi, A., Zisserman, A., Jawahar, C.: Cats and dogs. In: 2012
  IEEE conference on computer vision and pattern recognition. pp. 3498--3505.
  IEEE (2012)

\bibitem{sa2vp}
Pei, W., Xia, T., Chen, F., Li, J., Tian, J., Lu, G.: Sa$^2$vp: Spatially
  aligned-and-adapted visual prompt. In: Proceedings of the AAAI conference on
  artificial intelligence. vol.~38, pp. 4450--4458 (2024)

\bibitem{da-vpt}
Ren, L., Chen, C., Wang, L., Hua, K.: Da-vpt: Semantic-guided visual prompt
  tuning for vision transformers. In: Proceedings of the Computer Vision and
  Pattern Recognition Conference. pp. 4353--4363 (2025)

\bibitem{imagenet}
Russakovsky, O., Deng, J., Su, H., Krause, J., Satheesh, S., Ma, S., Huang, Z.,
  Karpathy, A., Khosla, A., Bernstein, M.S., Berg, A.C., Fei-Fei, L.: Imagenet
  large scale visual recognition challenge. International Journal of Computer
  Vision  \textbf{115},  211 -- 252 (2014)

\bibitem{collective}
Sen, P., Namata, G., Bilgic, M., Getoor, L., Galligher, B., Eliassi-Rad, T.:
  Collective classification in network data. AI magazine  \textbf{29}(3),
  93--93 (2008)

\bibitem{visionhgnn}
Srinivas, S.S., Sarkar, R.K., Gangasani, S., Runkana, V.: Vision hgnn: An
  electron-micrograph is worth hypergraph of hypernodes. arXiv preprint
  arXiv:2408.11351  (2024)

\bibitem{deit}
Touvron, H., Cord, M., Douze, M., Massa, F., Sablayrolles, A., J{\'e}gou, H.:
  Training data-efficient image transformers \& distillation through attention.
  In: International conference on machine learning. pp. 10347--10357. PMLR
  (2021)

\bibitem{chemical}
Wale, N., Watson, I.A., Karypis, G.: Comparison of descriptor spaces for
  chemical compound retrieval and classification. Knowledge and Information
  Systems  \textbf{14}(3),  347--375 (2008)

\bibitem{lion}
Wang, H., Chang, J., Zhai, Y., Luo, X., Sun, J., Lin, Z., Tian, Q.: Lion:
  Implicit vision prompt tuning. In: Proceedings of the AAAI conference on
  artificial intelligence. vol.~38, pp. 5372--5380 (2024)

\bibitem{orthogonal_subspace}
Wu, F., Hu, J., Min, G., Wang, S.: Efficient orthogonal fine-tuning with
  principal subspace adaptation. arXiv preprint arXiv:2505.11235  (2025)

\bibitem{self-supervised-vpt}
Yoo, S., Kim, E., Jung, D., Lee, J., Yoon, S.: Improving visual prompt tuning
  for self-supervised vision transformers. In: International Conference on
  Machine Learning. pp. 40075--40092. PMLR (2023)

\bibitem{vpeft}
Yu, B.X., Chang, J., Liu, L., Tian, Q., Chen, C.W.: Towards a unified view on
  visual parameter-efficient transfer learning. arXiv preprint arXiv:2210.00788
   (2022)

\bibitem{bitfit}
Zaken, E.B., Goldberg, Y., Ravfogel, S.: Bitfit: Simple parameter-efficient
  fine-tuning for transformer-based masked language-models. In: Proceedings of
  the 60th Annual Meeting of the Association for Computational Linguistics
  (Volume 2: Short Papers). pp.~1--9 (2022)

\bibitem{vtab}
Zhai, X., Puigcerver, J., Kolesnikov, A., Ruyssen, P., Riquelme, C., Lucic, M.,
  Djolonga, J., Pinto, A.S., Neumann, M., Dosovitskiy, A., et~al.: A
  large-scale study of representation learning with the visual task adaptation
  benchmark. arXiv preprint arXiv:1910.04867  (2019)

\bibitem{noah}
Zhang, Y., Zhou, K., Liu, Z.: Neural prompt search. IEEE Transactions on
  Pattern Analysis and Machine Intelligence  \textbf{47}(7),  5268--5280 (2024)

\bibitem{gps}
Zhang, Z., Zhang, Q., Gao, Z., Zhang, R., Shutova, E., Zhou, S., Zhang, S.:
  Gradient-based parameter selection for efficient fine-tuning. In: Proceedings
  of the IEEE/CVF Conference on Computer Vision and Pattern Recognition. pp.
  28566--28577 (2024)

\end{thebibliography}

\newpage
\appendix

\section{Dataset Statistics}
\begin{table}[tbp]
\centering
\caption{VTAB-1k datasets \cite{vtab} categorized into Natural, Specialized, and Structured groups. Training set sizes are 800 or 1,000 depending on availability.}
\small
\begin{tabular}{llcccc}
\toprule
\textbf{Category} & \textbf{Dataset} & \textbf{\# Classes} & \textbf{Train} & \textbf{Val} & \textbf{Test} \\
\midrule
\multirow{7}{*}{Natural} 
 & CIFAR100  & 100 & & & 10,000 \\
 & Caltech101  & 102 & & & 6,084 \\
 & DTD  & 47 & & & 1,880 \\
 & Oxford-Flowers102  & 102 & 800/1,000 & 200 & 6,149 \\
 & Oxford-Pets  & 37 & & & 3,669 \\
 & SVHN  & 10 & & & 26,032 \\
 & Sun397  & 397 & & & 21,750 \\
\midrule
\multirow{4}{*}{Specialized}
 & Patch Camelyon  & 2 & & & 32,768 \\
 & EuroSAT  & 10 & & & 5,400 \\
 & Resisc45  & 45 & 800/1,000 & 200 & 6,300 \\
 & Retinopathy  & 5 & & & 42,670 \\
\midrule
\multirow{8}{*}{Structured}
 & Clevr/count  & 8 & & & 15,000 \\
 & Clevr/distance  & 6 & & & 15,000 \\
 & DMLab  & 6 & & & 22,735 \\
 & KITTI-Dist  & 4 & 800/1,000 & 200 & 711 \\
 & dSprites/location  & 16 & & & 73,728 \\
 & dSprites/orientation  & 16 & & & 73,728 \\
 & SmallNORB/azimuth  & 18 & & & 12,150 \\
 & SmallNORB/elevation  & 18 & & & 12,150 \\
\bottomrule
\end{tabular}
\label{tab:vtab1k}
\end{table}
We provide detailed information about the datasets used in this paper, including the number of classes and the sizes of the training, validation and test sets in Table \ref{tab:vtab1k} and Table \ref{tab:fewshot}. 

The VTAB-1K datasets consists of three categories: Natural, Specialized and Structured tasks. The Natural category includes datasets such as CIFAR-100, Caltech101, DTD, Flowers102, Pets, SVHN, and Sun397. The Specialized category includes datasets such as Patch Camelyon, EuroSAT, Resisc45, and Diabetic-Retinopathy, and the Structured category includes Clevr/count, Clevr/distance, DMLab, KITTI/distance, dSprites/location, dSprites/orientation, SmallNORB/azimuth, and SmallNORB/elevation. 

\begin{table}[tbp]
\centering
\caption{Few-shot datasets used for evaluation. Training size varies (e.g., 1/2/4/8/16 per class), with fixed validation and test sets.}
\small
\begin{tabular}{lcccc}
\toprule
\textbf{Dataset} & \textbf{\# Classes} & \textbf{Train} & \textbf{Val} & \textbf{Test} \\
\midrule
Food-101 & 101 &  &20,200 & 30,300 \\
Stanford Cars & 196 &  &1,635 & 8,041 \\
Oxford-Flowers102 & 102 & 1/2/4/8/16 per class & 1,633 & 2,463 \\
FGVC-Aircraft & 100 &  &3,333 & 3,333 \\
Oxford-Pets & 37 &  &736 & 3,669 \\
\bottomrule
\end{tabular}
\label{tab:fewshot}
\end{table}

For the the few-shot fine-grained visual recognition, the datasets consists of FGVC-Aircraft\cite{aircraft}, Food-101\cite{food101}, Oxford-Flowers102 \cite{flower}, Oxford-Pets\cite{cats} and Stanford Cars\cite{cars}.

\begin{table}[tbp]
\centering
\caption{Experiment configurations for VTAB-1K and few-shot fine-grained visual recognition experiments.}
\small
\resizebox{\textwidth}{!}{%
\begin{tabular}{lccccccc}
\toprule
 \textbf{Dataset} & \textbf{optimizer} & \textbf{batch-size} & \textbf{learning-rate} & \textbf{weight-decay} & \textbf{epochs} & \textbf{lr-decay} & \textbf{ warm-up} \\
\midrule
VTAB-1K & AdamW & 64 & 1e-3 & 1e-4 & 100 & cosine & 10 \\
Few-shot & AdamW & 64 & 5e-3 & 1e-4 & 100 & cosine & 10 \\
\bottomrule
\end{tabular}
}
\label{tab:experiment-configuration}
\end{table}

\section{Experimental Setup}
In our experiments, we choose ViT-B/16\cite{vit} trained on ImageNet-21K as our backbone. For VTAB-1K, we resize the images to $224 \times 224$. Different from VTAB-1K, we use random augmentations during training, for validation/test samples, we resize them to 256$\times$256, crop them to 224$\times$224 at the center, and then normalize them with ImageNet’s mean and standard deviation.

For Swin Transformer, HyperAdapter is inserted in parallel to the attention and MLP modules in each transformer block across all stages, analogous to the ViT setting. We use a fixed number of hyperedges $K$ across stages, since hyperedges represent latent semantic groups rather than spatial partitions. The routing operates directly on token features, enabling adaptive grouping despite varying token resolutions.

Table \ref{tab:experiment-configuration} shows our experiment configurations.

\section{Complexity Analysis}

Let $N$ denote the number of patch tokens, $D$ the hidden dimension, $K$ the number of hyperedges, and $r$ the adapter bottleneck dimension.

\textbf{Parameter complexity.} HyperAdapter introduces trainable hyperedge prototypes and lightweight bottleneck weights, resulting in $O(KD + Dr)$, where $K \ll N$ and $r \ll D$. This ensures that the additional parameter overhead is minimal compared to the frozen ViT backbone.

\textbf{Computational complexity.} The main operations include token-to-hyperedge routing, hyperedge aggregation, and diffusion back to token space, each costing $O(NKD)$, while the hyperedge-level bottleneck adaptation requires $O(KDr)$.
Since self-attention in a standard ViT has complexity $O(N^2D)$, the extra computations introduced by HyperAdapter are negligible in practice, preserving the efficiency of parameter-efficient fine-tuning.

\begin{table}[tbp]
\centering
\caption{Efficiency comparison of different PEFT methods. HyperAdapter introduces a small computational overhead compared to baseline, AdaptFormer, and LoRA while remaining efficient relative to the frozen backbone.}
\label{tab:efficiency}
\begin{tabular}{lcccc}
\toprule
{Method} & {Train time} & {Inference time} & {Memory} & {FLOPs} \\
& (ms/batch) & (ms/batch) & (GB) & \\
\midrule
Baseline      & 224 & 121 & 3.0 & 17.6 \\
AdaptFormer   & 212 & 117 & 2.8 & 17.6 \\
LoRA          & 218 & 118 & 2.9 & 17.6 \\
HyperAdapter  & 239 & 129 & 3.2 & 17.8 \\
\bottomrule
\end{tabular}
\end{table}

\subsection{Training time and Inference Time}


We provide quantitative comparisons of training/inference latency, peak memory, and FLOPs measured on a single NVIDIA A4000 GPU (batch size 64).

HyperAdapter introduces only modest overhead compared with token-wise PEFT methods (Table \ref{tab:efficiency}): training latency increases from 212-218 to 239 ms/batch, inference latency from 117-121 to 129 ms/batch, and memory from 2.8-3.0 to 3.2 GB, while FLOPs remain nearly unchanged (17.8G vs. 17.6G). 
The overhead mainly comes from token-hyperedge routing and aggregation. Since these operations are performed over a small latent hyperedge space ($K=8$) within a low-rank bottleneck adapter, the additional complexity is ${O}(NKD + KDr)$, which is negligible compared to the backbone self-attention cost ${O}(N^2D)$. 

Thus, HyperAdapter preserves a favorable efficiency-performance trade-off while enabling structured group-level adaptation.



\textbf{Hyperedge number $K$.} Hyperedges in HyperAdapter represent \emph{latent semantic groups} rather than fixed spatial regions, so $K$ does not scale directly with image resolution or token count. Tokens are dynamically grouped through soft similarity-based routing, and moderate values (e.g., $K=8$) generalize well across datasets and architectures. Even for higher-resolution inputs, HyperAdapter compresses token features into a fixed number of hyperedges, keeping the interaction space compact and computationally efficient.

\section{More Ablations}
\subsection{Effect of Bottleneck Dimension}

\begin{table}[tbp]
\centering
\caption{Effect of bottleneck dimension $r$. Increasing $r$ increases adapter capacity but also the number of trainable parameters.}
\label{tab:bottleneck_ablation}
\begin{tabular}{lcc}
\toprule
$r$ & Average (\%) & Params (M) \\
\midrule
4  & 77.2 & 0.29 \\
\textbf{8}  & \textbf{77.6} & 0.44 \\
16 & 77.2 & 0.74 \\
32 & 77.2 & 1.32 \\
64 & 77.0 & 2.51 \\
\bottomrule
\end{tabular}
\end{table}

We analyze the impact of the adapter bottleneck dimension $r$, which controls the capacity of the hyperedge adapter.

As shown in Table~\ref{tab:bottleneck_ablation}, performance improves when increasing $r$ from 4 to 8, but quickly saturates thereafter. In particular, larger bottleneck dimensions ($r \ge 16$) provide no additional benefit despite introducing substantially more parameters. This suggests that the performance gains of HyperAdapter primarily arise from structured hyperedge routing rather than increased adapter capacity. Consequently, we adopt $r=8$ for all experiments as it offers the best trade-off between accuracy and parameter efficiency.

\subsection{Effect of CLS Token Aggregation}
We analyze different strategies for propagating hyperedge updates to the CLS token. 

As shown in Table~\ref{tab:cls_ablation}, aggregating hyperedge representations (MeanH) achieves the best performance.
Removing CLS updates (Zero) slightly degrades performance, while aggregating patch updates (MeanPatch) performs similarly but remains inferior to hyperedge aggregation.
This suggests that global representations formed in hyperedge space provide a more informative signal for the CLS token, highlighting the importance of structured token grouping in HyperAdapter.

\begin{table}[tbp]
\centering
\caption{Effect of CLS token aggregation strategy. We compare different ways of propagating hyperedge updates to the CLS token.}
\label{tab:cls_ablation}
\begin{tabular}{lc}
\toprule
CLS Aggregation & Average (\%) \\
\midrule
Zero (no CLS update) & 77.1 \\
MeanPatch & 77.0 \\
\textbf{MeanH (Ours)} & \textbf{77.6} \\
\bottomrule
\end{tabular}
\end{table}

\begin{figure}[tbp]
\centering
\setlength{\tabcolsep}{1.2pt}
\renewcommand{\arraystretch}{0.92}
\newcommand{\imgH}{35mm}
\newcommand{\rowlabw}{2.4cm}
\newcommand{\imgcellw}{0.28\textwidth}
\newcommand{\img}[1]{\includegraphics[height=\imgH,width=\imgcellw,keepaspectratio]{#1}}
\newcommand{\rowlab}[1]{\centering\scriptsize\bfseries #1}

\begin{adjustbox}{max width=\linewidth}
\begin{tabular}{@{} >{\centering\arraybackslash}m{\rowlabw} *{3}{>{\centering\arraybackslash}m{\imgcellw}} @{}}

& \scriptsize Block~1 & \scriptsize Block~6 & \scriptsize Block~12 \\[2pt]

\rowlab{CIFAR100(Attn)} &
\img{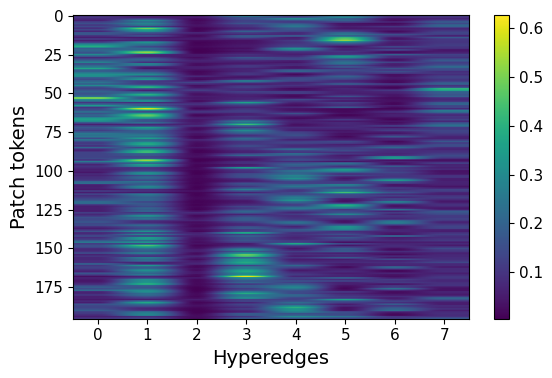} &
\img{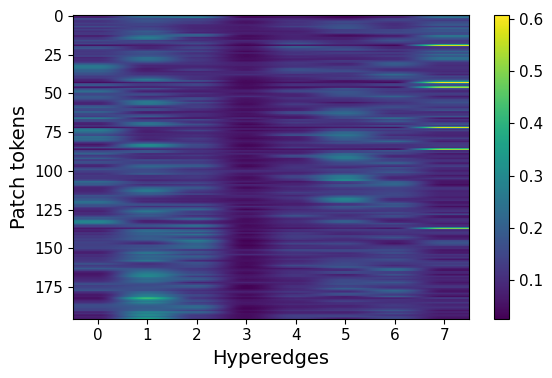} &
\img{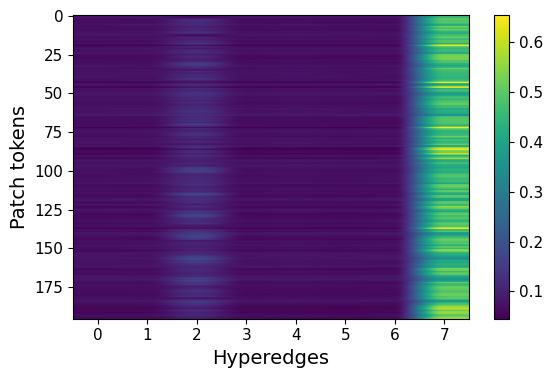} \\[-1pt]

\rowlab{CIFAR100(MLP)} &
\img{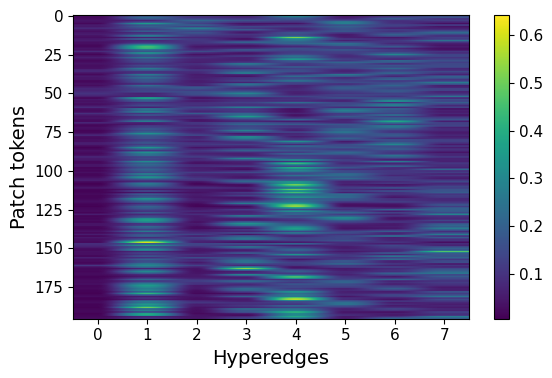} &
\img{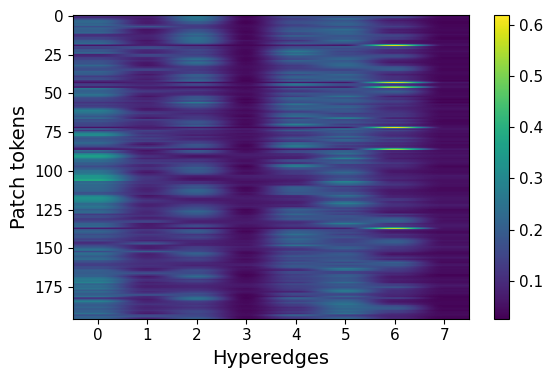} &
\img{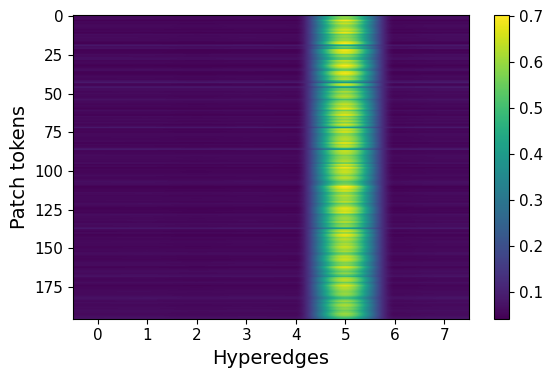} \\[-1pt]

\rowlab{EuroSAT(Attn)} &
\img{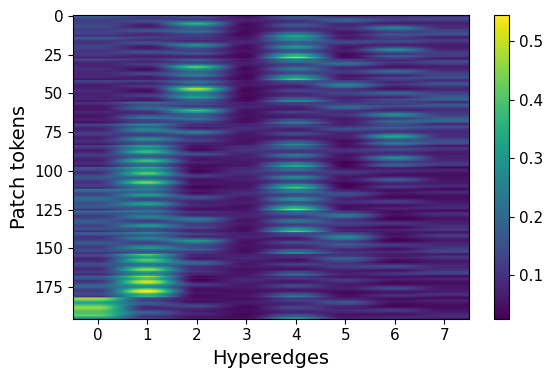} &
\img{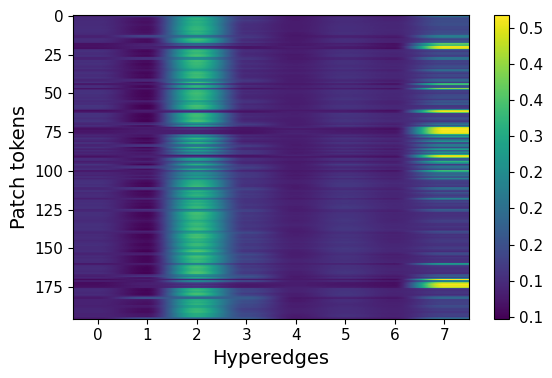} &
\img{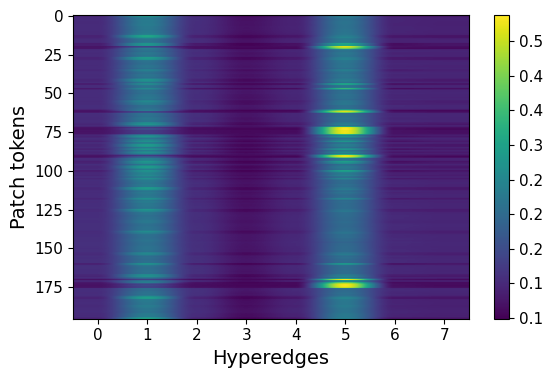} \\[-1pt]

\rowlab{EuroSAT(MLP)} &
\img{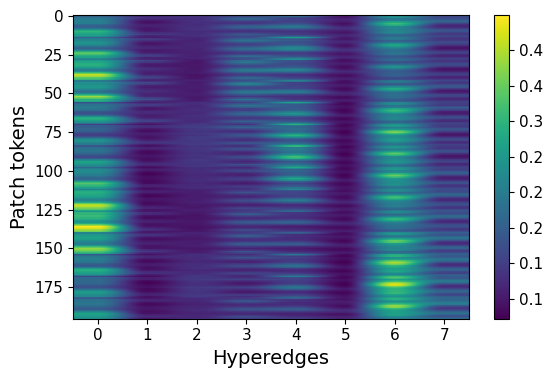} &
\img{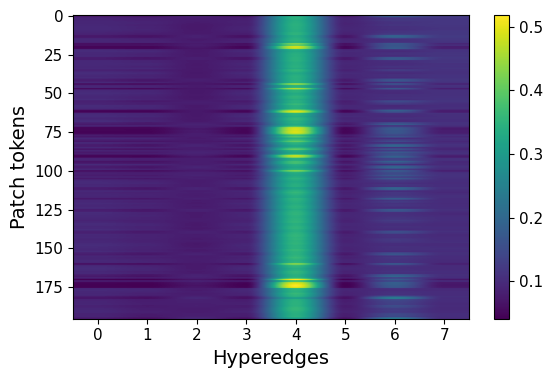} &
\img{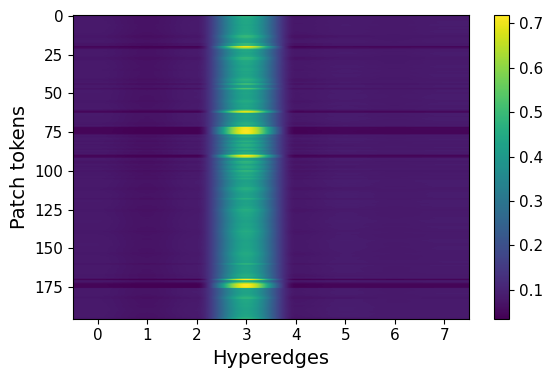} \\[-1pt]

\rowlab{KITTI(Attn)} &
\img{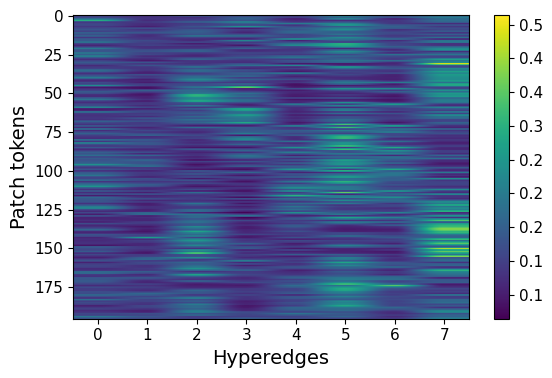} &
\img{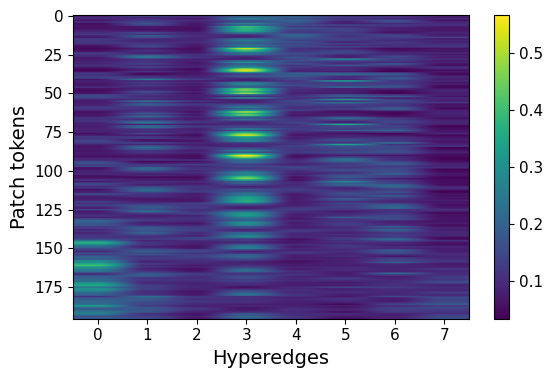} &
\img{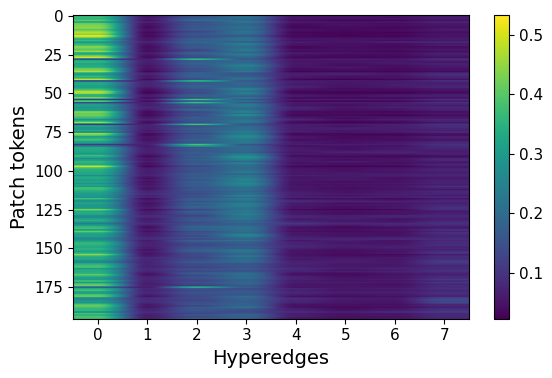} \\[-1pt]

\rowlab{KITTI(MLP)} &
\img{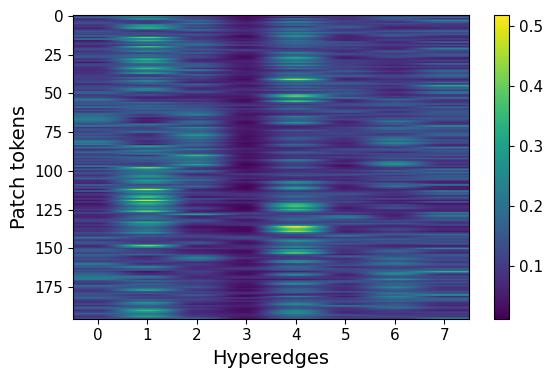} &
\img{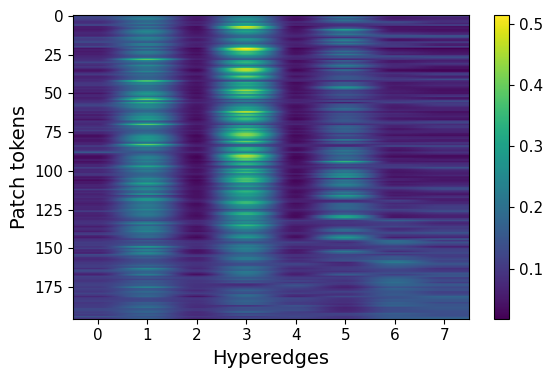} &
\img{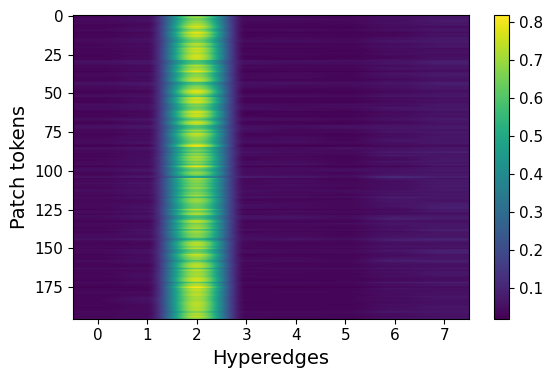} \\

\end{tabular}
\end{adjustbox}
\caption{ Token-to-hyperedge membership heatmaps across representative transformer layers. Each heatmap visualizes the routing matrix $\mM$, where rows correspond to patch tokens and columns correspond to hyperedges. We show routing patterns from Blocks 1, 6, and 12 for both attention and MLP adapters across CIFAR100, EuroSAT, and KITTI. Early layers exhibit diffuse and distributed routing assignments, while deeper layers show increasingly structured and concentrated patterns, indicating progressive specialization of hyperedges. This behavior is consistent across datasets and modules, suggesting that HyperAdapter learns hierarchical token groupings throughout the network. For clarity, we recommend zooming in.}
\label{fig:heatmaps}
\end{figure}

\section{More Analysis}

\subsection{Additional Discussions}

\textbf{Relation to HGNNs.} While HyperAdapter is conceptually related to hypergraph learning, it differs fundamentally from HGNNs \cite{feng2019hypergraph} in both objective and formulation. HGNNs are end-to-end backbone models for relational representation learning, whereas HyperAdapter is a lightweight PEFT module for frozen transformers that introduces structured adaptation without modifying the backbone.
HyperAdapter differs from HGNNs in three aspects: (i) \textbf{objective}: revisiting \emph{where adaptation occurs} in PEFT by shifting from token-wise to group-level adaptation; (ii) \textbf{computation}: replacing iterative graph propagation with lightweight \emph{routing $\rightarrow$ hyperedge adaptation $\rightarrow$ diffusion} inside bottleneck adapters; and (iii) \textbf{design}: preserving the low-rank and parameter-efficient properties of PEFT while remaining permutation equivariant (Prop. 2) and reducing to token-wise adapters as a special case (Prop. 1).
Importantly, hyperedges in HyperAdapter serve only as a transient interaction space for efficient adaptation rather than as the primary feature extraction mechanism. To the best of our knowledge, this is the first work formulating PEFT as \emph{hyperedge-level structured adaptation} for ViTs.

\begin{wraptable}{r}{0.5\linewidth}
\vspace{-1.2cm}
\centering
\caption{ADE20K semantic segmentation results. HyperAdapter achieves the highest mIoU among PEFT methods.}
\label{table:ade20k}
\resizebox{\linewidth}{!}{
\begin{tabular}{l|c|c|c}
\toprule
\textbf{Method} & \textbf{mIoU-SS} & \textbf{mIoU-MS} & \textbf{Params.(M)} \\
\midrule
Full fine-tuning  & 48.31 & 50.07 & 318.31 \\
Linear probing    & 35.12 & 37.46 & 13.18 \\
\midrule
VPT                & 42.11 & 44.06 & 13.43 \\
RepAdapter         & 44.44 & 46.17 & 13.82 \\
HyperAdapter       & \textbf{45.20} & \textbf{46.95} & 14.72 \\
\bottomrule
\end{tabular}}
\vspace{-0.8cm}
\end{wraptable}

\textbf{Generalization.}
While VTAB and few-shot FGVC are standard PEFT benchmarks for controlled comparison, HyperAdapter is designed as a general and architecture-agnostic PEFT module. Our main paper (Tables 1--3) demonstrates consistent improvements across diverse backbones, including ViT-B/16, ViT-L/16, and hierarchical Swin-Base transformers.
To further evaluate generalization beyond image classification, we additionally tested HyperAdapter on ADE20K semantic segmentation using a ViT-L backbone (Table \ref{table:ade20k}).
HyperAdapter outperforms strong PEFT baselines on dense prediction, showing that structured hyperedge adaptation generalizes beyond image classification to spatially dense vision tasks.

\begin{wraptable}{r}{0.45\linewidth}
\vspace{-1.2cm}
\centering
\caption{Performance under matched parameter budgets. HyperAdapter achieves the highest accuracy. 
}
\label{table:eff-acc}
\resizebox{\linewidth}{!}{
\begin{tabular}{l|c|c|c}
\toprule
Method & Bottleneck $r$ & Params (M) & Avg. Acc. \\
\midrule
Baseline      & 12 & 0.44 & 76.7 \\
AdaptFormer   & 24 & 0.44 & 76.6 \\
HyperAdapter     & 8  & 0.44 & 77.6 \\
\bottomrule
\end{tabular}}
\vspace{-0.8cm}
\end{wraptable}

\textbf{Params vs. gains.} We enlarged token-wise adapters to match HyperAdapter’s parameter budget (Table \ref{table:eff-acc}). 
HyperAdapter still achieves higher accuracy, showing that the gains come from structured hyperedge adaptation rather than increased capacity alone.

\textbf{Structural inductive bias.}
HyperAdapter does not explicitly impose spatial priors; instead, the structural bias emerges through representation similarity. In pretrained ViTs, token embeddings encode rich spatial and semantic information, allowing similarity-based routing to group semantically related regions into shared hyperedges.
Importantly, the routing is dynamic and data-dependent rather than manually constrained by spatial neighborhoods. As shown in Fig. 6 of the main paper, hyperedges align with meaningful object regions despite the absence of explicit spatial supervision.

\begin{table}[tbp]
\centering
\caption{Performance comparison with recent PEFT methods on VTAB-1K using a ViT-L backbone. HyperAdapter consistently outperforms prior approaches and achieves the highest average accuracy, particularly excelling on Structured tasks.}
\label{tab:vitl}
\begin{tabular}{l|c|c|c|c}
\toprule
Method & Natural & Specialized & Structured & Average (\%) \\
\midrule
VFPT (NeurIPS 2024)  & 81.4 & 84.9 & 60.2 & 75.5 \\
ViaPT (ACMMM 2024) & 82.6 & 85.2 & 61.3 & 76.4 \\
HyperAdapter & 82.3 & 86.6 & 63.8 & 77.6 \\
\bottomrule
\end{tabular}
\end{table}

\textbf{Comparison with recent baselines.}
We compared HyperAdapter with recent PEFT methods, including VFPT and ViaPT, under the same VTAB-1K protocol and ViT-L backbone setting (Table \ref{tab:vitl}).
HyperAdapter achieves the best overall performance, with strong gains on Structured tasks, further supporting the effectiveness of structured hyperedge adaptation. 

\subsection{Membership Heatmaps Across Layers}
Fig. \ref{fig:heatmaps} presents token-to-hyperedge
membership heatmaps across all transformer layers for CIFAR100, EuroSAT, and KITTI. Each heatmap corresponds to the routing matrix $\mM$, where rows represent patch tokens and columns correspond to hyperedges. 

Several consistent patterns can be observed. In early transformer layers (Block 1), routing assignments are relatively diffuse, with tokens distributed across multiple hyperedges. This suggests that low-level representations remain shared and broadly distributed. As the network depth increases, routing patterns become progressively more structured and concentrated, with tokens exhibiting stronger preferences for specific hyperedges. This indicates that hyperedges become increasingly specialized in capturing semantic token groups. This trend is visible across both attention and MLP adapters,
demonstrating that hyperedge-level adaptation operates consistently across transformer submodules. 
Moreover, the same hierarchical behavior appears across different datasets, suggesting that the routing mechanism learns dataset-agnostic grouping structures. 

These visualizations provide further evidence that HyperAdapter progressively organizes tokens into semantically coherent groups as features propagate through the network.

\begin{figure}[tbp]
\centering
\setlength{\tabcolsep}{1.2pt}
\renewcommand{\arraystretch}{0.9}
\newcommand{\imgH}{35mm}
\newcommand{\rowlabw}{2.4cm}
\newcommand{\imgcellw}{0.25\textwidth}
\newcommand{\img}[1]{\includegraphics[height=\imgH,width=\imgcellw,keepaspectratio]{#1}}
\newcommand{\rowlab}[1]{\centering\scriptsize\bfseries #1}
\begin{adjustbox}{max width=\linewidth}
\begin{tabular}{@{} >{\centering\arraybackslash}m{\rowlabw} *{3}{>{\centering\arraybackslash}m{\imgcellw}} @{}}
& \scriptsize Block~1 & \scriptsize Block~6 & \scriptsize Block~12 \\[2pt]
\rowlab{Flowers102(Attn)} &
\img{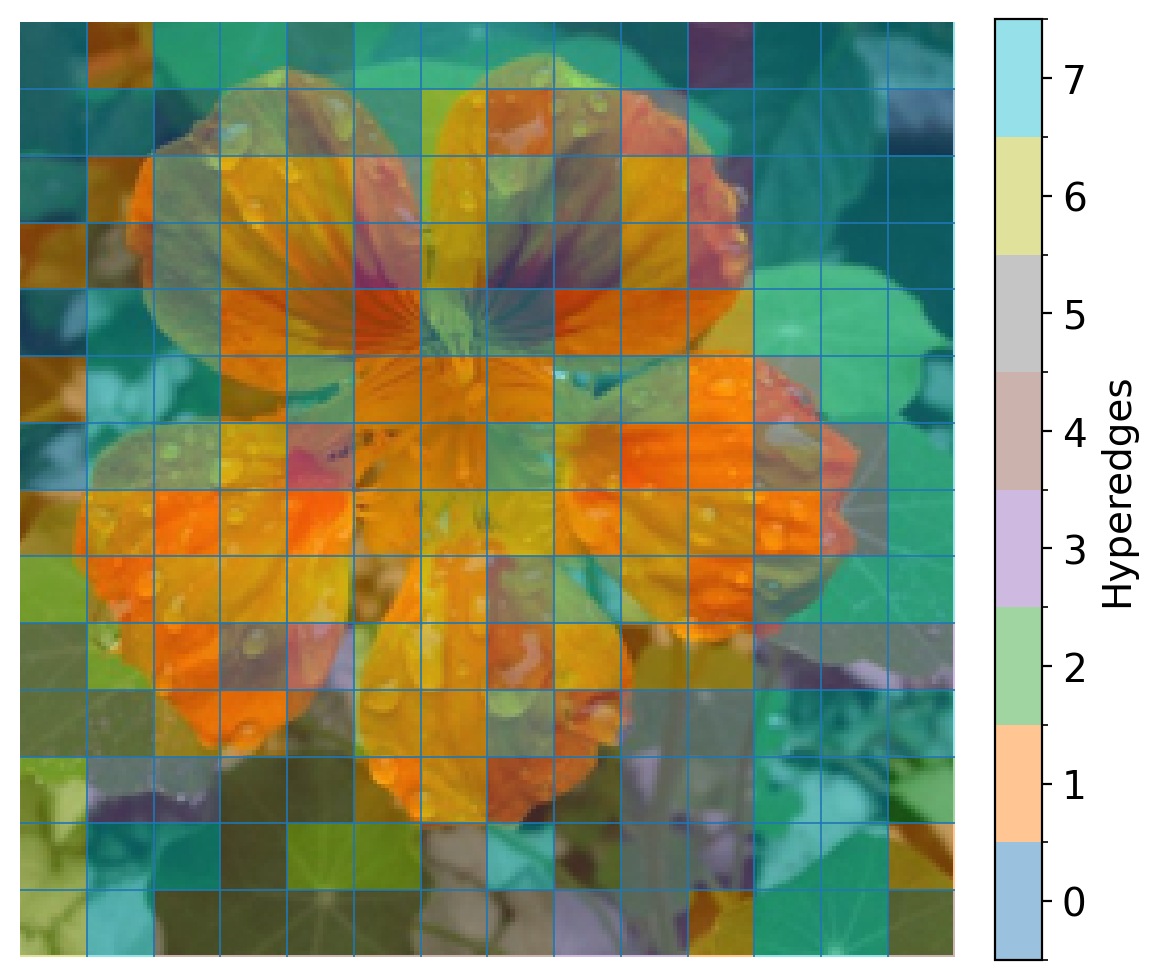} &
\img{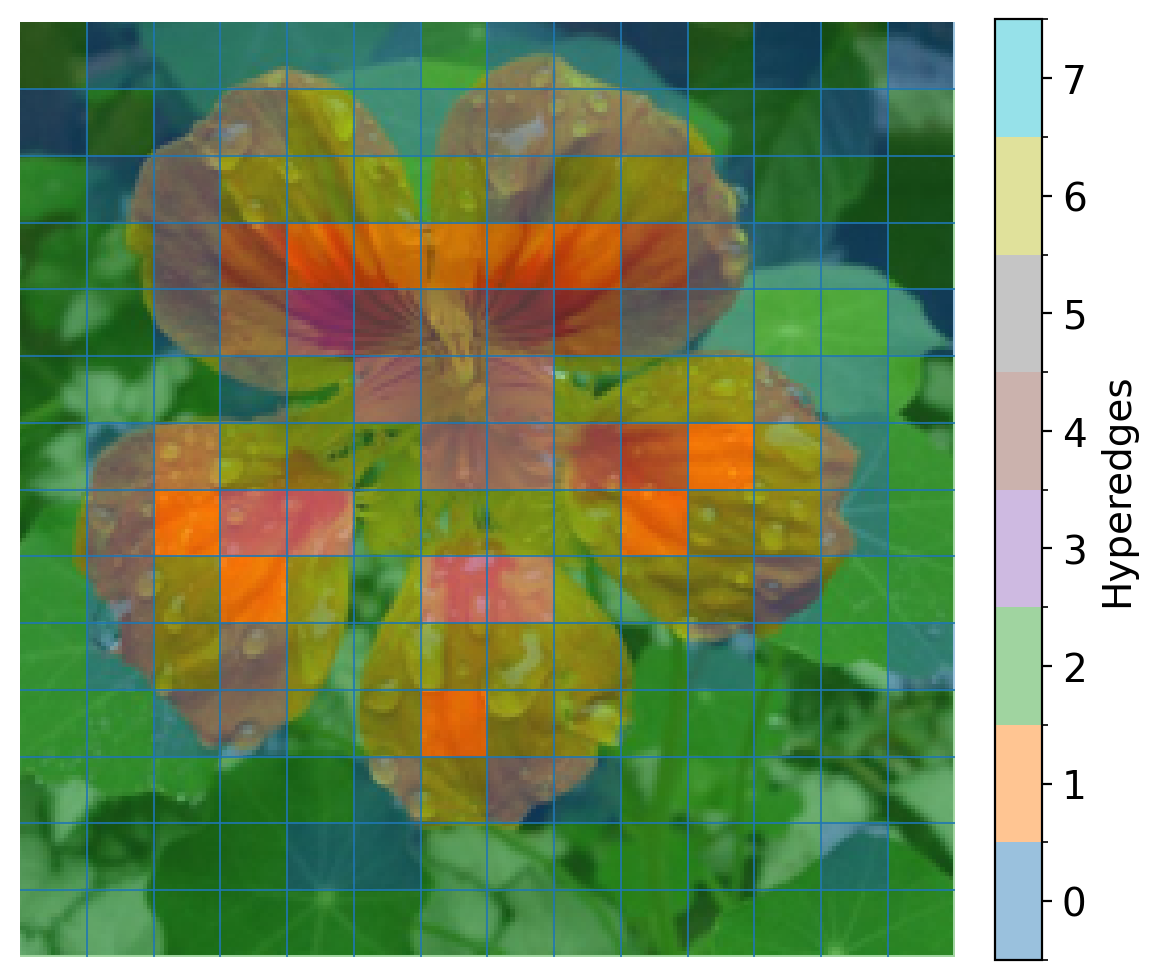} &
\img{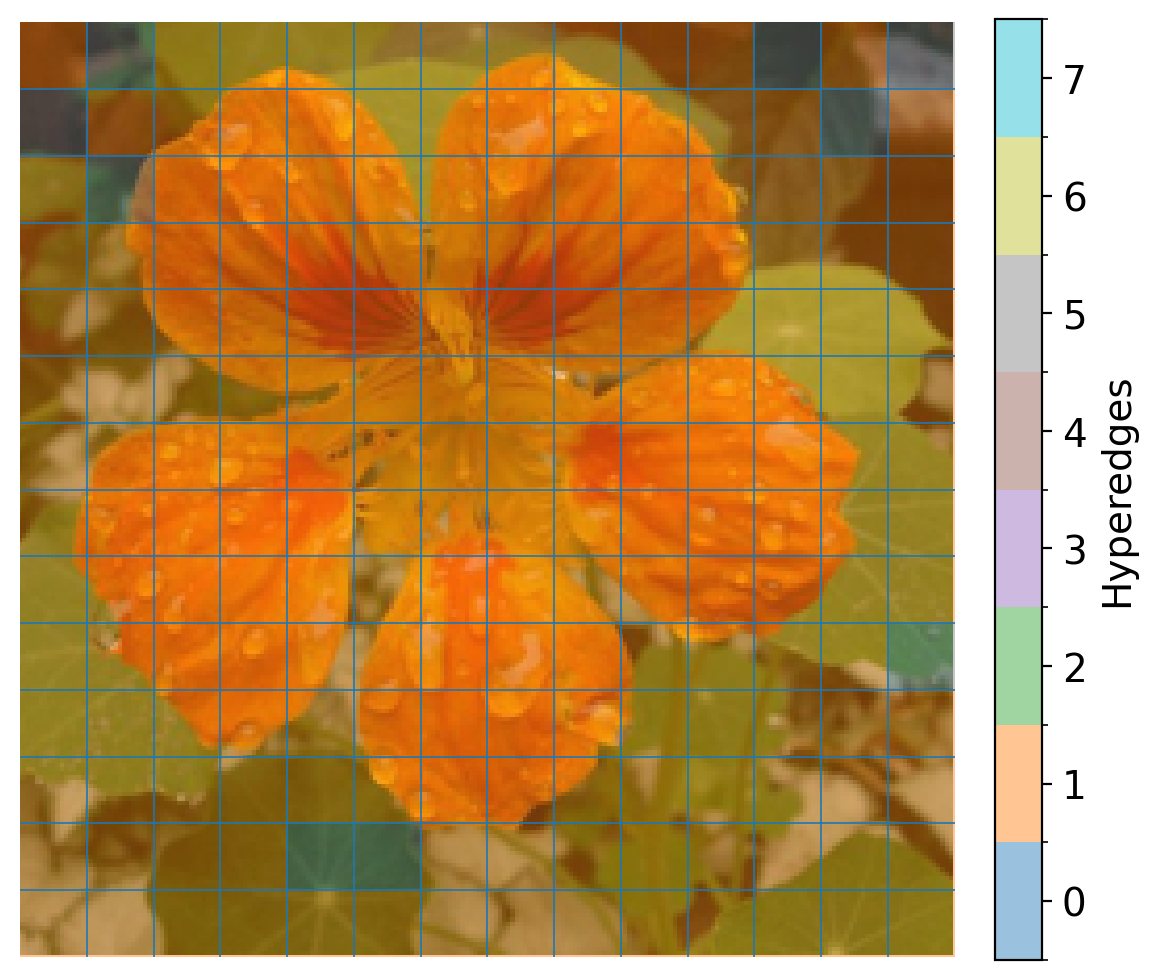} \\[-1pt]

\rowlab{Flowers102(MLP)} &
\img{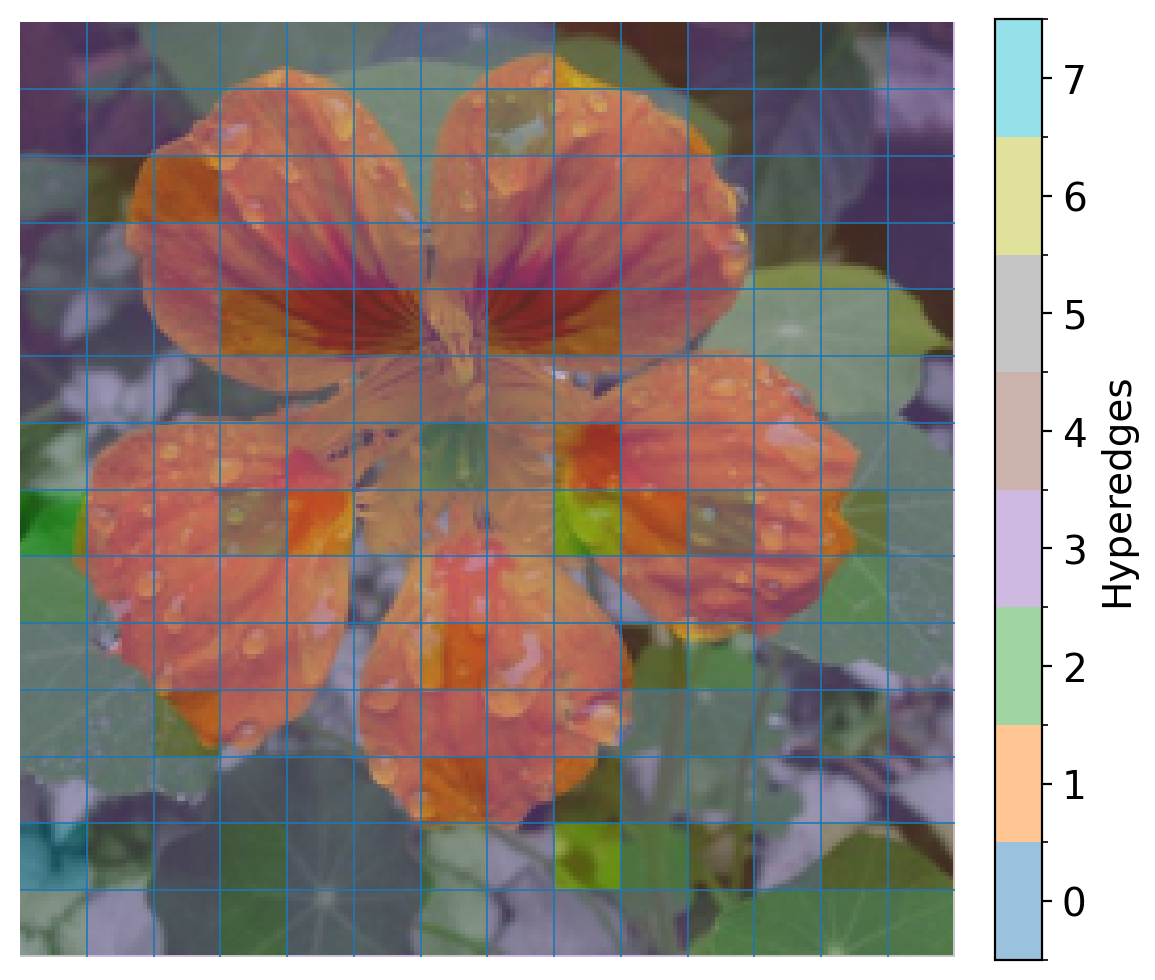} &
\img{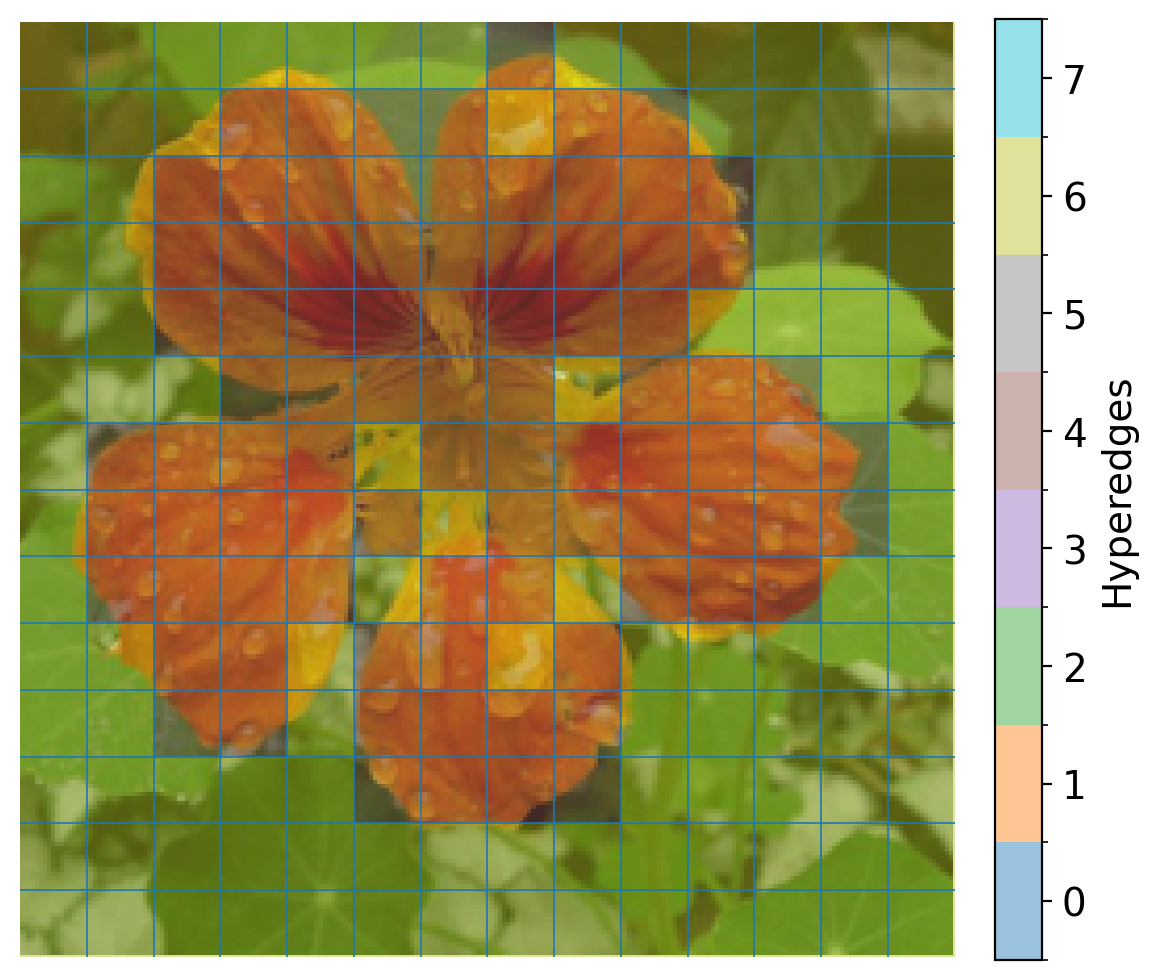} &
\img{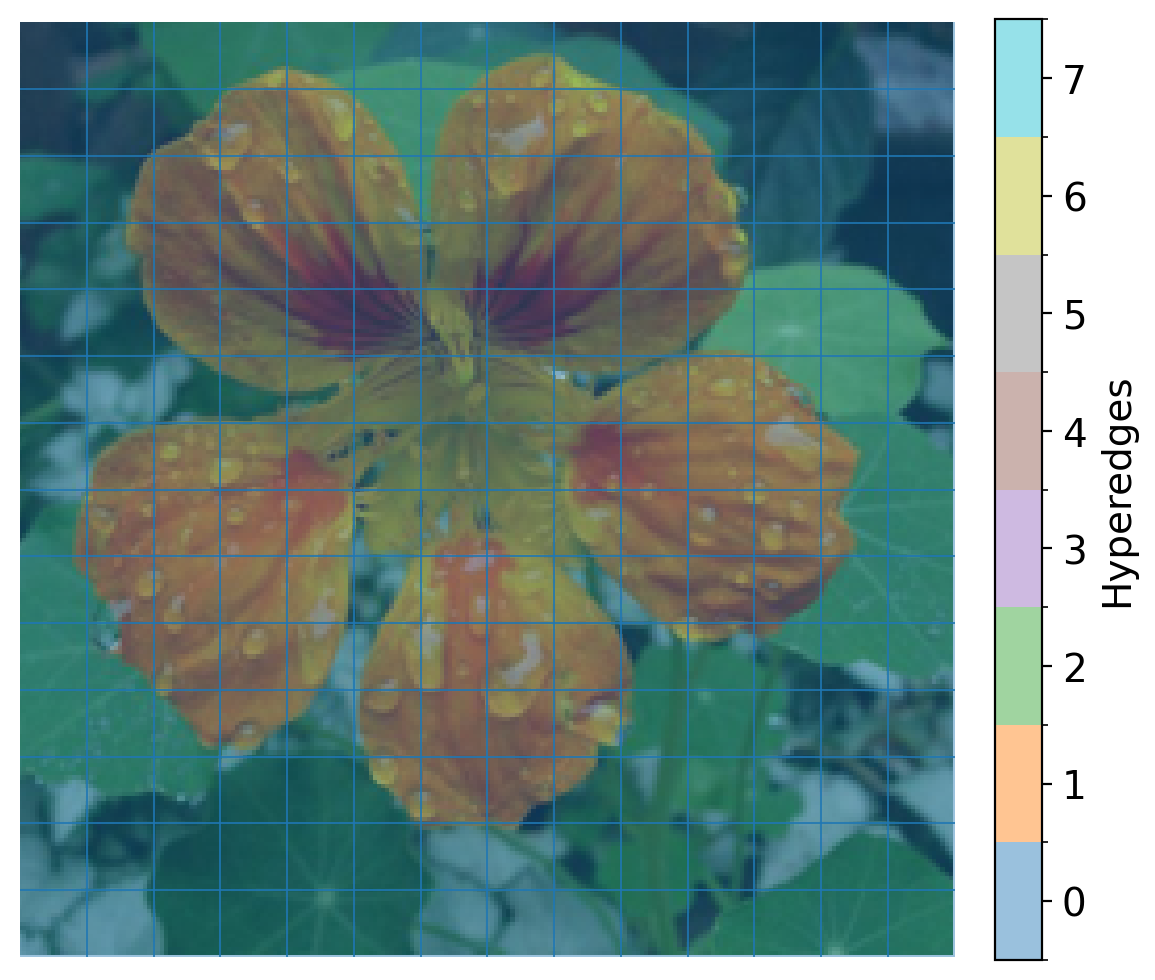} \\[-1pt]

\rowlab{Pets(Attn)} &
\img{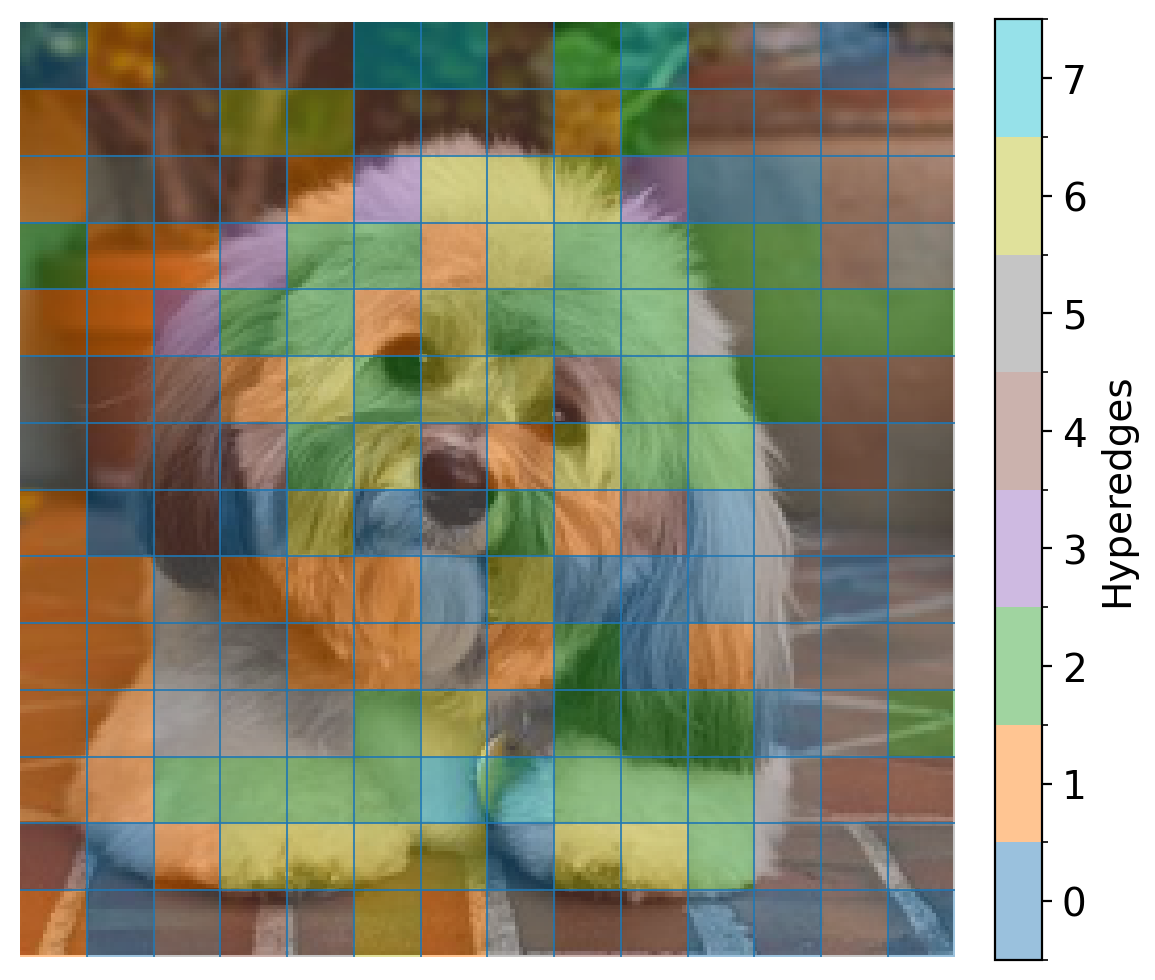} &
\img{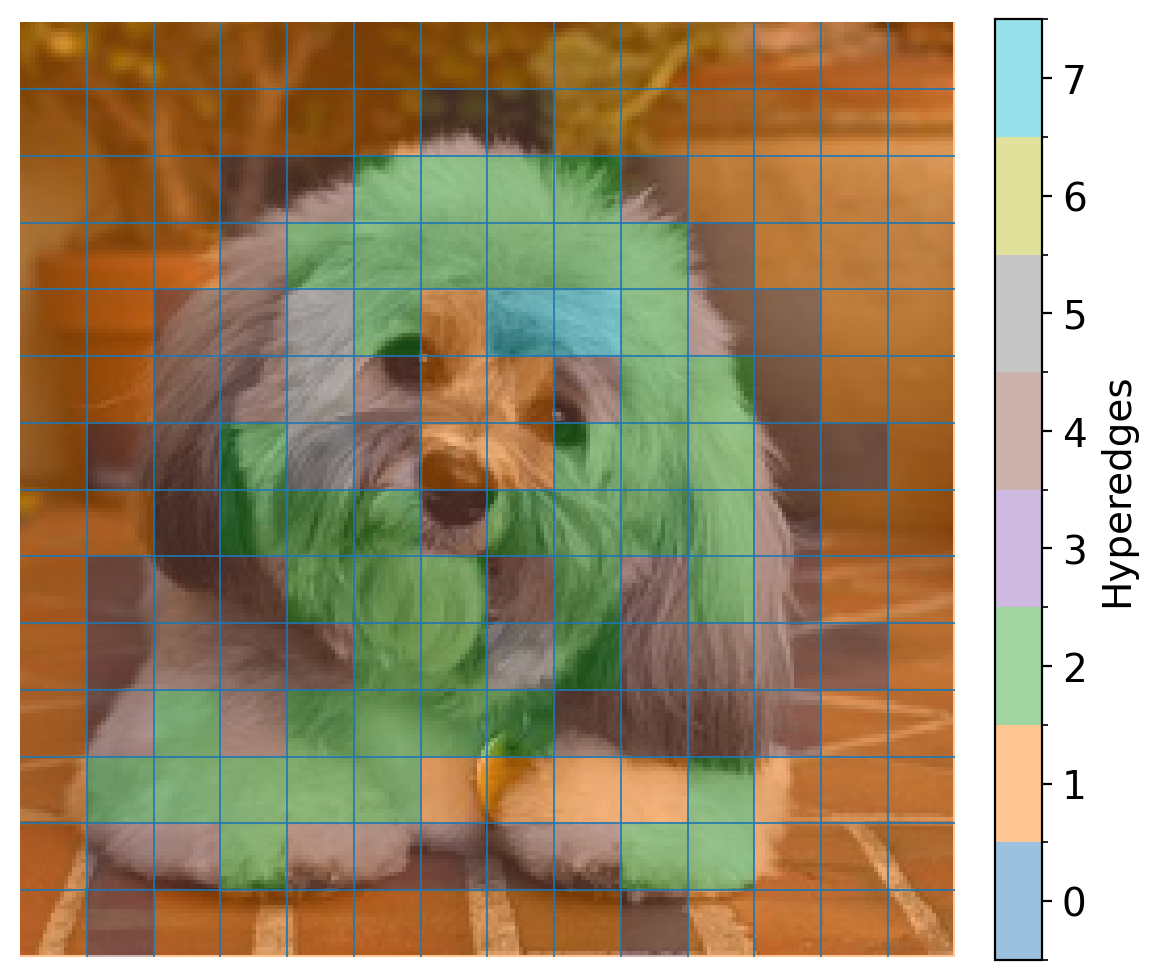} &
\img{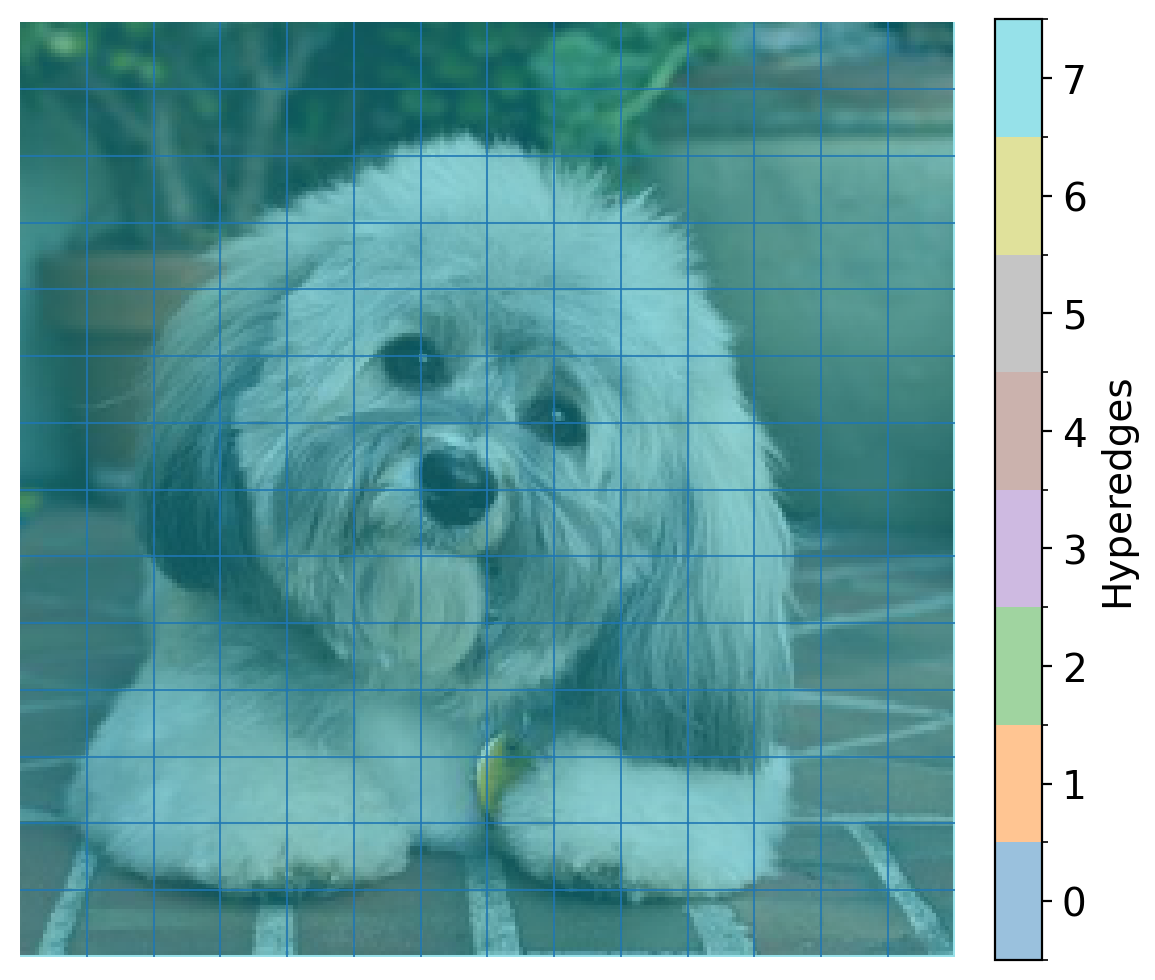} \\[-1pt]

\rowlab{Pets(MLP)} &
\img{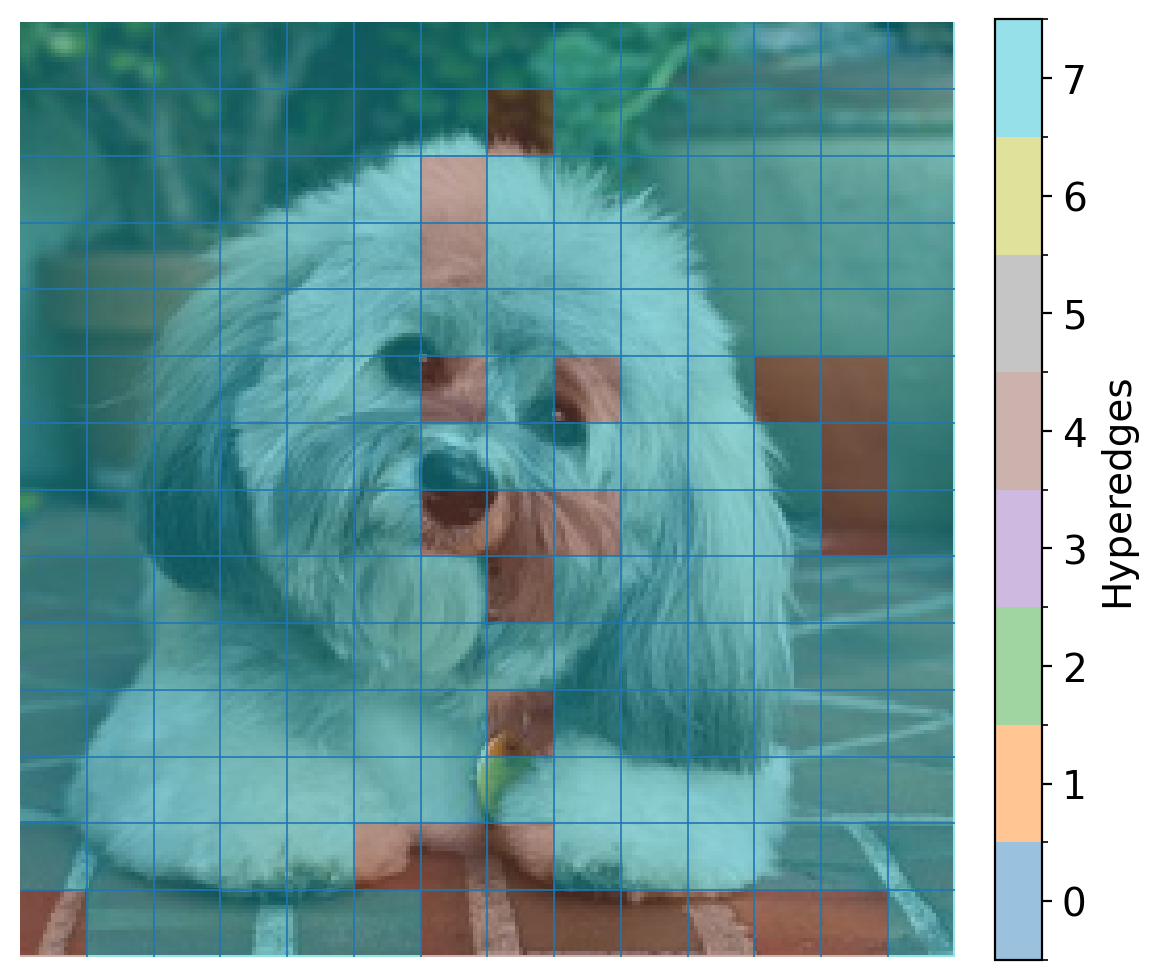} &
\img{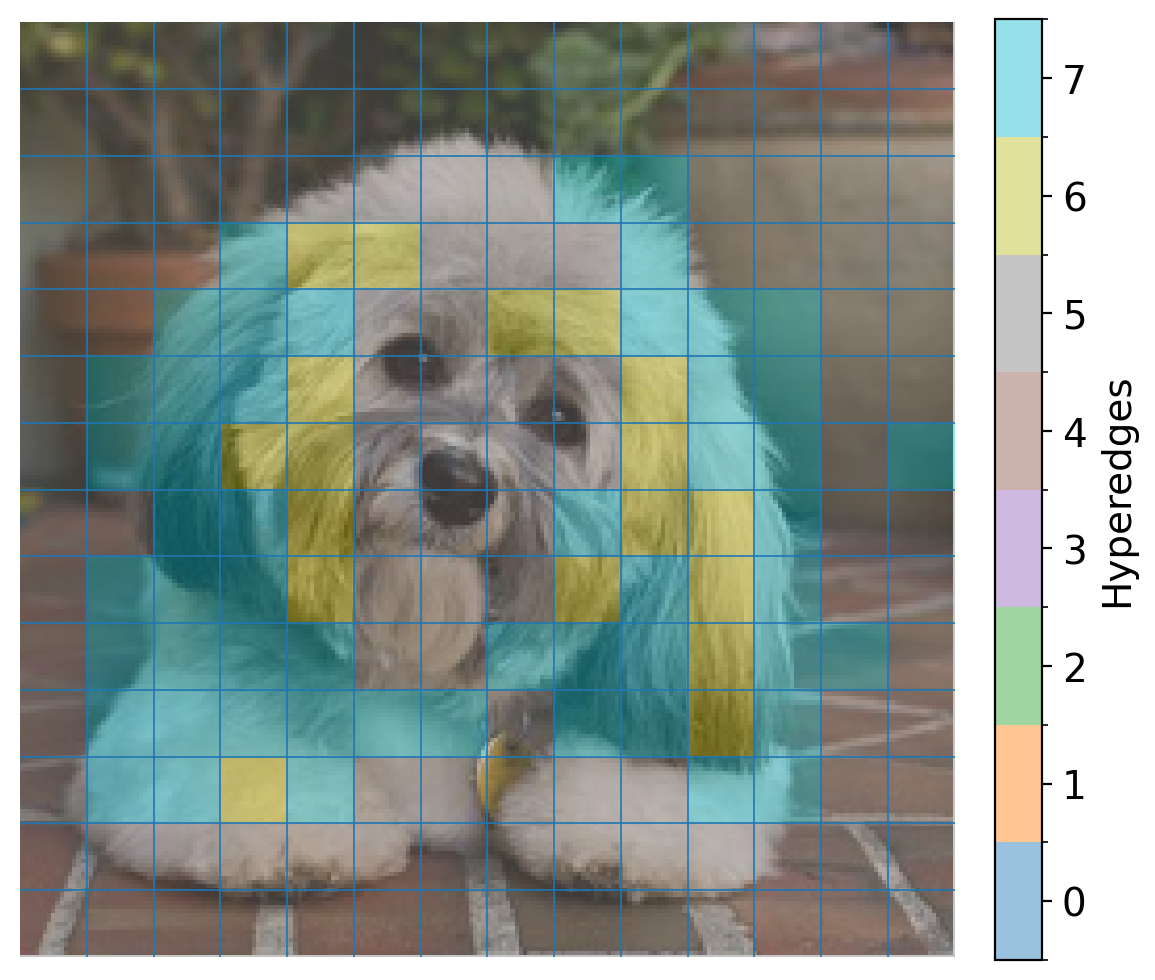} &
\img{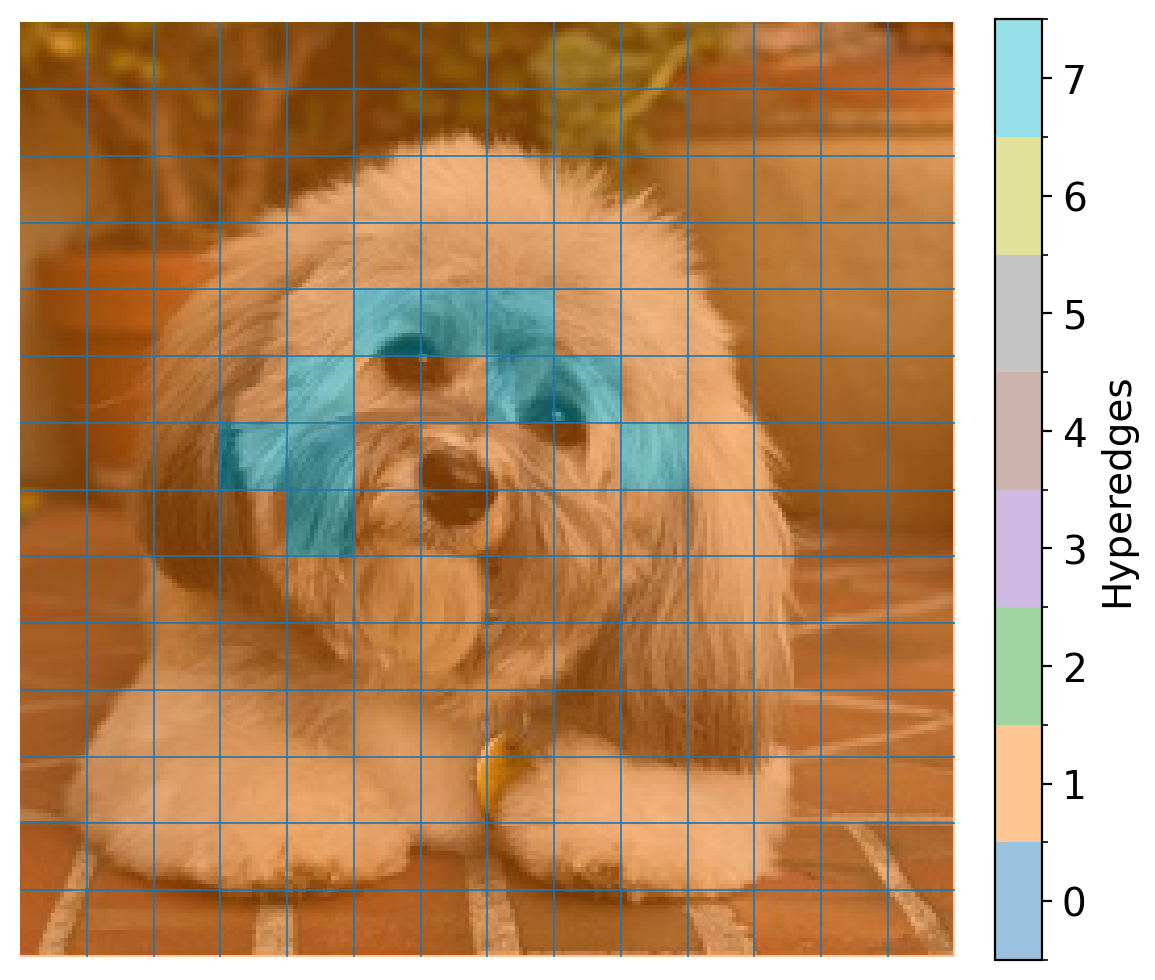} \\[-1pt]

\rowlab{KITTI(Attn)} &
\img{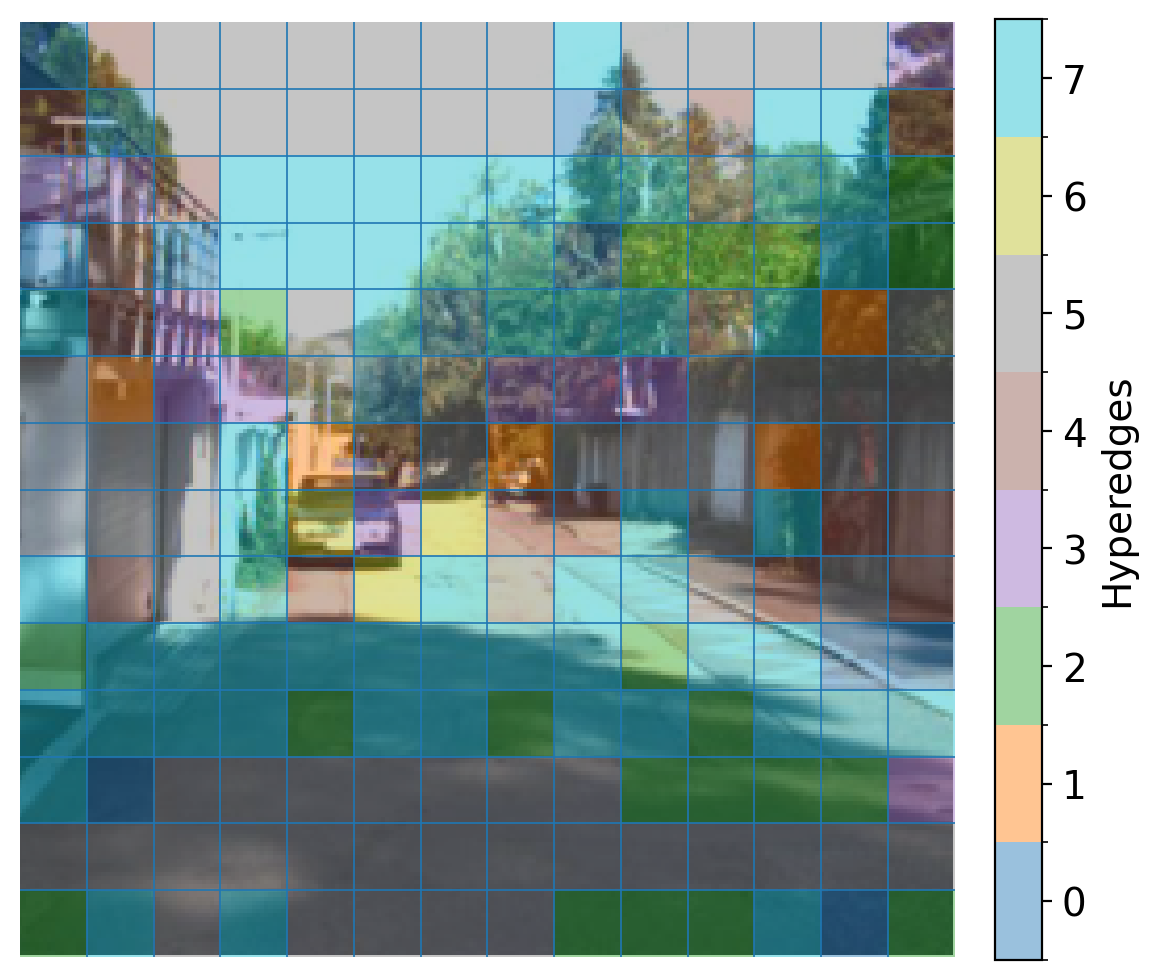} &
\img{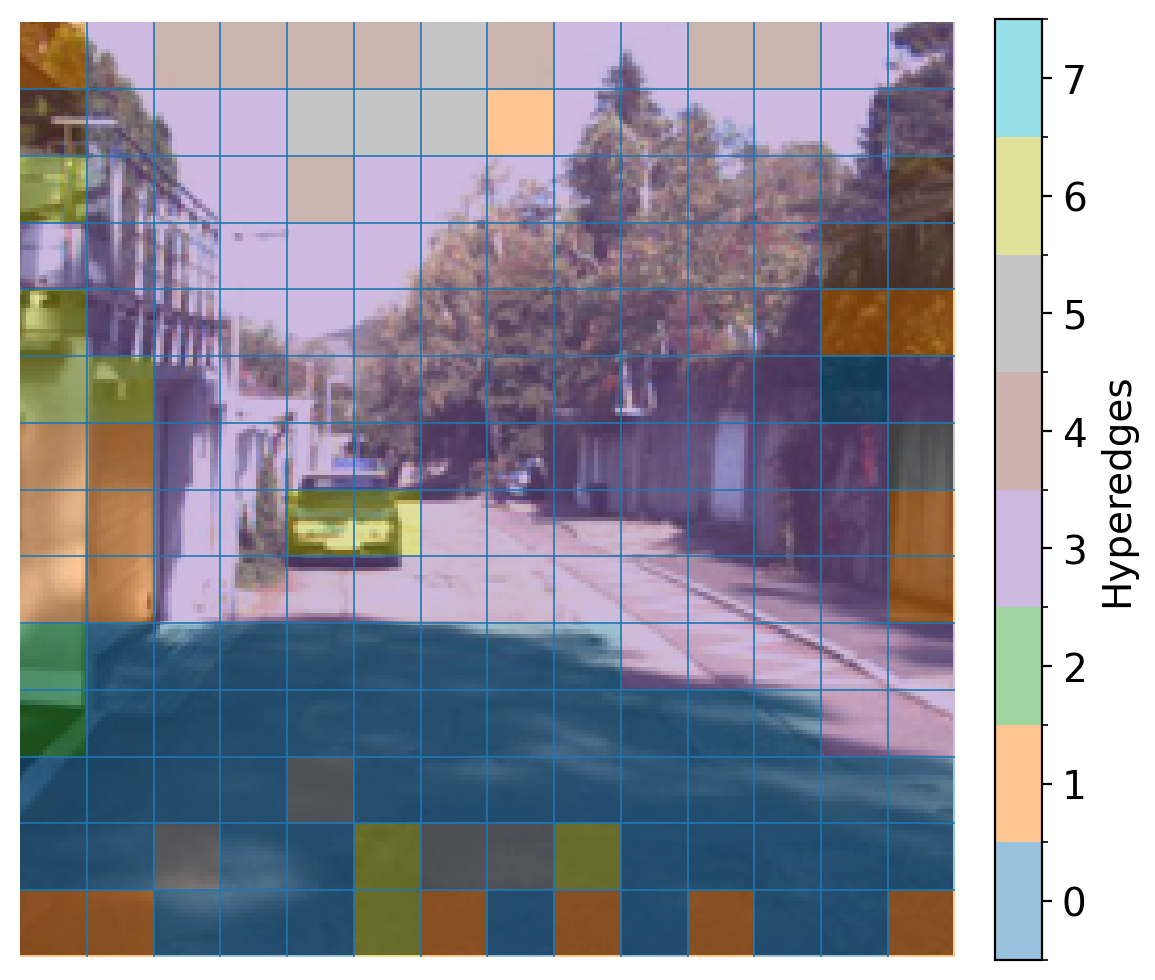} &
\img{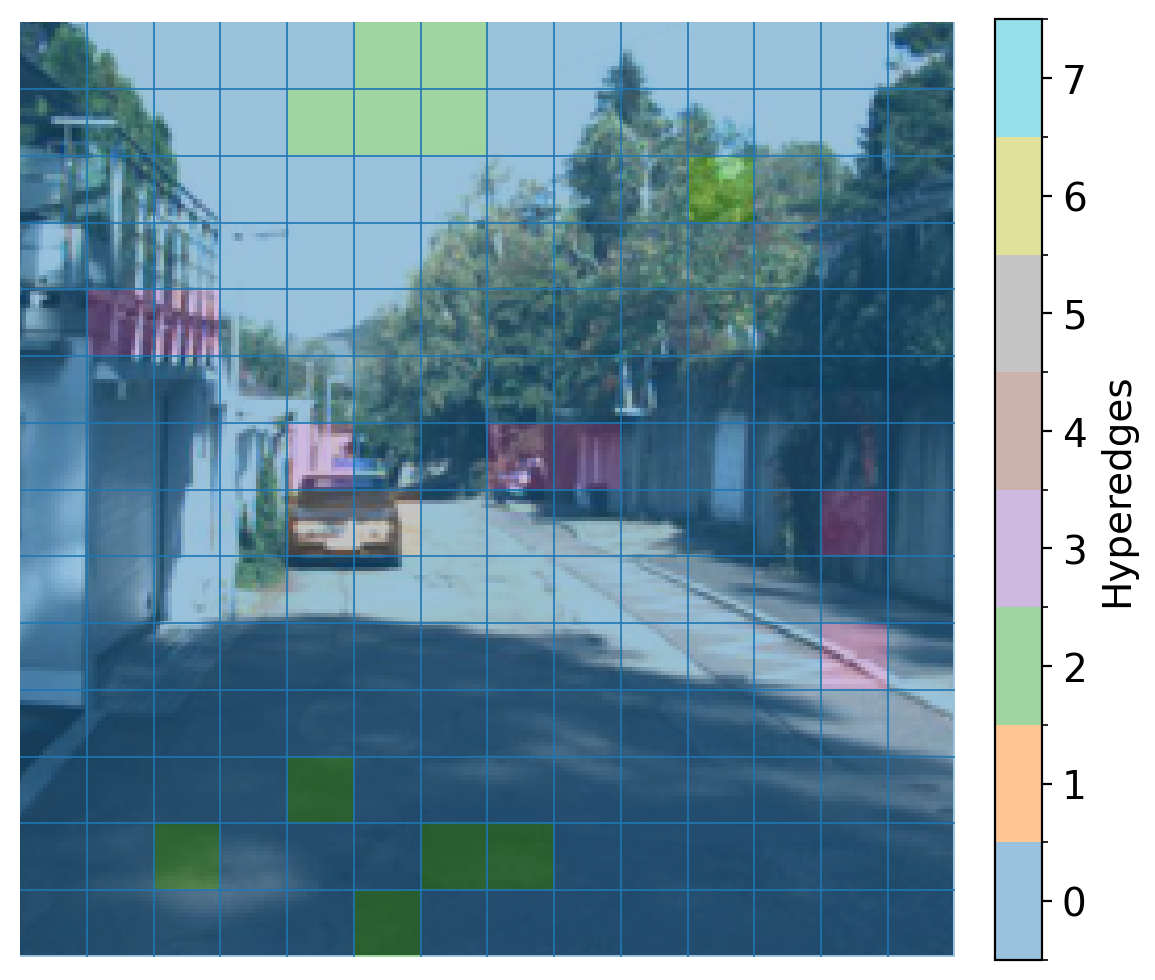} \\[-1pt]

\rowlab{KITTI(MLP)} &
\img{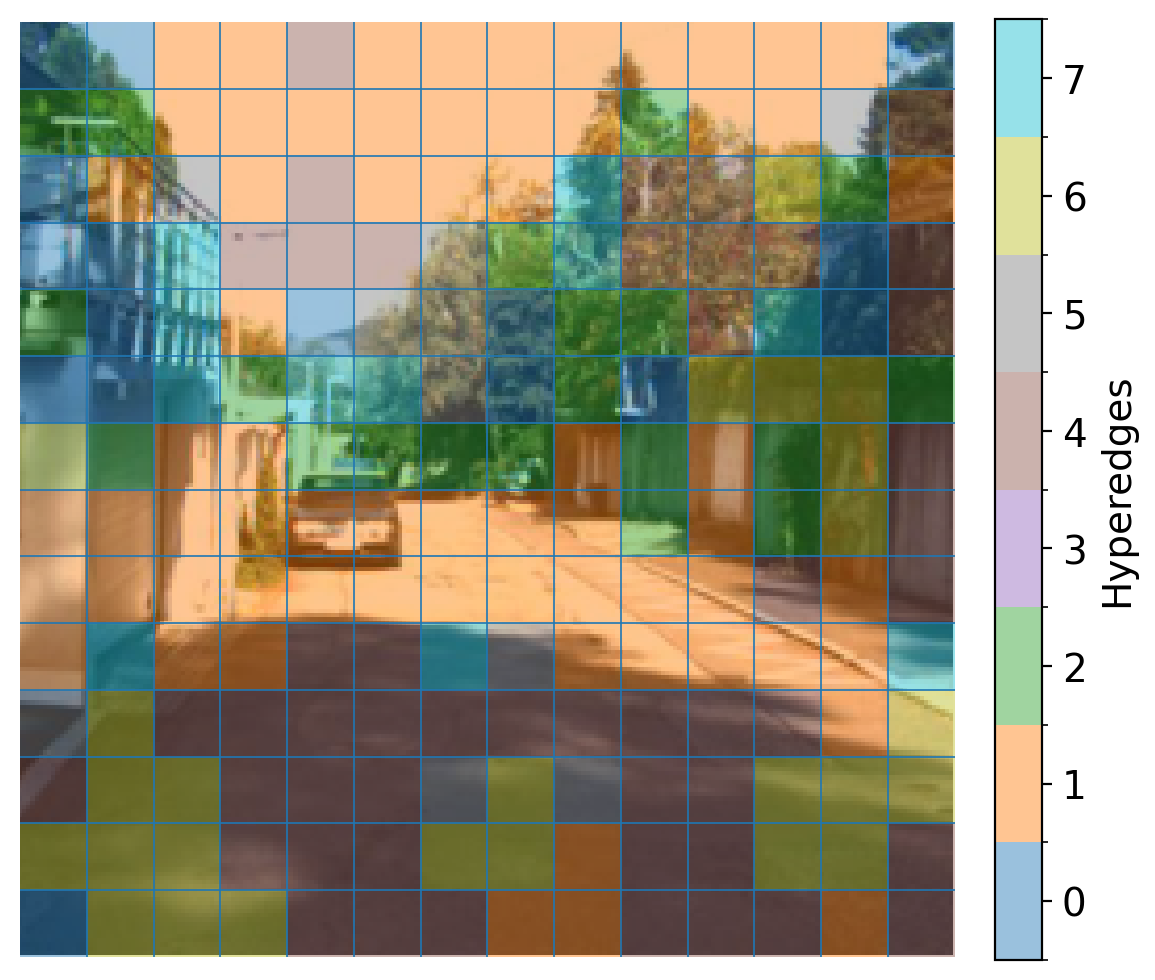} &
\img{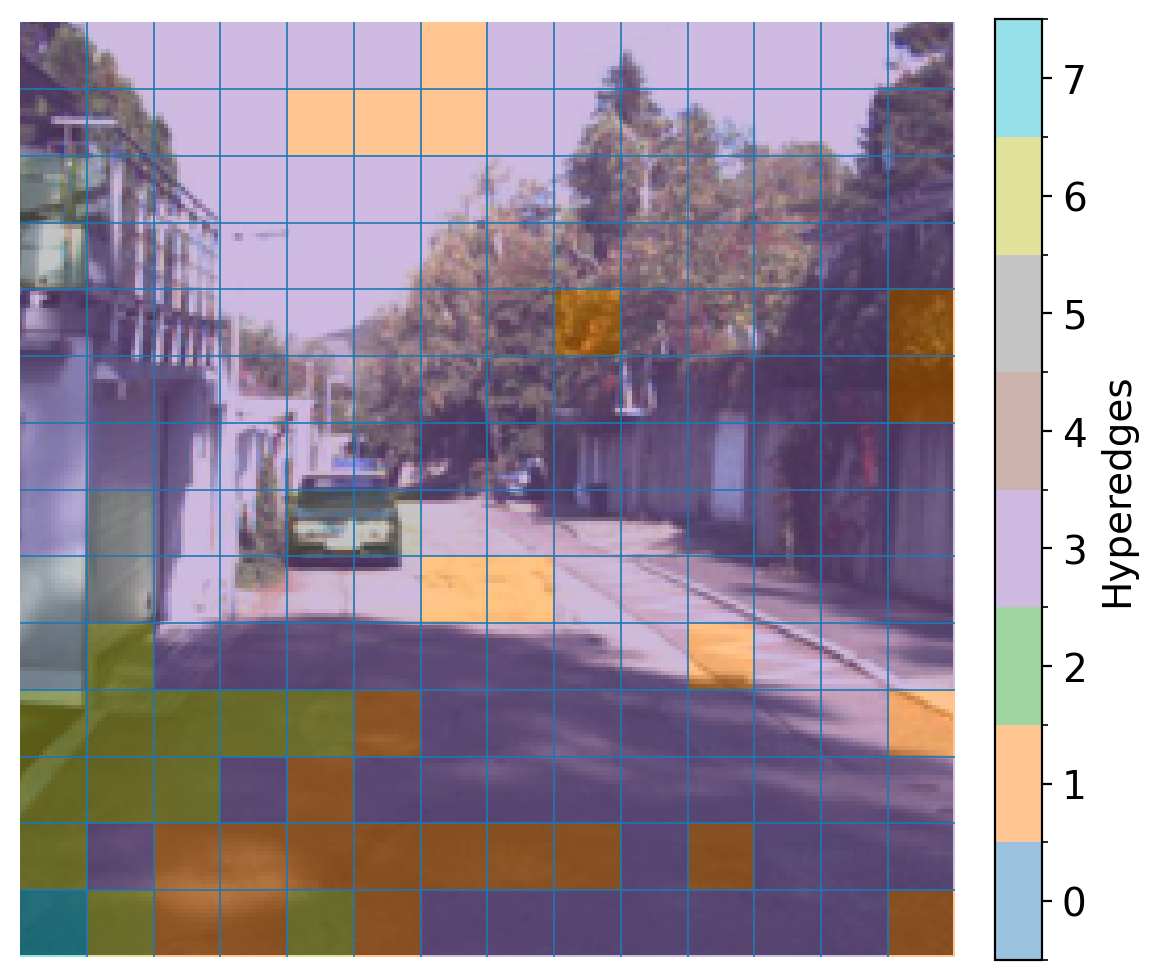} &
\img{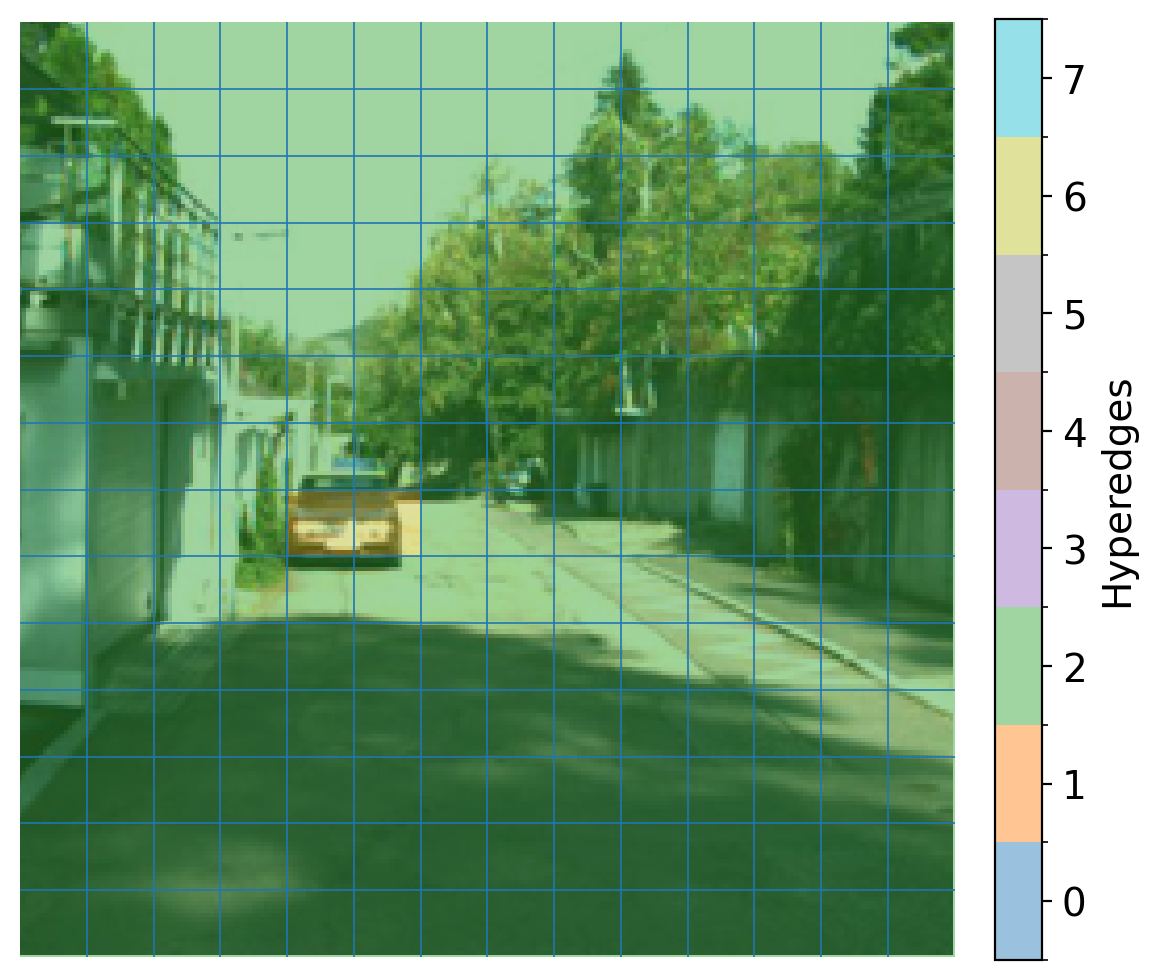} \\
\end{tabular}
\end{adjustbox}
\caption{Patch-grid routing visualization.
Each patch is colored according to the hyperedge receiving the highest routing probability.
We show representative routing patterns from Blocks 1, 6, and 12 across datasets and modules.
Early layers exhibit fragmented assignments, while deeper layers form more coherent spatial
groups, indicating that HyperAdapter organizes tokens into semantically meaningful hyperedge clusters. For clarity, we recommend zooming in.}
\label{fig:patch-grid}
\end{figure}

\subsection{Patch-Grid Routing Visualizations}
Fig.~\ref{fig:patch-grid} visualizes the spatial routing behavior of HyperAdapter by coloring each patch according to the hyperedge receiving the highest routing
probability. We show representative routing patterns across Blocks 1, 6, and 12 for multiple datasets and both attention and MLP adapters. 

Several consistent patterns emerge from these visualizations. In early transformer layers (e.g., Block 1), routing assignments appear fragmented, with patches distributed across multiple hyperedges. This behavior suggests
that low-level visual features remain broadly shared and are not yet organized into distinct semantic groups. As depth increases (Block 6), routing patterns become progressively more structured, with neighboring patches frequently assigned to the same hyperedge. This indicates that the model begins to capture spatially coherent regions corresponding to object parts
or contextual background structures.
In deeper layers (Block 12), the routing becomes more stable and semantically meaningful. Larger contiguous regions of the image are often assigned to the same hyperedge, suggesting that hyperedges specialize in aggregating semantically related tokens. Importantly, this behavior is observed across both attention and MLP adapters and remains consistent across datasets, indicating that the hyperedge routing mechanism learns dataset-agnostic grouping structures.

These visualizations provide qualitative evidence that HyperAdapter performs structured token grouping rather than independent token-wise adaptation, progressively organizing patch tokens into coherent hyperedge clusters as features propagate through the transformer.

\begin{figure}[tbp]
\vspace{-0.8cm}
\centering
\begin{tabular}{cccc}
\textbf{Original} & \textbf{Baseline} & \textbf{AdaptFormer} &\textbf{Ours} \\[2mm]
\includegraphics[width=0.24\linewidth,height=3.0cm]{img/daam/flowers/flower.jpg} &
\includegraphics[width=0.24\linewidth,height=3.0cm]{img/daam/flowers/baseline/baseline_DAAM_Block12.jpg} &
\includegraphics[width=0.24\linewidth,height=3.0cm]{img/daam/flowers/adaptformer/adaptformer_DAAM_Block12.jpg} &
\includegraphics[width=0.24\linewidth,height=3.0cm]{img/daam/flowers/hyperadapter/hyperadapter_DAAM_Block12.jpg} \\
\includegraphics[width=0.24\linewidth,height=3.0cm]{img/daam/pet/pet.jpg} &
\includegraphics[width=0.24\linewidth,height=3.0cm]{img/daam/pet/baseline/baseline_DAAM_Block12.jpg} &
\includegraphics[width=0.24\linewidth,height=3.0cm]{img/daam/pet/adaptformer/adaptformer_DAAM_Block12.jpg} &
\includegraphics[width=0.24\linewidth,height=3.0cm]{img/daam/pet/hyperadapter/hyperadapter_DAAM_Block12.jpg} \\
\includegraphics[width=0.24\linewidth,height=3.0cm]{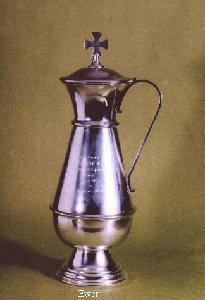} &
\includegraphics[width=0.24\linewidth,height=3.0cm]{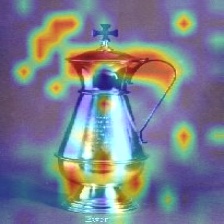} &
\includegraphics[width=0.24\linewidth,height=3.0cm]{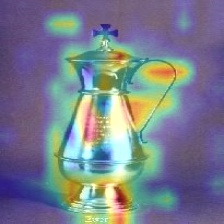} &
\includegraphics[width=0.24\linewidth,height=3.0cm]{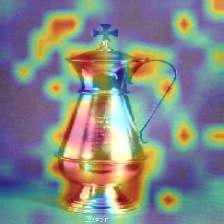} \\
\includegraphics[width=0.24\linewidth,height=3.0cm]{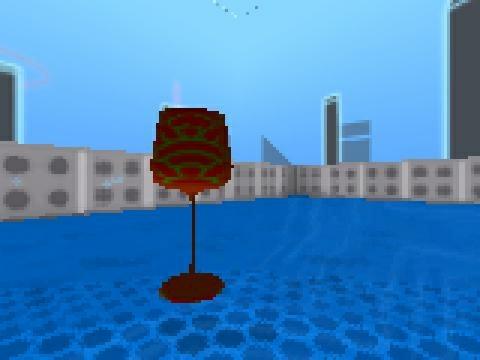} &
\includegraphics[width=0.24\linewidth,height=3.0cm]{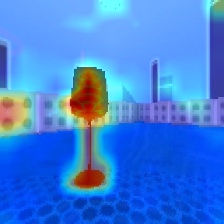} &
\includegraphics[width=0.24\linewidth,height=3.0cm]{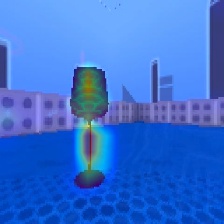} &
\includegraphics[width=0.24\linewidth,height=3.0cm]{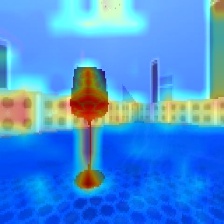} \\
\includegraphics[width=0.24\linewidth,height=3.0cm]{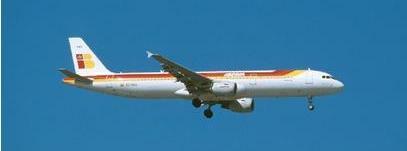} &
\includegraphics[width=0.24\linewidth,height=3.0cm]{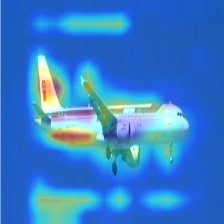} &
\includegraphics[width=0.24\linewidth,height=3.0cm]{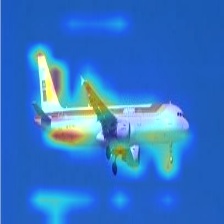} &
\includegraphics[width=0.24\linewidth,height=3.0cm]{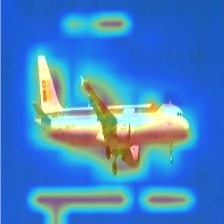} \\
\end{tabular}
\vspace{-0.2cm}
\caption{DAAM\cite{daam} visualizations comparing spatial attribution across PEFT methods. Columns show the original image, token-wise baseline, AdaptFormer, and HyperAdapter. HyperAdapter produces more concentrated and semantically aligned activations, highlighting relevant object regions while reducing background noise, reflecting the benefits of hyperedge-based routing.}
\label{fig:visualization_daam}
\end{figure}

\section{More Visualizations}
\subsection{Comparisons with Baseline and AdaptFormer}
Fig.~\ref{fig:visualization_daam} presents DAAM activation maps comparing the baseline, AdaptFormer, and our proposed HyperAdapter. Each visualization highlights the spatial regions that contribute most strongly to the model’s prediction. 

Several qualitative patterns emerge across examples. First, the baseline often produces relatively diffuse activations that spread across both object and background regions. This indicates that token updates are largely independent and may capture less structured spatial relationships. AdaptFormer improves localization in some cases, but the resulting attribution maps remain partially scattered, with activations occasionally extending into irrelevant background areas. In contrast, HyperAdapter consistently produces more concentrated and semantically aligned activations. The highlighted regions closely correspond to the primary object structures, such as the flower petals, the dog’s face, the teapot body, the street object, and the aircraft. This improved spatial focus suggests that hyperedge-based routing encourages tokens belonging to related semantic regions to be grouped and updated collectively. As a result, the model is able to emphasize coherent object parts while suppressing background noise. Across diverse image categories, the visualizations indicate that HyperAdapter leads to more structured and interpretable feature representations compared to conventional token-wise adaptation.

\begin{figure}[tbp]
\centering
\setlength{\tabcolsep}{1.2pt}        
\renewcommand{\arraystretch}{0.92}   
\newcommand{\imgH}{35mm}                
\newcommand{\rowlabw}{2.4cm}            
\newcommand{\imgcellw}{0.080\textwidth} 
\newcommand{\img}[1]{\includegraphics[height=\imgH,width=\imgcellw,keepaspectratio]{#1}}
\newcommand{\rowlab}[1]{\centering\scriptsize\bfseries #1}
\begin{adjustbox}{max width=\linewidth}
\begin{tabular}{@{} >{\centering\arraybackslash}m{\rowlabw} *{12}{>{\centering\arraybackslash}m{\imgcellw}} @{}}
& \scriptsize Block~1 & \scriptsize Block~2 & \scriptsize Block~3 &
\scriptsize Block~4 & \scriptsize Block~5 & \scriptsize Block~6 &
\scriptsize Block~7 & \scriptsize Block~8 & \scriptsize Block~9 &
\scriptsize Block~10 & \scriptsize Block~11 & \scriptsize Block~12 \\[2pt]
\rowlab{Baseline} &
\img{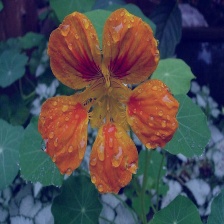} &
\img{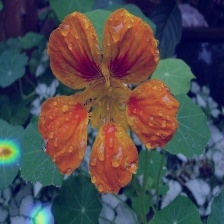} &
\img{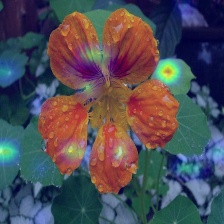} &
\img{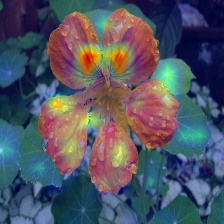} &
\img{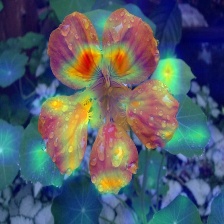} &
\img{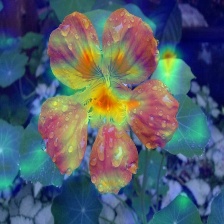} &
\img{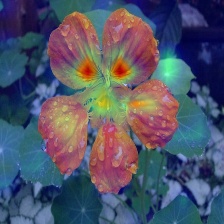} &
\img{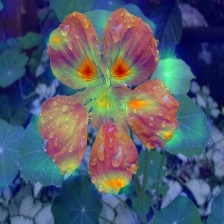} &
\img{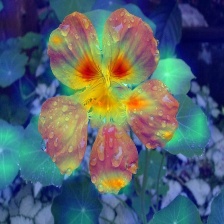} &
\img{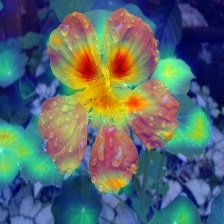} &
\img{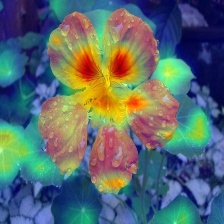} &
\img{img/daam/flowers/baseline/baseline_DAAM_Block12.jpg} \\[-1.2pt]

\rowlab{AdaptFormer} &
\img{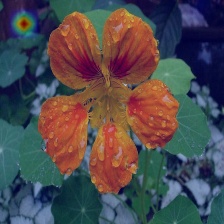} &
\img{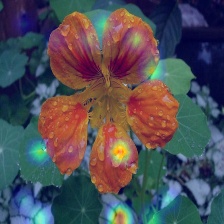} &
\img{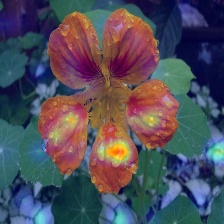} &
\img{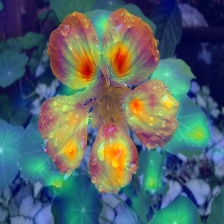} &
\img{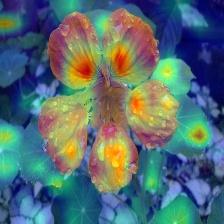} &
\img{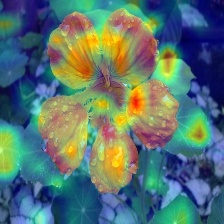} &
\img{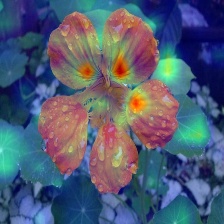} &
\img{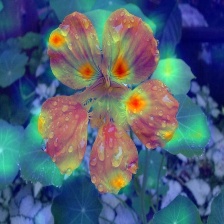} &
\img{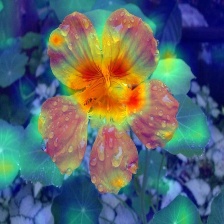} &
\img{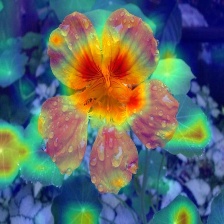} &
\img{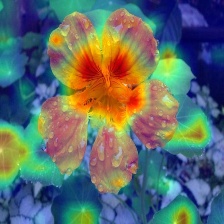} &
\img{img/daam/flowers/adaptformer/adaptformer_DAAM_Block12.jpg} \\[-0.6pt]

\rowlab{Ours} &
\img{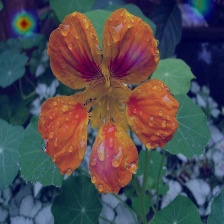} &
\img{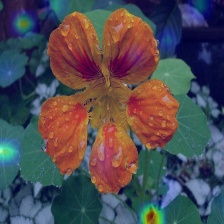} &
\img{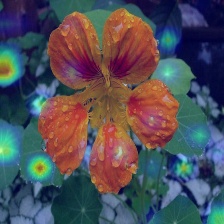} &
\img{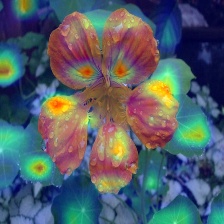} &
\img{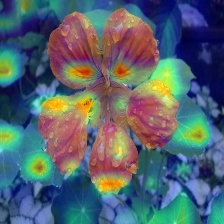} &
\img{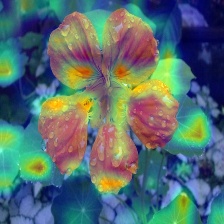} &
\img{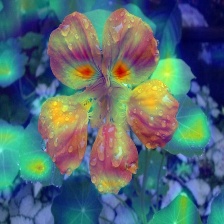} &
\img{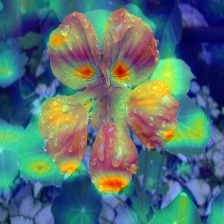} &
\img{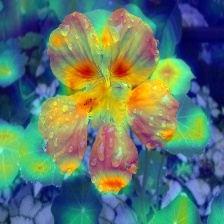} &
\img{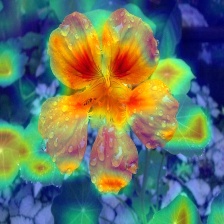} &
\img{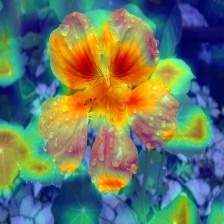} &
\img{img/daam/flowers/hyperadapter/hyperadapter_DAAM_Block12.jpg} \\[-0.6pt]

\rowlab{Baseline} &
\img{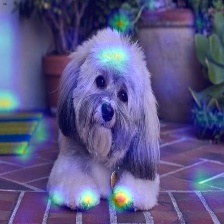} &
\img{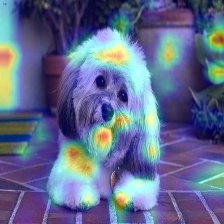} &
\img{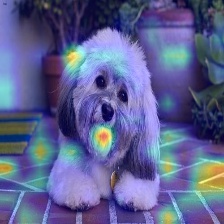} &
\img{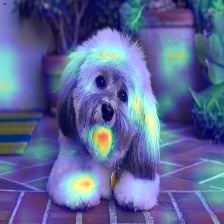} &
\img{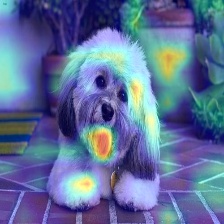} &
\img{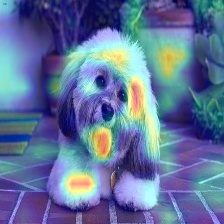} &
\img{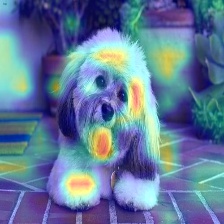} &
\img{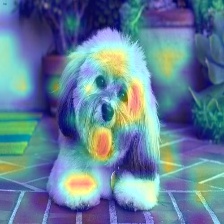} &
\img{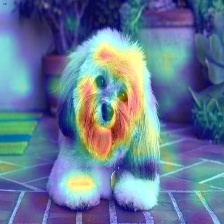} &
\img{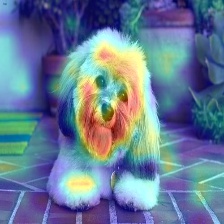} &
\img{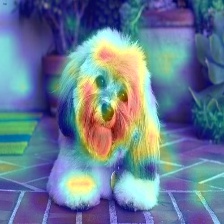} &
\img{img/daam/pet/baseline/baseline_DAAM_Block12.jpg} \\[-0.6pt]

\rowlab{AdaptFormer} &
\img{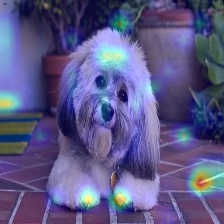} &
\img{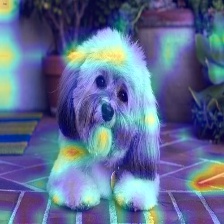} &
\img{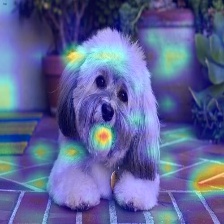} &
\img{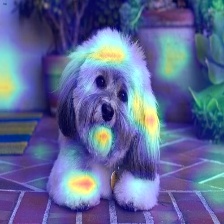} &
\img{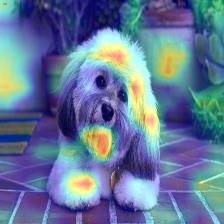} &
\img{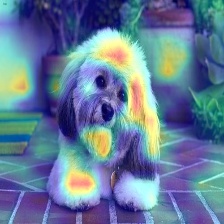} &
\img{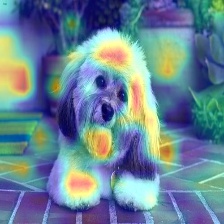} &
\img{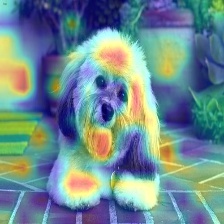} &
\img{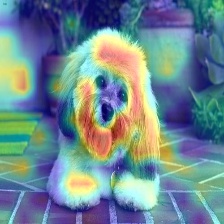} &
\img{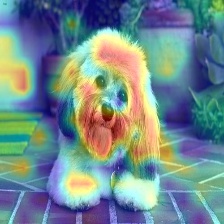} &
\img{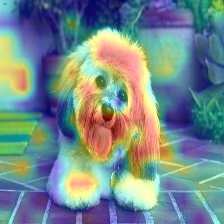} &
\img{img/daam/pet/adaptformer/adaptformer_DAAM_Block12.jpg} \\[-0.6pt]

\rowlab{Ours} &
\img{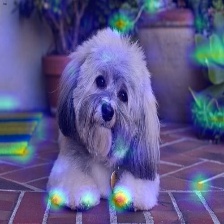} &
\img{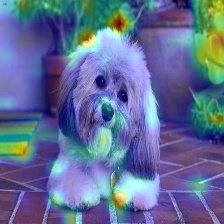} &
\img{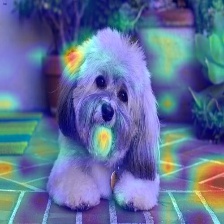} &
\img{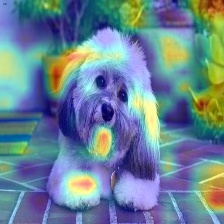} &
\img{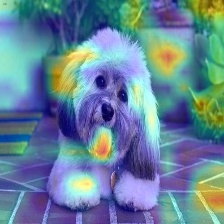} &
\img{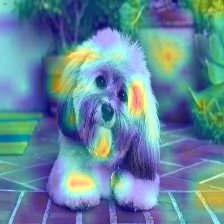} &
\img{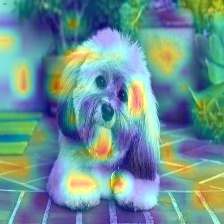} &
\img{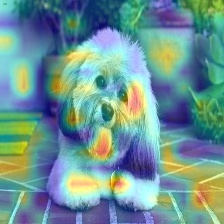} &
\img{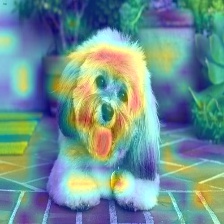} &
\img{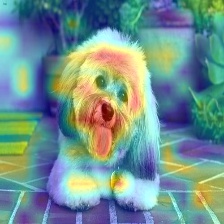} &
\img{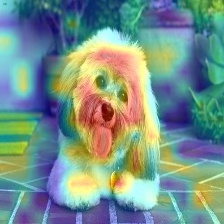} &
\img{img/daam/pet/hyperadapter/hyperadapter_DAAM_Block12.jpg} \\[-0.6pt]

\rowlab{Baseline} &
\img{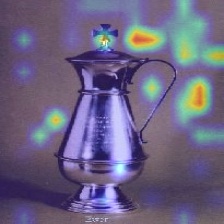} &
\img{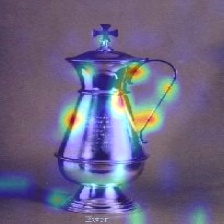} &
\img{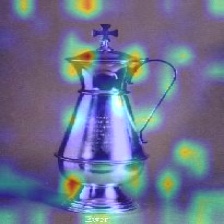} &
\img{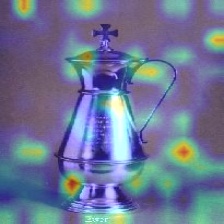} &
\img{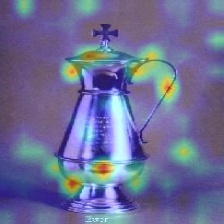} &
\img{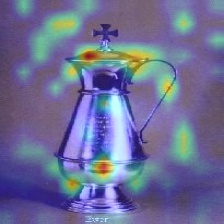} &
\img{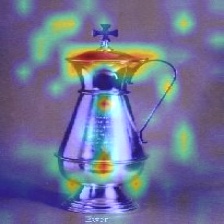} &
\img{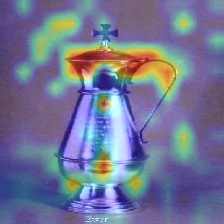} &
\img{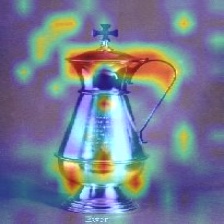} &
\img{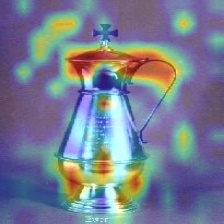} &
\img{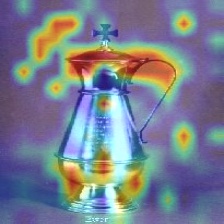} &
\img{img/daam/cup/baseline/baseline_DAAM_Block12.jpg} \\[-0.6pt]

\rowlab{AdaptFormer} &
\img{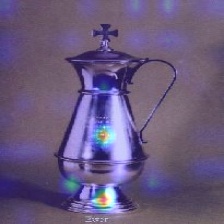} &
\img{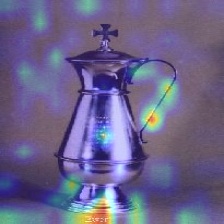} &
\img{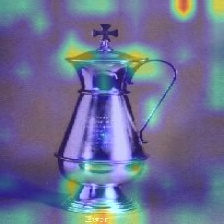} &
\img{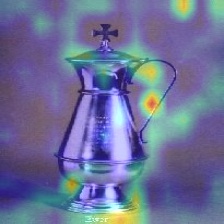} &
\img{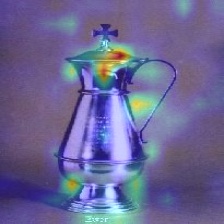} &
\img{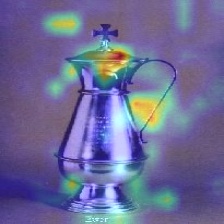} &
\img{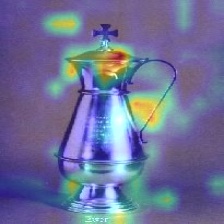} &
\img{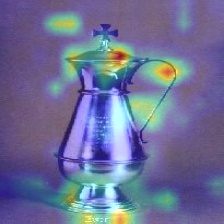} &
\img{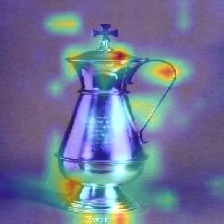} &
\img{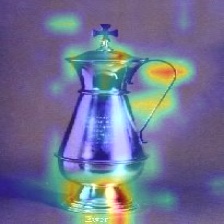} &
\img{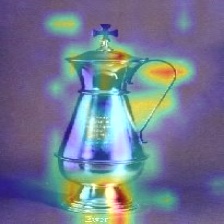} &
\img{img/daam/cup/adaptformer/adaptformer_DAAM_Block12.jpg} \\[-0.6pt]

\rowlab{Ours} &
\img{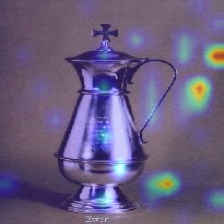} &
\img{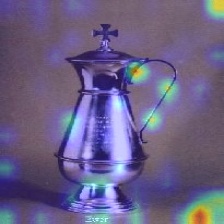} &
\img{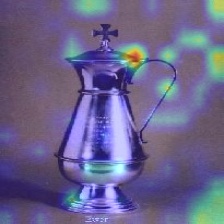} &
\img{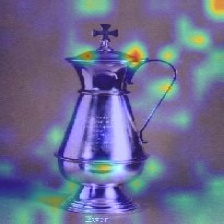} &
\img{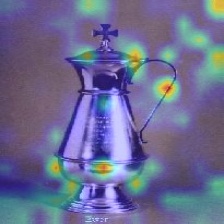} &
\img{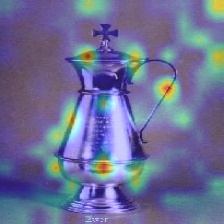} &
\img{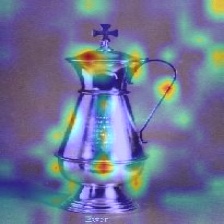} &
\img{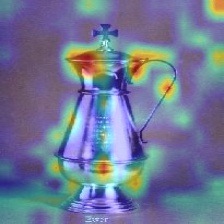} &
\img{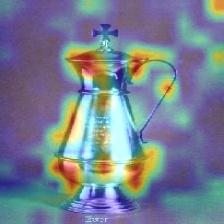} &
\img{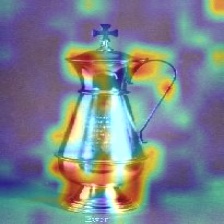} &
\img{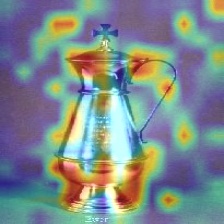} &
\img{img/daam/cup/hyperadapter/hyperadapter_DAAM_Block12.jpg} \\[-0.6pt]

\rowlab{Baseline} &
\img{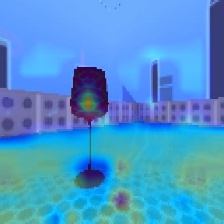} &
\img{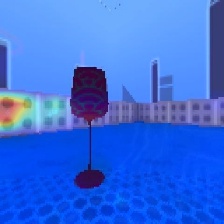} &
\img{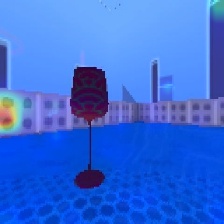} &
\img{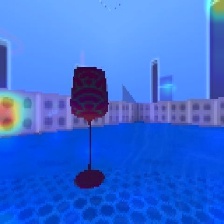} &
\img{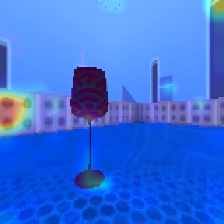} &
\img{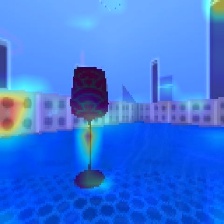} &
\img{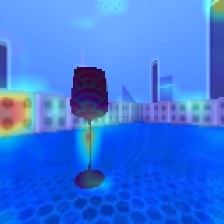} &
\img{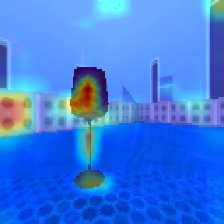} &
\img{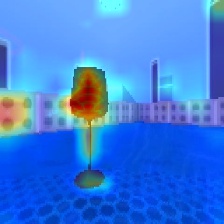} &
\img{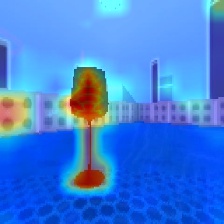} &
\img{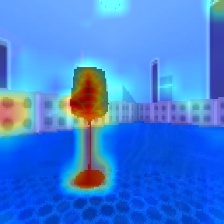} &
\img{img/daam/dmlab/baseline/baseline_DAAM_Block12.jpg} \\[-0.6pt]

\rowlab{AdaptFormer} &
\img{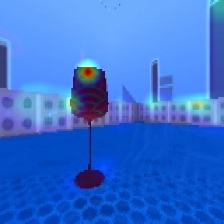} &
\img{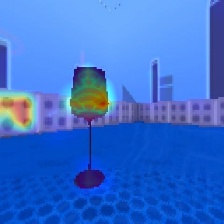} &
\img{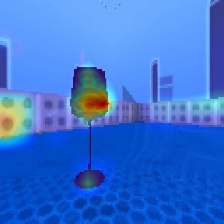} &
\img{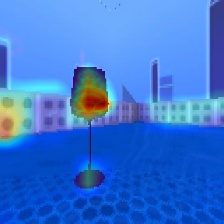} &
\img{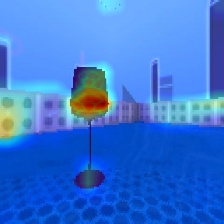} &
\img{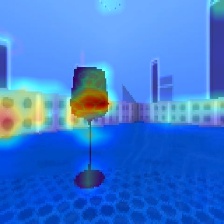} &
\img{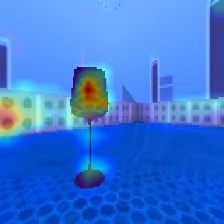} &
\img{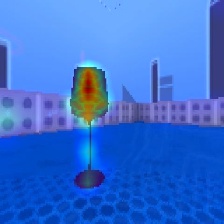} &
\img{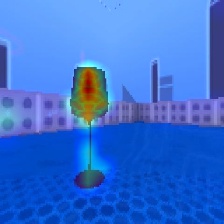} &
\img{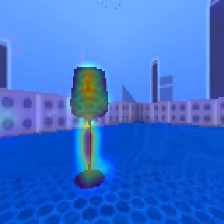} &
\img{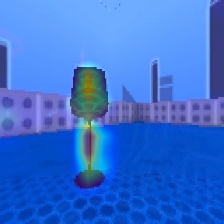} &
\img{img/daam/dmlab/adaptformer/adaptformer_DAAM_Block12.jpg} \\[-0.6pt]

\rowlab{Ours} &
\img{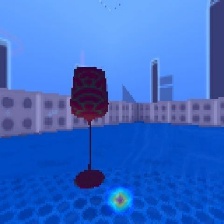} &
\img{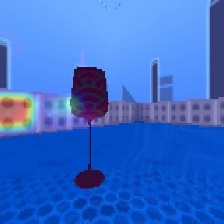} &
\img{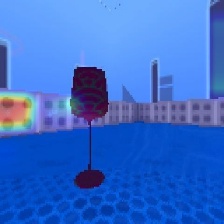} &
\img{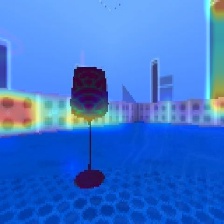} &
\img{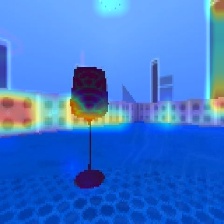} &
\img{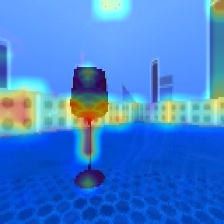} &
\img{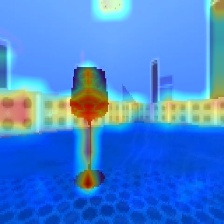} &
\img{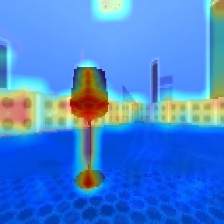} &
\img{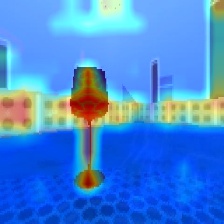} &
\img{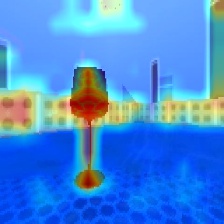} &
\img{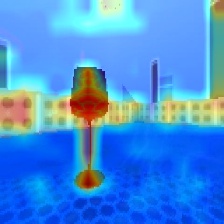} &
\img{img/daam/dmlab/hyperadapter/hyperadapter_DAAM_Block12.jpg} \\[-0.6pt]

\rowlab{Baseline} &
\img{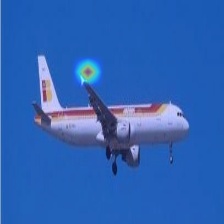} &
\img{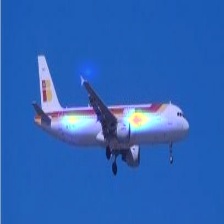} &
\img{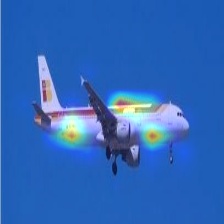} &
\img{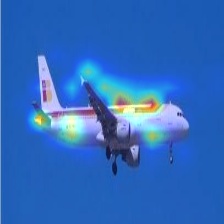} &
\img{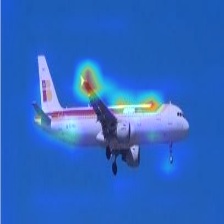} &
\img{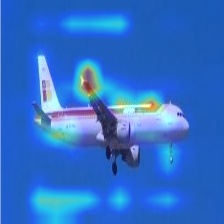} &
\img{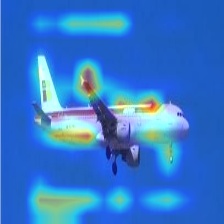} &
\img{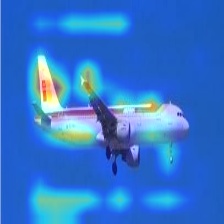} &
\img{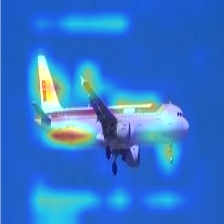} &
\img{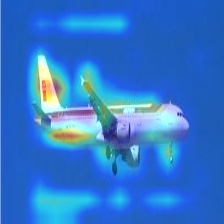} &
\img{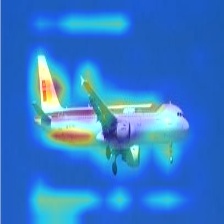} &
\img{img/daam/plane/baseline/baseline_DAAM_Block12.jpg} \\[-0.6pt]

\rowlab{AdaptFormer} &
\img{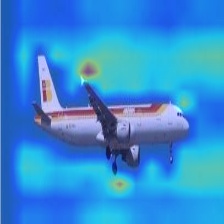} &
\img{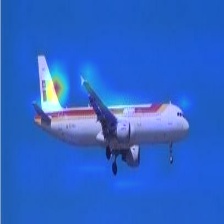} &
\img{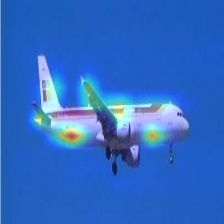} &
\img{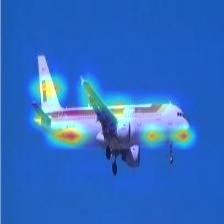} &
\img{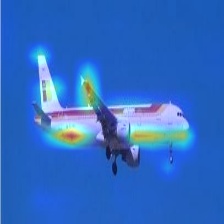} &
\img{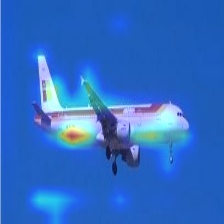} &
\img{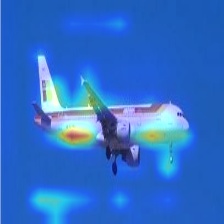} &
\img{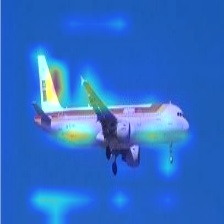} &
\img{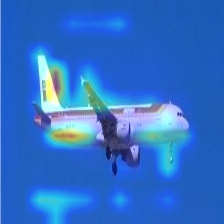} &
\img{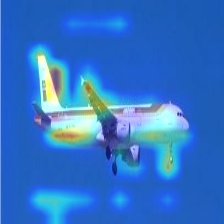} &
\img{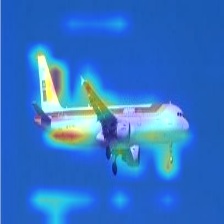} &
\img{img/daam/plane/adaptformer/adaptformer_DAAM_Block12.jpg} \\[-0.6pt]

\rowlab{Ours} &
\img{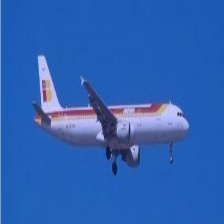} &
\img{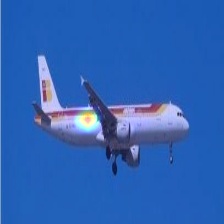} &
\img{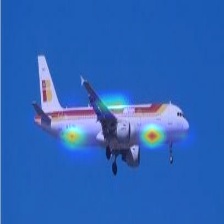} &
\img{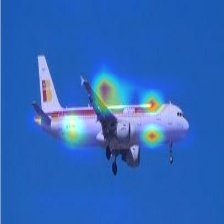} &
\img{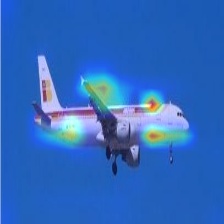} &
\img{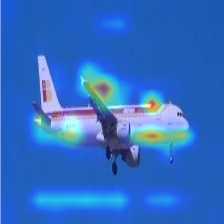} &
\img{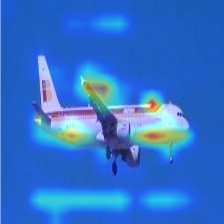} &
\img{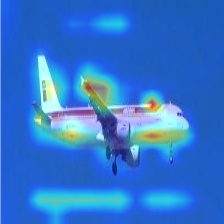} &
\img{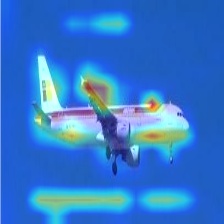} &
\img{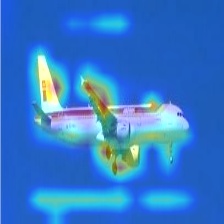} &
\img{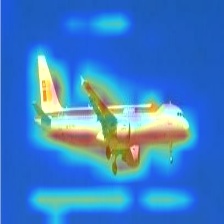} &
\img{img/daam/plane/hyperadapter/hyperadapter_DAAM_Block12.jpg} \\[-1.2pt]
\end{tabular}
\end{adjustbox}
\caption{DAAM \cite{daam} visualizations comparing HyperAdapter model with baseline and AdaptFormer models on VTAB-1K across all 12 transformer blocks.}
\label{fig:daam_layers}
\end{figure}

\subsection{Layer-wise DAAM Visualizations}
Fig.~\ref{fig:daam_layers} presents DAAM visualizations across all transformer blocks, comparing the  baseline, AdaptFormer, and the proposed HyperAdapter. The figure illustrates how spatial attribution evolves throughout the network depth. 

Across early transformer blocks (Blocks 1-4), all methods exhibit relatively diffuse activations distributed across broad image regions. This behavior reflects the low-level feature extraction stage, where representations remain largely shared across tokens. As depth increases, attribution maps become progressively more localized. In the baseline and AdaptFormer, however, activations often remain fragmented or spread across background areas even in deeper layers. In contrast, HyperAdapter produces increasingly concentrated and semantically aligned activations as features propagate through the network. In later blocks (Blocks 9-12), the attribution maps strongly focus on the primary object regions, such as the flower petals, the dog’s face, the teapot body, the street object, and the aircraft fuselage. Background activations are noticeably reduced compared to the other methods. This progressive sharpening of spatial attribution suggests that hyperedge-based routing encourages tokens belonging to related semantic regions to be grouped and updated collectively. As a result, HyperAdapter forms more coherent feature representations across layers, enabling clearer object localization and more interpretable model behavior. 

These observations provide qualitative evidence that hyperedge-level adaptation improves the spatial organization of token representations throughout the transformer hierarchy.

\section{Limitations and Future Work}

Despite its effectiveness, several limitations of HyperAdapter remain. First, the proposed hyperedge routing mechanism introduces additional computation compared to conventional token-wise adapters, as it requires token-to-hyperedge assignment, aggregation, and propagation operations. Although the overhead is modest in practice, further optimization of the routing mechanism could improve scalability for larger backbones and higher-resolution inputs.

Second, the current framework uses a fixed number of hyperedges $K$ across all layers and datasets. While our experiments show that moderate values of $K$ perform well, adaptive or data-dependent hyperedge  construction could potentially capture richer token relationships and 
improve flexibility across tasks.

Finally, our experiments focus primarily on vision classification benchmarks such as VTAB-1K and few-shot fine-grained visual recognition datasets. Extending HyperAdapter to other tasks, including dense prediction, video understanding, and multimodal learning, remains an interesting direction for future work.
\end{document}